\documentclass[pdflatex,sn-mathphys-ay]{sn-jnl}

\usepackage{graphicx}
\usepackage{amsmath,amssymb,amsfonts}
\usepackage{amsthm}
\usepackage{mathtools}
\usepackage{mathrsfs}
\usepackage[title]{appendix}

\theoremstyle{thmstyleone}
\newtheorem{theorem}{Theorem}
\newtheorem{lemma}[theorem]{Lemma}
\newtheorem{proposition}[theorem]{Proposition}
\newtheorem{corollary}[theorem]{Corollary}
\theoremstyle{thmstyletwo}
\newtheorem{remark}[theorem]{Remark}
\theoremstyle{thmstylethree}
\newtheorem{definition}[theorem]{Definition}

\newcommand{\bP}{\mathbf{P}}
\def\qed{\hfill\ensuremath{\Box}}
\newcommand{\R}{\mathbb{R}}

\newcommand{\bX}{\mathbf{X}}
\newcommand{\bY}{\mathbf{Y}}
\newcommand{\bW}{\mathbf{W}}
\newcommand{\bDelta}{\boldsymbol{\Delta}}
\newcommand{\bPsi}{\mathbf{\Psi}}
\newcommand{\bM}{\mathbf{M}}
\newcommand{\bH}{\mathbf{H}}
\newcommand{\bI}{\mathbf{I}}
\newcommand{\bD}{\mathbf{D}}
\newcommand{\bR}{\mathbf{R}}
\newcommand{\bV}{\mathbf{V}}

\newcommand{\bone}{\mathbf{1}}
\newcommand{\bzero}{\mathbf{0}}
\newcommand{\bmu}{\boldsymbol{\mu}}
\newcommand{\balpha}{\boldsymbol{\alpha}}
\newcommand{\bSigma}{\boldsymbol{\Sigma}}
\newcommand{\bGamma}{\boldsymbol{\Gamma}}
\newcommand{\bLambda}{\boldsymbol{\Lambda}}
\newcommand{\bepsilon}{\boldsymbol{\epsilon}}
\newcommand{\bpi}{\boldsymbol{\pi}}
\newcommand{\bTheta}{\boldsymbol{\Theta}}
\newcommand{\bdelta}{\boldsymbol{\delta}}
\newcommand{\bA}{\mathbf{A}}
\newcommand{\bC}{\mathbf{C}}
\newcommand{\bE}{\mathbf{E}}
\newcommand{\bF}{\mathbf{F}}
\newcommand{\bL}{\mathbf{L}}
\newcommand{\bQ}{\mathbf{Q}}
\newcommand{\bB}{\mathbf{B}}
\newcommand{\bT}{\mathbf{T}}
\newcommand{\bZ}{\mathbf{Z}}
\newcommand{\bPi}{\boldsymbol{\Pi}}
\newcommand{\Stief}{\mathrm{St}}
\newcommand{\col}{\mathrm{col}}
\newcommand{\rank}{\mathrm{rank}}
\newcommand{\tr}{\mathrm{tr}}
\newcommand{\diag}{\mathrm{diag}}
\newcommand{\Sb}{{S_b^{\mathrm{ML}}}}
\newcommand{\Sw}{{S_w^{\mathrm{ML}}}}
\newcommand{\Stml}{{S_t^{\mathrm{ML}}}}
\newcommand{\psd}{\succeq}
\newcommand{\npsd}{\preceq}

\newcommand{\kmax}{k_{\max}}
\newcommand{\gap}{\mathrm{gap}}
\newcommand{\Sbpop}{{S_b^{\mathrm{pop}}}}
\newcommand{\Swpop}{{S_w^{\mathrm{pop}}}}
\newcommand{\Stmlpop}{{S_t^{\mathrm{ML,pop}}}}
\newcommand{\Mstar}{{\bM^*}}
\newcommand{\What}{\hat{\bW}}
\newcommand{\Wstar}{{\bW^*}}
\newcommand{\Vstar}{{\bV^*}}
\newcommand{\Gr}{\mathrm{Gr}}
\newcommand{\bSigmay}{\bSigma_{\mathbf{y}}}
\newcommand{\nmin}{n_{\min}}
\newcommand{\Swpopc}{{S_{w,c}^{\mathrm{pop}}}}

\renewcommand{\orcid}[1]{\,\href{https://orcid.org/#1}{\textsc{orcid}~#1}}

\begin{document}

\title[Orthogonal Multilabel Fisher Discriminants]{On the Spectral Structure and Objective Equivalence of Orthogonal Multilabel Fisher Discriminants}

\author*[1]{\fnm{Brian} \sur{Keith-Norambuena}\orcid{0000-0001-5734-8962}}\email{brian.keith@ucn.cl}
\author[2]{\fnm{Juan} \sur{Bekios-Calfa}\orcid{0000-0003-0085-2425}}\email{juan.bekios@ucn.cl}

\affil[1]{\orgdiv{Department of Computing \& Systems Engineering}, \orgname{Universidad Cat\'olica del Norte}, \orgaddress{\city{Antofagasta}, \country{Chile}}}
\affil[2]{\orgdiv{School of Engineering}, \orgname{Universidad Cat\'olica del Norte}, \orgaddress{\city{Coquimbo}, \country{Chile}}}

\abstract{%
We provide a unified theoretical analysis of Linear Discriminant
Analysis with simultaneous multilabel scatter matrix formulations and Stiefel
orthogonality constraints. Our contributions span both algebraic structure and
statistical guarantees. On the algebraic side, we characterize the rank of the
multilabel between-class scatter matrix, showing that the effective
discriminant dimensionality can strictly exceed the classical single-label
bound of $C-1$; we establish a multilabel partition of variance and prove that
all four Fisher objectives are equivalent under the $\bW^\top \Stml \bW =
\bI_r$ constraint while characterizing their divergence under the Stiefel
constraint; and we prove a two-sided label-distance preservation bound
relating projected distances to Hamming distances in label space. On the
statistical side, we establish a finite-sample $O(\kmax\sqrt{d\log d/n}/\gap_r)$
bound on the subspace estimation error under sub-Gaussian noise with a
matching $\Omega(\sigma^2 d/(n\,\gap_r))$ minimax lower
bound, establishing a near-minimax-optimal rate (matching up to
logarithmic and $\kmax$ factors) for multilabel discriminant subspace
estimation. We further provide high-probability distance
concentration, robustness guarantees under label interactions, and a
regularization analysis preserving the spectral structure when $d \gg n$.
All results are verified numerically on synthetic data generated from the
linear label-effect model, covering both the algebraic identities and the
multilabel-specific quantities ($\kmax$, $\kappa(\Stml)$,
$\|\bGamma/n\|_2$, $\Delta_r$) that govern the statistical bounds.
The numerical experiments are designed as a sanity check for the
theorems rather than as an empirical benchmark; evaluation on real
multilabel datasets is left to future work targeting
application-oriented venues.
}



\maketitle

\section{Introduction}
\label{sec:intro}

Linear Discriminant Analysis (LDA) \citep{fisher1936,rao1948} is among the
most fundamental techniques for supervised dimensionality reduction. Given $n$
samples in $\R^d$ with class labels, LDA seeks a linear projection that
maximizes between-class separation relative to within-class spread, as measured
by scatter matrices. The classical formulation solves a generalized eigenvalue
problem $S_b \mathbf{w} = \lambda S_w \mathbf{w}$, yielding at most $C-1$
discriminant directions for $C$ classes.

Two important extensions of classical LDA have been developed largely
independently.

\paragraph{Orthogonality-constrained LDA.} Standard LDA eigenvectors are
$S_w$-conjugate and orthogonal ($\mathbf{w}_i^\top S_w \mathbf{w}_j = 0$ for $i
\neq j$) but not Euclidean-orthogonal. Several methods enforce explicit
orthogonality: the Foley--Sammon method \citep{foley1975} extracts directions
sequentially with $\bW^\top \bW = \bI$ constraints via Lagrange multipliers;
Orthogonal LDA (OLDA) \citep{ye2005,ye2006} provides a batch formulation via
SVD and QR decomposition, with \citet{chu2010} providing a fast variant for
undersampled problems; Uncorrelated LDA (ULDA) \citep{jin2001} enforces
$\bW^\top S_t \bW = \bI$ for statistical uncorrelatedness (see
\citealt{hou2015} for a comparison of ULDA and OLDA); and trace ratio
solvers \citep{ngo2012,wang2023} address the true trace ratio problem on the
Stiefel manifold with globally convergent algorithms. \citet{luo2011} proved
that under orthonormality constraints, the four standard Fisher
objectives---trace ratio, ratio trace, determinant ratio, and trace
difference---yield equivalent optimal solutions in the single-label setting.

\paragraph{Multilabel LDA.} In multilabel classification
\citep{tsoumakas2007}, each sample may belong to multiple classes
simultaneously, requiring scatter matrix redefinition. \citet{park2008} gave an
early treatment of LDA for multilabel problems. \citet{wang2010}
introduced multilabel scatter matrices $\Sb$ and $\Sw$ with weighted per-label
class means. \citet{ji2009} proposed a regression-based framework subsuming
LDA, PLS, and CCA. \citet{sun2008} explored hypergraph spectral learning for
multilabel dimensionality reduction. \citet{xu2018} unified five weighting
schemes for multilabel scatter matrices, later extended with saliency-based
weights by \citet{xu2021}. \citet{zhang2010} and \citet{zheng2014} explored
dependence maximization (using HSIC \citep{gretton2005}) and kernel alignment
approaches, respectively. \citet{bekios2017} proposed a Gram--Schmidt-based
multilabel LDA extension. See \citet{siblini2021} for a comprehensive survey.

\paragraph{Why this combination matters.}
Orthogonality and multilabel scatter matrices address two complementary
weaknesses of classical LDA when applied to modern multilabel problems.
First, the single-label rank bound $\rank(S_b) \leq C-1$ is an artefact of
the scatter construction: in a multilabel problem with $L$ labels, the
natural analogue is governed by the label co-occurrence
structure, and---as we show---can be as large as $L$ rather than $L-1$,
yielding genuinely more discriminant directions than the single-label
setting. Second, the eigenvectors of the
generalized eigenvalue problem $S_b\mathbf{w} = \lambda S_w\mathbf{w}$ are
$S_w$-conjugate but not Euclidean-orthogonal, which makes the resulting
projection sensitive to noise scaling and complicates downstream tasks
that rely on Euclidean geometry (nearest-neighbor classification,
clustering, visualization). Imposing a Stiefel constraint gives an
explicit control of projected noise
(Corollary~\ref{cor:noise})---a property that matters precisely when the
projected features are used with a distance-based downstream classifier.
The combination of the two---multilabel scatter matrices with a Stiefel
(or multilabel total-scatter) orthogonality constraint---is therefore the
natural target for a theoretical analysis of multilabel discriminant
projections, even though it has so far been studied only in restricted
forms.

\paragraph{The gap.} Despite the maturity of both families, no published work
directly analyzes the \emph{theoretical properties} of LDA with simultaneous
multilabel scatter matrices and Stiefel orthogonality constraints:
\begin{equation}\label{eq:central}
\max_{\bW \in \Stief(d,r)} \; \frac{\tr(\bW^\top \Sb \bW)}{\tr(\bW^\top \Sw \bW)},
\end{equation}
where $\Stief(d,r) = \{\bW \in \R^{d \times r} : \bW^\top \bW = \bI_r\}$ is the
Stiefel manifold. The closest existing methods are Stiefel-constrained and
multilabel-aware, yet each of them optimizes an objective that is
\emph{not} an explicit Fisher ratio over multilabel scatter matrices, and
none provides the algebraic structure or statistical guarantees established
here:
\begin{itemize}
\item \textbf{kaLDA} \citep{zheng2014} maximizes the kernel-alignment
  objective $\tr(\bW^\top\bX^\top\bY\bY^\top\bX\bW)$, which is related
  to but not identical with the multilabel between-class trace and,
  critically, ignores the within-class scatter entirely---so there is no
  Fisher ratio, no analysis of the effective rank of $\Sb$, and no
  finite-sample subspace estimation rate.
\item \textbf{MDDM} \citep{zhang2008,zhang2010} maximizes the HSIC
  dependence between features and labels under a Stiefel constraint.
  The resulting spectral structure, objective-equivalence behavior,
  and finite-sample rates analyzed here are orthogonal to the
  dependence-maximization viewpoint and have not been derived for the
  Fisher formulation.
\item The \textbf{regression framework} of \citet{ji2009} unifies LDA,
  PLS, and CCA as least-squares problems, but it does not analyze the
  multilabel rank of $\Sb$, the equivalence of the four Fisher
  objectives, or the effect of Stiefel constraints on the solution.
\item \textbf{Sparse multilabel bilinear embedding on Stiefel manifolds}
  \citep{liu2018} and the \textbf{orthogonal multi-view framework} of
  \citet{wang2022} combine Stiefel constraints with multilabel-aware
  objectives, but the objectives are sparsity-regularized bilinear or
  multi-view-specific rather than the classical Fisher ratio with
  multilabel scatter matrices.
\end{itemize}
In short, the closest prior methods either drop the Fisher-ratio form,
replace it by a surrogate of $\Sb$ alone, or target a sparse/multi-view
variant; none of them characterizes the rank of the multilabel
between-class scatter, the equivalence (and divergence) of the four
Fisher objectives under Stiefel and multilabel total-scatter constraints,
the label-distance preservation property, or the near-minimax-optimal
subspace estimation rate for the formulation~\eqref{eq:central}.  The
present paper fills exactly this theoretical gap.

\paragraph{Contributions.} We provide a unified analysis of the
formulation~\eqref{eq:central} across four groups of results:
\begin{enumerate}
\item \textbf{Rank characterization} (Section~\ref{sec:rank}):
  $\rank(\Sb) = \rank(\tilde{\bX}^\top \bY) \leq \min(d, n-1, \rank(\bY))$,
  recovering the classical $C-1$ bound as a special case and showing
  that the multilabel setting can provide strictly more discriminant
  dimensions than the single-label case.
\item \textbf{Objective equivalence} (Section~\ref{sec:equiv}):
  a multilabel partition of variance identity and the equivalence of
  all four Fisher objectives under $\bW^\top \Stml \bW = \bI_r$; under
  the Stiefel constraint, we identify the commutativity conditions for
  equivalence and characterize the divergence under variable cardinalities.
\item \textbf{Label-distance preservation} (Section~\ref{sec:distance}):
  a two-sided bound relating expected squared projected distances to
  Hamming distances, with the noise contribution bounded by the
  orthogonality constraint.
\item \textbf{Statistical guarantees} (Section~\ref{sec:statistical}):
  an $O(\kmax\sqrt{d\log d/n}/\gap_r)$ finite-sample upper bound
  (Theorem~\ref{thm:subspace}) with a matching $\Omega(\sigma^2 d/(n\,\gap_r))$
  minimax lower bound (Theorem~\ref{thm:minimax}), plus
  high-probability distance concentration
  (Theorem~\ref{thm:concentration}), robustness to label interactions
  (Theorem~\ref{thm:interactions}), and a regularization analysis
  preserving the spectral structure when $d \gg n$
  (Theorem~\ref{thm:regularization}).
\end{enumerate}

\section{Preliminaries}
\label{sec:prelim}

\subsection{Notation}

Let $\bX \in \R^{n \times d}$ be a data matrix with $n$ samples in $d$
dimensions, and $\bY \in \{0,1\}^{n \times L}$ be a binary label matrix where
$y_{i\ell} = 1$ indicates that sample $i$ has label $\ell$. In the multilabel
setting, each row of $\bY$ may contain multiple ones. We denote:
\begin{itemize}
\item $\bmu = \frac{1}{n}\sum_{i=1}^n \mathbf{x}_i$ --- the global mean;
\item $\bmu_\ell = \frac{1}{n_\ell}\sum_{i: y_{i\ell}=1} \mathbf{x}_i$ --- the
  mean of samples with label $\ell$, where $n_\ell = \sum_i y_{i\ell}$;
\item $k_i = \sum_{\ell=1}^L y_{i\ell}$ --- the label cardinality of sample $i$;
\item $K = \sum_{i=1}^n k_i = \sum_{\ell=1}^L n_\ell$ --- the total label count;
\item $\bGamma = \bY^\top \bY \in \R^{L \times L}$ --- the label co-occurrence
  (Gram) matrix;
\item $\bD_n = \diag(n_1, \ldots, n_L)$ --- diagonal matrix of label frequencies;
\item $\bH = \bI_n - \frac{1}{n}\bone_n \bone_n^\top$ --- the centering matrix
  (the orthogonal projector onto
  $\bone_n^\perp = \{\mathbf{v} \in \R^n : \bone_n^\top\mathbf{v} = 0\}$,
  with $\rank(\bH) = n - 1$);
\item $\tilde{\bX} = \bH\bX$ --- the centered data matrix.
\end{itemize}
When a distributional model is assumed (Section~\ref{sec:distance}), we write
$\pi_\ell = \mathbb{E}[y_{i\ell}]$ for the marginal probability of label~$\ell$
and $\bD_\pi = \diag(\pi_1, \ldots, \pi_L)$.

\subsection{Multilabel Scatter Matrices}

Following \citet{wang2010}, the multilabel scatter matrices are defined as:
\begin{align}
\Sb &= \sum_{\ell=1}^L n_\ell (\bmu_\ell - \bmu)(\bmu_\ell - \bmu)^\top, \label{eq:Sb}\\
\Sw &= \sum_{\ell=1}^L \sum_{i:\, y_{i\ell}=1} (\mathbf{x}_i - \bmu_\ell)(\mathbf{x}_i - \bmu_\ell)^\top. \label{eq:Sw}
\end{align}
The standard (single-label) total scatter is $S_t = \sum_{i=1}^n (\mathbf{x}_i
- \bmu)(\mathbf{x}_i - \bmu)^\top = \tilde{\bX}^\top \tilde{\bX}$.

\begin{remark}
In the multilabel setting, a sample $\mathbf{x}_i$ with labels
$\{\ell_1, \ell_2, \ldots, \ell_{k_i}\}$ contributes to $\Sw$ once for
\emph{each} of its $k_i$ labels. This sample-reweighting is the fundamental
distinction from the single-label case and is the source of the phenomena we
analyze.
\end{remark}

\subsection{The Stiefel Manifold and Fisher Objectives}

The Stiefel manifold $\Stief(d,r) = \{\bW \in \R^{d \times r} : \bW^\top \bW =
\bI_r\}$ is a smooth compact manifold \citep{edelman1998,absil2008}. On it,
we consider four Fisher-type objectives with multilabel scatter matrices:
\begin{align}
J_{\mathrm{TR}}(\bW) &= \frac{\tr(\bW^\top \Sb \bW)}{\tr(\bW^\top \Sw \bW)},
  &\text{(trace ratio)} \label{eq:TR}\\
J_{\mathrm{RT}}(\bW) &= \tr\!\left((\bW^\top \Sw \bW)^{-1} \bW^\top \Sb \bW\right),
  &\text{(ratio trace)} \label{eq:RT}\\
J_{\mathrm{DR}}(\bW) &= \frac{|\bW^\top \Sb \bW|}{|\bW^\top \Sw \bW|},
  &\text{(determinant ratio)} \label{eq:DR}\\
J_{\mathrm{TD}}(\bW) &= \tr(\bW^\top \Sb \bW) - \tr(\bW^\top \Sw \bW).
  &\text{(trace difference)} \label{eq:TD}
\end{align}

\subsection{Distributional Assumptions and Population Quantities}

For the statistical guarantees in Section~\ref{sec:statistical}, we require
distributional assumptions on the data-generating process.

\begin{definition}[Linear label-effect model]\label{def:model}
A sample $(\mathbf{x}_i, \mathbf{y}_i)$ is generated as:
\begin{equation}\label{eq:model}
\mathbf{x}_i = \bmu + \sum_{\ell=1}^L y_{i\ell}\, \balpha_\ell + \bepsilon_i,
\end{equation}
where $\balpha_\ell \in \R^d$ is the additive effect of label $\ell$ on the
feature vector, $\bmu \in \R^d$ is the baseline mean, and $\bepsilon_i \sim
(0, \bSigma_w)$ is zero-mean noise with covariance $\bSigma_w \succ 0$,
independent of $\mathbf{y}_i$.
\end{definition}

Let $\bA = [\balpha_1, \ldots, \balpha_L] \in \R^{d \times L}$
denote the matrix of label effects.

\begin{definition}[Sub-Gaussian noise]\label{def:subgaussian}
The noise vector $\bepsilon_i$ in the model~\eqref{eq:model} is called
$\sigma$-sub-Gaussian if for every unit vector $\mathbf{u} \in \R^d$,
\[
\mathbb{E}\!\left[\exp(t\, \mathbf{u}^\top \bepsilon_i)\right]
\leq \exp\!\left(\frac{t^2 \sigma^2}{2}\right)
\quad \text{for all } t \in \R.
\]
In particular, Gaussian noise $\bepsilon_i \sim \mathcal{N}(\bzero,
\bSigma_w)$ is $\sqrt{\lambda_{\max}(\bSigma_w)}$-sub-Gaussian.
\end{definition}

\begin{definition}[Population scatter matrices]\label{def:pop-scatter}
Under the linear label-effect model~\eqref{eq:model} with label
distribution $\pi_\ell = \mathbb{E}[y_{i\ell}]$, define the population scatter
matrices:
\begin{align}
\Sbpop &= \bA \bD_\pi \bA^\top, \quad \bD_\pi = \diag(\pi_1, \ldots, \pi_L),
  \label{eq:Sbpop}\\
\Swpop &= K^{\mathrm{pop}} \bSigma_w, \quad K^{\mathrm{pop}} = \textstyle\sum_{\ell=1}^L \pi_\ell,
  \label{eq:Swpop}\\
\Mstar &= 2\Sbpop - \Stmlpop = \Sbpop - \Swpop,
  \label{eq:Mstar}
\end{align}
where $\Stmlpop = \Sbpop + \Swpop$ is the population multilabel total scatter.
The eigenvalue gap of $\Mstar$ at rank $r$ is
$\gap_r(\Mstar) = \lambda_r(\Mstar) - \lambda_{r+1}(\Mstar)$,
where $\lambda_1 \geq \lambda_2 \geq \cdots$ are the eigenvalues of $\Mstar$.
\end{definition}

\begin{remark}[Label co-occurrence structure]\label{rem:cooccurrence}
The multilabel \emph{label co-occurrence matrix}
$\bC = \mathbb{E}[\mathbf{y}_i\mathbf{y}_i^\top]$
has entries $C_{\ell\ell} = \pi_\ell$ and
$C_{\ell\ell'} = \pi_{\ell\ell'}
:= \Pr(y_{i\ell} = y_{i\ell'} = 1)$
for $\ell \neq \ell'$.
The \emph{label covariance} is
$\bSigmay = \bC - \bpi\bpi^\top$.
The finite-sample discriminant
$\hat\bM/n = (\Sb - \Sw)/n$ converges (in the noiseless limit)
to the \emph{centered population reference}
\begin{equation}\label{eq:Mstar-ml}
\Mstar_c = \bA\,\bQ_\pi\,\bA^\top - K^{\mathrm{pop}}\bSigma_w,
\end{equation}
where $\bQ_\pi = \bB_\pi - \bW_\pi$ is the \emph{net label discriminant matrix} with
$\bB_\pi = \bSigmay\,\bD_\pi^{-1}\,\bSigmay$ (population centered
between-class label scatter) and
$\bW_\pi = \sum_\ell \pi_\ell\,
\mathrm{Cov}(\mathbf{y} \mid y_\ell = 1)$ (population within-class
label scatter).  In the single-label case, $\bW_\pi = \bzero$ and
$\bQ_\pi = \bD_\pi - \bpi\bpi^\top$, so the top-$r$ eigenspace of
$\Mstar_c$ coincides with that of
$\Mstar = \bA\bD_\pi\bA^\top - K^{\mathrm{pop}}\bSigma_w$.
\end{remark}

\section{Rank Characterization of the Multilabel Between-Class Scatter}
\label{sec:rank}

The rank of the between-class scatter matrix determines the dimensionality of
the useful discriminant subspace and dictates whether the orthogonality
constraint is binding. In single-label LDA with $C$ classes, the classical
result is $\rank(S_b) \leq C-1$. The $-1$ in this bound comes from the
single-label identity $\sum_\ell n_\ell(\bmu_\ell - \bmu) = 0$, which holds
because every sample is counted exactly once across the classes
($\sum_\ell n_\ell \bmu_\ell = \sum_i \mathbf{x}_i = n\bmu$) and which
forces the $C$ vectors $(\bmu_\ell - \bmu)_{\ell=1}^C$ to span an at-most
$(C-1)$-dimensional subspace.
In the multilabel case a sample with $k_i$ labels contributes to
$k_i$ class-means, so
$\sum_\ell n_\ell \bmu_\ell = \sum_i k_i \mathbf{x}_i$, and this is no
longer proportional to $n\bmu$ as soon as the label cardinalities $k_i$
vary across samples.  The one rank reduction that produced the $-1$ in
the classical bound therefore disappears, and---as we show below---the
effective multilabel discriminant dimensionality can reach
$\rank(\bY)$ rather than $\rank(\bY)-1$.

\subsection{Matrix Factorization of $\Sb$}

We begin by expressing $\Sb$ in factored form.

\begin{lemma}[Factorization of $\Sb$]\label{lem:factor}
Let $\bM = [\sqrt{n_1}(\bmu_1 - \bmu), \ldots, \sqrt{n_L}(\bmu_L - \bmu)]
\in \R^{d \times L}$. Then $\Sb = \bM\bM^\top$. Moreover,
\begin{equation}\label{eq:M-factor}
\bM = \tilde{\bX}^\top \bY \bD_n^{-1/2},
\end{equation}
where $\tilde{\bX} = \bH\bX$ is the centered data matrix.
\end{lemma}

\begin{proof}
The identity $\Sb = \bM\bM^\top$ is immediate from the definition. For the
factorization, note that:
\[
\bmu_\ell - \bmu = \frac{1}{n_\ell}\bX^\top \mathbf{y}_\ell - \frac{1}{n}\bX^\top \bone_n
= \bX^\top\!\left(\frac{\mathbf{y}_\ell}{n_\ell} - \frac{\bone_n}{n}\right),
\]
where $\mathbf{y}_\ell$ is the $\ell$-th column of $\bY$. Thus,
\begin{align*}
\sqrt{n_\ell}(\bmu_\ell - \bmu)
&= \bX^\top\!\left(\frac{\mathbf{y}_\ell}{\sqrt{n_\ell}} - \frac{\sqrt{n_\ell}}{n}\bone_n\right).
\end{align*}
Define $\tilde{\mathbf{y}}_\ell =
\frac{\mathbf{y}_\ell}{\sqrt{n_\ell}} -
\frac{\sqrt{n_\ell}}{n}\bone_n$. We verify directly that
$\tilde{\mathbf{y}}_\ell = \bH
\frac{\mathbf{y}_\ell}{\sqrt{n_\ell}}$:
\[
\bH \frac{\mathbf{y}_\ell}{\sqrt{n_\ell}}
= \left(\bI_n - \tfrac{1}{n}\bone_n\bone_n^\top\right)
  \frac{\mathbf{y}_\ell}{\sqrt{n_\ell}}
= \frac{\mathbf{y}_\ell}{\sqrt{n_\ell}}
  - \frac{\bone_n^\top\mathbf{y}_\ell}{n\sqrt{n_\ell}}\,\bone_n
= \frac{\mathbf{y}_\ell}{\sqrt{n_\ell}}
  - \frac{\sqrt{n_\ell}}{n}\,\bone_n
= \tilde{\mathbf{y}}_\ell,
\]
where we used $\bone_n^\top\mathbf{y}_\ell = n_\ell$ (the number of
samples with label~$\ell$).  As a consistency check,
$\bone_n^\top \tilde{\mathbf{y}}_\ell = \sqrt{n_\ell} -
\sqrt{n_\ell} = 0$, confirming $\tilde{\mathbf{y}}_\ell \in
\bone_n^\perp = \col(\bH)$ (recall that $\bH$ is the orthogonal
projector onto $\bone_n^\perp$, so its column space is exactly the set of
vectors orthogonal to $\bone_n$).
Therefore, the matrix $\tilde{\bY} = [\tilde{\mathbf{y}}_1, \ldots,
\tilde{\mathbf{y}}_L] = \bH \bY \bD_n^{-1/2}$ and $\bM = \bX^\top
\tilde{\bY} = \bX^\top \bH \bY \bD_n^{-1/2} = \tilde{\bX}^\top \bY
\bD_n^{-1/2}$, since $\bH$ is symmetric and $\tilde{\bX} = \bH\bX$.
\end{proof}

\subsection{Rank Theorem}

\begin{theorem}[Rank of $\Sb$]\label{thm:rank}
Assume all labels are present ($n_\ell \geq 1$ for all $\ell$). Then:
\begin{enumerate}
\item[(i)] $\rank(\Sb) = \rank(\tilde{\bX}^\top \bY)$.
\item[(ii)] $\rank(\Sb) \leq \min(d,\; n-1,\; \rank(\bY))$.
\item[(iii)] If $\bone_n \in \col(\bY)$, then $\rank(\Sb) \leq \min(d,\; n-1,\; \rank(\bY) - 1)$.
\item[(iv)] If $\rank(\tilde{\bX}) \geq \rank(\bH\bY)$
  and $\col(\bH\bY)$ is in general position relative to
  $\col(\tilde{\bX})$
  (i.e.,
  $\col(\bH\bY) \cap \ker(\tilde{\bX}^\top) = \{\bzero\}$;
  see Appendix~\ref{app:rank} for discussion),
  then $\rank(\Sb) = \rank(\bH\bY)$.
\end{enumerate}
\end{theorem}

\begin{proof}
\textit{(i)} From Lemma~\ref{lem:factor}, $\rank(\Sb) = \rank(\bM\bM^\top) =
\rank(\bM) = \rank(\tilde{\bX}^\top \bY \bD_n^{-1/2})$. Since $\bD_n^{-1/2}$
is an invertible diagonal matrix (all $n_\ell \geq 1$), $\rank(\tilde{\bX}^\top
\bY \bD_n^{-1/2}) = \rank(\tilde{\bX}^\top \bY)$.

\textit{(ii)} By the rank inequality for matrix products,
$\rank(\tilde{\bX}^\top \bY) \leq \min(\rank(\tilde{\bX}),
\rank(\bY))$. Since $\tilde{\bX} = \bH\bX$ and $\bH$ has rank $n-1$,
we have $\rank(\tilde{\bX}) \leq n-1$. Also $\rank(\tilde{\bX}) \leq d$. Thus
$\rank(\Sb) \leq \min(d, n-1, \rank(\bY))$.

\textit{(iii)} The matrix $\bM = \tilde{\bX}^\top \bH\bY\bD_n^{-1/2}$, so
$\rank(\bM) \leq \rank(\bH\bY)$ (since $\bD_n^{-1/2}$ is invertible and
right-multiplication by an invertible matrix preserves rank). Now assume
$\bone_n \in \col(\bY)$. By the definition of the centering matrix,
$\bH\bone_n = (\bI_n - \frac{1}{n}\bone_n\bone_n^\top)\bone_n = \bone_n -
\bone_n = \bzero$; that is, $\bone_n \in \ker(\bH)$.
Since $\bone_n$ is a nonzero vector in both $\col(\bY)$ and $\ker(\bH)$,
applying $\bH$ to $\bY$ annihilates this direction: formally,
$\rank(\bH\bY) = \rank(\bY) - \dim(\col(\bY) \cap \ker(\bH))
\leq \rank(\bY) - 1$,
because $\ker(\bH) = \mathrm{span}(\bone_n)$ and
$\bone_n \in \col(\bY)$ gives $\dim(\col(\bY) \cap \ker(\bH)) \geq 1$.
Thus $\rank(\Sb) \leq \rank(\bH\bY) \leq \rank(\bY) - 1$, and
combining with part~(ii) yields $\rank(\Sb) \leq \min(d, n-1, \rank(\bY) - 1)$.

\textit{(iv)} We have $\bM = \tilde{\bX}^\top \bH\bY\bD_n^{-1/2}$.
Since $\bD_n^{-1/2}$ is invertible,
$\col(\bH\bY\bD_n^{-1/2}) = \col(\bH\bY)$. Moreover,
$\col(\bH\bY) \subseteq \col(\bH) = \bone_n^\perp$
(because $\bH$ is a projector: every column of $\bH\bY$ is of the form
$\bH\mathbf{v}$, which lies in $\col(\bH)$).

The rank of $\bM = \tilde{\bX}^\top (\bH\bY\bD_n^{-1/2})$ is the
dimension of the image of $\col(\bH\bY)$ under the linear map
$\tilde{\bX}^\top \colon \R^n \to \R^d$. By the rank--nullity theorem
applied to $\tilde{\bX}^\top$ restricted to $\col(\bH\bY)$:
\[
\rank(\bM)
= \dim\!\bigl(\tilde{\bX}^\top(\col(\bH\bY))\bigr)
= \rank(\bH\bY) - \dim\!\bigl(\col(\bH\bY) \cap \ker(\tilde{\bX}^\top)\bigr).
\]
The ``general-position'' condition (see Appendix~\ref{app:rank} for a
detailed discussion) states precisely that
$\col(\bH\bY) \cap \ker(\tilde{\bX}^\top) = \{\bzero\}$: no nonzero
vector in $\col(\bH\bY)$ is annihilated by $\tilde{\bX}^\top$.  Under
this condition the second term vanishes and
$\rank(\bM) = \rank(\bH\bY)$.
\end{proof}

\subsection{Consequences and Comparison with the Single-Label Case}

\begin{corollary}[Recovery of the classical bound]\label{cor:classical}
In the single-label setting where each sample has exactly one label (each row
of $\bY$ has exactly one nonzero entry), $\rank(\Sb) \leq C - 1$, where
$C = L$ is the number of classes, assuming all classes are represented.
Equality $\rank(\Sb) = C - 1$ holds when additionally
$\rank(\tilde{\bX}) \geq C - 1$ and the general position condition
of Theorem~\ref{thm:rank}(iv) is satisfied.
\end{corollary}

\begin{proof}
In the single-label case, $\bY\bone_L = \bone_n$ (each sample has exactly one
label), so $\bone_n \in \col(\bY)$. Moreover, $\rank(\bY) = C$ when all classes
are represented. By Theorem~\ref{thm:rank}(iii), $\rank(\Sb) \leq C - 1$.
\end{proof}

\begin{corollary}[Excess discriminant dimensions in multilabel]\label{cor:excess}
In the multilabel setting with $L$ labels:
\begin{enumerate}
\item[(i)] If $\bone_n \notin \col(\bY)$, $\rank(\bY) = L$,
  $\rank(\tilde{\bX}) \geq L$, and
  $\col(\bH\bY) \cap \ker(\tilde{\bX}^\top) = \{\bzero\}$
  (the general position condition of Theorem~\ref{thm:rank}(iv),
  which holds almost surely for continuous data distributions),
  then $\rank(\Sb) = L$, exceeding the single-label bound of $L - 1$.
\item[(ii)] More generally, $\rank(\Sb) \leq \rank(\bY)$, and the bound
  $\rank(\bY)$ depends on the label co-occurrence structure: it equals
  $\min(n, L)$ when the label vectors are linearly independent and may be
  strictly less when labels exhibit linear dependence.
\end{enumerate}
\end{corollary}

\begin{proof}
\textit{(i)} Suppose $\bone_n \notin \col(\bY)$ and $\rank(\bY) = L$.
Because $\bone_n \notin \col(\bY)$, the only vector in
$\col(\bY) \cap \ker(\bH)$ is $\bzero$ (recall $\ker(\bH) =
\mathrm{span}(\bone_n)$), so centering does not reduce the rank:
$\rank(\bH\bY) = \rank(\bY) = L$. Theorem~\ref{thm:rank}(iv) then
applies: provided the data is sufficiently rich
($\rank(\tilde{\bX}) \geq L$) and $\col(\bH\bY)$ is in general
position relative to $\col(\tilde{\bX})$ (which holds for continuous
data distributions almost surely; see Appendix~\ref{app:rank}), we
conclude $\rank(\Sb) = \rank(\bH\bY) = L$, exceeding the single-label
bound of $L - 1$.

\textit{(ii)} Follows directly from Theorem~\ref{thm:rank}(ii) and the
observation that $\rank(\bY) \leq \min(n, L)$.
\end{proof}

\begin{remark}[Connection to the label co-occurrence graph]
The label co-occurrence matrix $\bGamma = \bY^\top\bY$ encodes pairwise label
co-occurrence counts. Since $\rank(\bY) = \rank(\bGamma)$ (by the standard
identity $\rank(\bA^\top\bA) = \rank(\bA)$), the rank of $\Sb$ is ultimately
governed by the spectral structure of $\bGamma$. Specifically, $\rank(\Sb) \leq
\rank(\bGamma)$ when $\bone_n \notin \col(\bY)$, and $\rank(\Sb) \leq
\rank(\bGamma) - 1$ otherwise.
\end{remark}

\section{Objective Equivalence Under Multilabel Scatter}
\label{sec:equiv}

We now study when the four Fisher objectives~\eqref{eq:TR}--\eqref{eq:TD}
yield the same optimal projection under orthogonality constraints with
multilabel scatter matrices. In the single-label setting, \citet{luo2011}
established equivalence under orthonormality constraints. We extend this
analysis to the multilabel case, identifying both where the equivalence
is preserved and where it breaks down.

\subsection{Multilabel Partition of Variance}

The starting point is the relationship between $\Sb$, $\Sw$, and the
total scatter.

\begin{theorem}[Multilabel partition of variance]\label{thm:partition}
Assume every sample has at least one label ($k_i \geq 1$ for all $i$).
Define the \emph{multilabel total scatter} as
\begin{equation}\label{eq:Stml}
\Stml = \Sb + \Sw.
\end{equation}
Then:
\begin{enumerate}
\item[(i)] $\Stml = \sum_{i=1}^n k_i (\mathbf{x}_i - \bmu)(\mathbf{x}_i - \bmu)^\top$,
  where $k_i = \sum_\ell y_{i\ell}$ is the label cardinality of sample $i$.
\item[(ii)] The residual $\bR = \Stml - S_t = \sum_{i=1}^n (k_i - 1)
  (\mathbf{x}_i - \bmu)(\mathbf{x}_i - \bmu)^\top$ is positive semidefinite.
\item[(iii)] $\Stml = S_t$ if and only if $(k_i - 1)(\mathbf{x}_i - \bmu) =
  \bzero$ for all $i$; equivalently, $k_i = 1$ for every sample not at the
  global mean (the single-label case, up to degenerate exceptions).
\item[(iv)] $\Sb \npsd \Stml$ (i.e., $\Stml - \Sb = \Sw \psd 0$).
\end{enumerate}
\end{theorem}

\begin{proof}
\textit{(i)} For each label $\ell$, the standard single-class decomposition
gives:
\[
\sum_{i:\, y_{i\ell}=1} (\mathbf{x}_i - \bmu)(\mathbf{x}_i - \bmu)^\top
= n_\ell(\bmu_\ell - \bmu)(\bmu_\ell - \bmu)^\top
+ \sum_{i:\, y_{i\ell}=1} (\mathbf{x}_i - \bmu_\ell)(\mathbf{x}_i - \bmu_\ell)^\top.
\]
Summing the left-hand side over all $\ell = 1, \ldots, L$ gives
$\Sb + \Sw$ by the definitions~\eqref{eq:Sb}--\eqref{eq:Sw}.
For the right-hand side we exchange the order of summation: sample~$i$
appears in the inner sum for label~$\ell$ if and only if $y_{i\ell} = 1$,
so each outer product $(\mathbf{x}_i - \bmu)(\mathbf{x}_i - \bmu)^\top$
is counted once for every label that sample~$i$ possesses:
\begin{align*}
\Sb + \Sw
&= \sum_{\ell=1}^L \sum_{i:\, y_{i\ell}=1}
   (\mathbf{x}_i - \bmu)(\mathbf{x}_i - \bmu)^\top \\
&= \sum_{i=1}^n \left(\sum_{\ell=1}^L y_{i\ell}\right)
   (\mathbf{x}_i - \bmu)(\mathbf{x}_i - \bmu)^\top
= \sum_{i=1}^n k_i (\mathbf{x}_i - \bmu)(\mathbf{x}_i - \bmu)^\top.
\end{align*}

\textit{(ii)} Since $k_i \geq 1$ for all samples (every sample has at least
one label), we have $k_i - 1 \geq 0$, and $\bR$ is a sum of positive
semidefinite (PSD) rank-one matrices with non-negative coefficients.
Any conic (non-negative) combination of PSD matrices is itself PSD,
so $\bR \psd 0$.

\textit{(iii)} $\bR = 0$ iff $k_i = 1$ for all $i$ (assuming $\mathbf{x}_i
\neq \bmu$ for at least one $i$ with $k_i > 1$; more precisely, $\bR = 0$ iff
$(k_i - 1)(\mathbf{x}_i - \bmu) = \bzero$ for all $i$, which holds iff $k_i =
1$ for all $i$ not at the mean).

\textit{(iv)} $\Sw = \Stml - \Sb$ and $\Sw$ is a sum of positive semidefinite
matrices, hence $\Sw \psd 0$.
\end{proof}

\begin{remark}\label{rem:key-difference}
The identity $\Stml = \Sb + \Sw$ is formally the same as the classical
partition of variance $S_t = S_b + S_w$, but with a crucial difference: the
``total scatter'' $\Stml$ is a \emph{cardinality-weighted} version of the
standard total scatter $S_t$. Each sample is counted $k_i$ times rather than
once. This reweighting is the root cause of the phenomena we identify in the
remainder of this section.
\end{remark}

\subsection{Equivalence Under $\Stml$-Orthogonality}

We first consider the constraint $\bW^\top \Stml \bW = \bI_r$, which
generalizes the ULDA constraint \citep{jin2001,ye2005} to the multilabel
setting, and establish the multilabel analogue of the single-label
equivalence result of \citet{luo2011}.

\begin{theorem}[Objective equivalence under $\Stml$-orthogonality]\label{thm:equiv-St}
Assume $\Sw \succ 0$ (positive definite). Under the constraint $\bW^\top
\Stml \bW = \bI_r$ with $r \leq \rank(\Sb)$, the four objectives
$J_{\mathrm{TR}}$, $J_{\mathrm{RT}}$, $J_{\mathrm{DR}}$, $J_{\mathrm{TD}}$
attain their respective optima at the same $\bW$.
\end{theorem}

\begin{proof}
Since $\Sw \succ 0$ and $\Sb \psd 0$, we have
$\Stml = \Sb + \Sw \succeq \Sw \succ 0$, so $\Stml$ is positive
definite and the constraint $\bW^\top \Stml \bW = \bI_r$ is well-defined.
Under this constraint, let
$\bTheta = \bW^\top \Sb \bW$. Since $\Stml = \Sb + \Sw$
(Theorem~\ref{thm:partition}), we have
$\bW^\top \Sw \bW = \bW^\top \Stml \bW - \bW^\top \Sb \bW
= \bI_r - \bTheta$.

We first show that the eigenvalues $\theta_1 \geq \cdots \geq \theta_r$ of
$\bTheta$ lie in $[0,1)$. Since $\Sb \psd 0$ (it is a sum of PSD rank-one
matrices by definition~\eqref{eq:Sb}), we have
$\bTheta = \bW^\top \Sb \bW \psd 0$, giving $\theta_i \geq 0$.
Since $\Sw \succ 0$ by assumption and $\bW$ has full column rank (any
$\bW$ satisfying $\bW^\top \Stml \bW = \bI_r$ with $\Stml \succ 0$ must
have rank $r$), we get $\bW^\top \Sw \bW = \bI_r - \bTheta \succ 0$, which
means every eigenvalue of $\bI_r - \bTheta$ is strictly positive, i.e.,
$\theta_i < 1$ for all $i$.  Together: $0 \leq \theta_i < 1$.

We now express each objective in terms of $\theta_1, \ldots, \theta_r$.
Under the constraint,
$\tr(\bW^\top \Sb \bW) = \sum_i \theta_i$
and
$\tr(\bW^\top \Sw \bW)
= r - \sum_i \theta_i$.
The four objectives become:
\begin{alignat}{2}
J_{\mathrm{TR}} &= \frac{\sum_i \theta_i}{r - \sum_i \theta_i},
  &\quad&\text{monotone increasing in } \textstyle\sum_i \theta_i; \label{eq:TR-theta}\\
J_{\mathrm{RT}} &= \sum_{i=1}^r \frac{\theta_i}{1 - \theta_i},
  &&\text{Schur-convex on } [0,1)^r; \label{eq:RT-theta}\\
J_{\mathrm{DR}} &= \prod_{i=1}^r \frac{\theta_i}{1 - \theta_i};
  \label{eq:DR-theta}\\
J_{\mathrm{TD}} &= 2\sum_i \theta_i - r,
  &&\text{monotone increasing in } \textstyle\sum_i \theta_i. \label{eq:TD-theta}
\end{alignat}
For $J_{\mathrm{RT}}$, the expression follows from
$(\bW^\top \Sw \bW)^{-1} = (\bI_r - \bTheta)^{-1}$, which is
diagonal in the eigenbasis of $\bTheta$ with entries
$1/(1-\theta_i)$, so $\tr((\bI_r-\bTheta)^{-1}\bTheta) =
\sum_i \theta_i/(1-\theta_i)$.  For $J_{\mathrm{DR}}$,
$|\bW^\top\Sb\bW|/|\bW^\top\Sw\bW| =
|\bTheta|/|\bI_r - \bTheta| = \prod_i \theta_i/(1-\theta_i)$.

\medskip
\noindent\textit{Common optimizer via a change of variable.}
We now show that all four objectives share the same optimizer by
reducing to a Stiefel problem.  Set $\bW = (\Stml)^{-1/2}\bV$.
The constraint $\bW^\top \Stml \bW = \bI_r$ becomes
$\bV^\top(\Stml)^{-1/2}\Stml(\Stml)^{-1/2}\bV = \bV^\top \bV = \bI_r$,
so $\bV \in \Stief(d,r)$.

Let $\bP = {\Stml}^{-1/2} \Sb\, {\Stml}^{-1/2}$ with eigenvalues
$\lambda_1 \geq \cdots \geq \lambda_d$.  Then
\[
\bTheta = \bW^\top \Sb \bW
= \bV^\top {\Stml}^{-1/2}\Sb\,{\Stml}^{-1/2} \bV = \bV^\top \bP \bV.
\]
Since $\bP$ is PSD with eigenvalues in $[0,1)$ (by the same argument as
for $\bTheta$, applied to the full $d$-dimensional setting), the
Poincar\'{e} separation theorem \citep{horn2013} states that the
eigenvalues $(\theta_1, \ldots, \theta_r)$ of $\bV^\top \bP \bV$
satisfy the componentwise interlacing bounds
$\theta_i \leq \lambda_i$ for each
$i = 1, \ldots, r$, with equality when $\bV$ spans the top~$r$
eigenvectors of~$\bP$.

For $J_{\mathrm{TR}}$ and $J_{\mathrm{TD}}$, which depend only on $\sum_i
\theta_i$, the Ky Fan inequality \citep{fan1949} (which states
$\max_{\bV \in \Stief(d,r)} \tr(\bV^\top \bP \bV) = \sum_{i=1}^r \lambda_i$,
attained at the top $r$ eigenvectors) gives
$\sum_i \theta_i \leq \sum_{i=1}^r \lambda_i$, so both
are maximized by the top eigenvectors. For
$J_{\mathrm{RT}} = \sum_i \theta_i/(1-\theta_i)$, the Poincar\'{e}
separation theorem gives the componentwise bound $\theta_i \leq \lambda_i$
for each~$i$. Since $g(\theta) = \theta/(1-\theta)$ is increasing on
$[0,1)$, we obtain $g(\theta_i) \leq g(\lambda_i)$ for each~$i$, and
summing gives $J_{\mathrm{RT}} \leq \sum_{i=1}^r g(\lambda_i)$, with
equality at the top eigenvectors. (The same conclusion follows from the
Schur-convexity of $J_{\mathrm{RT}}$ \citep{marshall2011}; see
Appendix~\ref{app:schur} for details.) For $J_{\mathrm{DR}}$, the argument is more delicate because the
determinant ratio is a product rather than a sum of eigenvalue functions;
the full details are given in Lemma~\ref{lem:det-ratio}
(Appendix~\ref{app:schur}), and we summarize the key steps here.
Note that
$\bV^\top(\bI - \bP)\bV = \bI_r - \bV^\top\bP\bV$ (since $\bV^\top\bV
= \bI_r$), so its eigenvalues are $1 - \theta_i$. Thus
$J_{\mathrm{DR}} = \prod_i \theta_i/(1-\theta_i)$.
Observe that $J_{\mathrm{DR}} = 0$ whenever any $\theta_i = 0$: the product
vanishes if even one eigenvalue of $\bW^\top\Sb\bW$ is zero.
Conversely, since the assumption
$r \leq \rank(\Sb)$ ensures the top~$r$ eigenvalues
$\lambda_1 \geq \cdots \geq \lambda_r$ of $\bP$ are strictly positive,
the choice $\bV = [\mathbf{v}_1,\ldots,\mathbf{v}_r]$ (top eigenvectors)
achieves $\theta_i = \lambda_i > 0$ for all~$i$, giving
$J_{\mathrm{DR}} = \prod_i \lambda_i/(1-\lambda_i) > 0$.
Any other $\bV$ yielding some $\theta_i = 0$ is strictly suboptimal.
For $\bV$ with all $\theta_i > 0$, the Poincar\'{e}
separation theorem gives $\theta_i \leq \lambda_i$ componentwise, and since
$t/(1-t)$ is increasing on $[0,1)$, each factor satisfies
$\theta_i/(1-\theta_i) \leq \lambda_i/(1-\lambda_i)$, so the product is
maximized when $\theta_i = \lambda_i$ for all $i$, i.e., when $\bV$ spans
the top $r$ eigenvectors of $\bP$.

All four objectives are thus simultaneously maximized when the vector
$(\theta_1, \ldots, \theta_r)$ equals the $r$ largest eigenvalues of $\bP$.
The corresponding optimal $\bW$ is:
\begin{equation}\label{eq:optimal-W-St}
\bW^* = (\Stml)^{-1/2} [\mathbf{v}_1, \ldots, \mathbf{v}_r],
\end{equation}
where $\mathbf{v}_1, \ldots, \mathbf{v}_r$ are the top $r$ eigenvectors of
$(\Stml)^{-1/2} \Sb (\Stml)^{-1/2}$, which is equivalent to the top $r$
generalized eigenvectors of $\Sb \mathbf{w} = \lambda \Stml \mathbf{w}$.
\end{proof}

\subsection{Equivalence and Inequivalence Under the Stiefel Constraint}

We now turn to the Stiefel constraint $\bW^\top \bW = \bI_r$. Here the
situation is fundamentally different because $\bW^\top \Stml \bW$ is no longer
fixed---it varies over the Stiefel manifold.

\begin{theorem}[Structure under uniform label cardinality]\label{thm:equiv-uniform}
If all label cardinalities are equal, i.e., $k_i = k$ for all $i$ and some
constant $k \geq 1$, then $\Stml = k \cdot S_t$ and:
\begin{enumerate}
\item[(i)] Under the $\Stml$-orthogonality constraint $\bW^\top \Stml \bW =
  \bI_r$ (equivalently, $\bW^\top S_t \bW = \frac{1}{k}\bI_r$), all four
  objectives share the same optimizer (by Theorem~\ref{thm:equiv-St}).
\item[(ii)] Under $\bW^\top \bW = \bI_r$, the trace difference optimizer is
  the top $r$ eigenspace of $2\Sb - kS_t$, and the trace ratio optimizer is
  the top $r$ eigenspace of $(1 + \lambda^*)\Sb - \lambda^* k S_t$ where
  $\lambda^*$ is the optimal trace ratio value. Two logically distinct
  conditions govern when the four objectives share a common optimizer:
  \begin{enumerate}
  \item[\textup{(a)}] \emph{Stiefel feasibility of the unconstrained
    optimizer:} When $\Sb$ and $S_t$ commute
    ($\Sb S_t = S_t \Sb$), all four objectives have optimizers drawn from
    the common eigenbasis. In particular, the unconstrained RT and DR
    maximizer (the top~$r$ generalized eigenvectors of $(\Sb, \Sw)$)
    is already orthonormal and hence Stiefel-feasible, so the
    Stiefel-constrained and unconstrained optima coincide for RT and DR.
  \item[\textup{(b)}] \emph{Eigenvalue ordering consistency:} Even under
    commutativity, the eigenvalue \emph{orderings} of
    $\alpha\Sb - \beta S_t$ can depend on $\alpha/\beta$ (because the
    eigenvalues of $S_t$ vary across eigenvectors). The TD and TR subspaces
    coincide when, additionally, no eigenvalue crossing occurs at position~$r$
    between the two scalings---i.e., the top~$r$ eigenvectors of
    $2\Sb - kS_t$ are the same as those of
    $(1+\lambda^*)\Sb - \lambda^* k S_t$. A clean sufficient condition for
    both~(a) and~(b) simultaneously is $S_t = c\bI_d$ for some $c > 0$
    (isotropic data after centering), which makes all eigenvalue orderings
    of $\alpha\Sb - \beta S_t$ consistent regardless of $\alpha/\beta$.
    If, additionally, $r \leq \rank(\Sb)$ and the generalized eigenvalue
    ordering of $(\Sb, \Sw)$ agrees with that of $2\Sb - kS_t$ at
    position~$r$, the ratio trace and determinant ratio also share this
    common optimal subspace.
  \end{enumerate}
\end{enumerate}
\end{theorem}

\begin{proof}
\textit{(i)} When $k_i = k$ for all~$i$, Theorem~\ref{thm:partition}(i) gives
$\Stml = k S_t$, so $\Stml$-orthogonality ($\bW^\top \Stml \bW = \bI_r$) is
equivalent to $\bW^\top S_t \bW = \frac{1}{k}\bI_r$.
Theorem~\ref{thm:equiv-St} applies directly.

\textit{(ii)} Under the uniform-cardinality assumption, $\Sw = \Stml - \Sb
= kS_t - \Sb$ (using part~(i)). The trace difference therefore becomes:
\begin{align*}
J_{\mathrm{TD}}(\bW)
&= \tr(\bW^\top\Sb\bW) - \tr(\bW^\top\Sw\bW)
= \tr\!\bigl(\bW^\top(\Sb - (kS_t - \Sb))\bW\bigr) \\
&= \tr\!\bigl(\bW^\top(2\Sb - kS_t)\bW\bigr).
\end{align*}
The \emph{Ky Fan maximum principle} \citep{fan1949} states that for any
symmetric matrix $\bM$ with eigenvalues $\lambda_1 \geq \cdots \geq
\lambda_d$, the maximum of $\tr(\bV^\top \bM \bV)$ over $\bV \in
\Stief(d,r)$ equals $\sum_{i=1}^r \lambda_i$ and is attained when $\bV$
spans the top~$r$ eigenvectors of $\bM$.  Applying this with
$\bM = 2\Sb - kS_t$, the TD objective under $\bW^\top\bW = \bI_r$ is
maximized by the top~$r$ eigenvectors of $2\Sb - kS_t$.

For the trace ratio, the equivalent formulation \citep{ngo2012} gives the
optimal subspace as the top $r$ eigenspace of $\Sb - \lambda^*\Sw = (1 +
\lambda^*)\Sb - \lambda^* k S_t$, where $\lambda^*$ is the optimal trace
ratio value. Both TD and TR optimizers are top eigenspaces of matrices
of the form $\alpha\Sb - \beta S_t$ with $\alpha, \beta > 0$.

When $\Sb$ and $S_t$ commute ($\Sb S_t = S_t \Sb$), the spectral theorem
guarantees that they share a common orthonormal eigenbasis
$\{\mathbf{v}_1, \ldots, \mathbf{v}_d\}$ with $\Sb\mathbf{v}_i =
s_i\mathbf{v}_i$ and $S_t\mathbf{v}_i = \mu_i\mathbf{v}_i$.
Any linear combination $\alpha\Sb - \beta S_t$ has the \emph{same}
eigenvectors with eigenvalues $\alpha s_i - \beta\mu_i$.  However,
because $\mu_i$ varies across eigenvectors, the ordering of
$\alpha s_i - \beta\mu_i$ can change with $\alpha/\beta$: two
eigenvectors $\mathbf{v}_i$ and $\mathbf{v}_j$ swap their relative
ranking whenever $\alpha/\beta$ crosses the threshold
$({\mu_i - \mu_j})/({s_i - s_j})$ (provided $s_i \neq s_j$).
Consequently, the top~$r$ eigenspace of $\alpha\Sb - \beta S_t$ is
preserved across all positive $\alpha/\beta$ if and only if no such
crossing occurs at position~$r$.  When $S_t = c\bI_d$ for some
$c > 0$, all $\mu_i$ are equal and no crossing is possible, so the
top~$r$ subspace is independent of $\alpha/\beta$.

For the ratio trace and determinant ratio under $\bW^\top\bW = \bI_r$:
the matrix $\Sw = kS_t - \Sb$ also
commutes with both (since $\Sw$ is a linear combination of commuting
matrices), so $\Sb$, $\Sw$, and $S_t$ share the common eigenbasis, and the
generalized eigenvectors of $\Sb\mathbf{w} = \lambda\Sw\mathbf{w}$ are
precisely these common (orthonormal) eigenvectors with generalized
eigenvalues $s_i/(k\mu_i - s_i)$.  The \emph{unconstrained}
maximum of $J_{\mathrm{RT}}$ over all full-rank
$\bW \in \R^{d \times r}$ equals $\sum_{i=1}^r \lambda_i(\Sb, \Sw)$
(the sum of the top $r$ generalized eigenvalues), attained at the top
$r$ generalized eigenvectors.  Since these eigenvectors are orthonormal,
the unconstrained maximizer already satisfies
$\bW^\top\bW = \bI_r$---it lies on the Stiefel manifold.  Consequently
the Stiefel-constrained maximum of $J_{\mathrm{RT}}$ equals the
unconstrained maximum and is attained at the same $\bW$ (because the
constrained maximum is bounded above by the unconstrained one, and the
unconstrained maximizer is feasible).  The same argument applies to
$J_{\mathrm{DR}}$.

The RT/DR optimizer (the top~$r$ generalized eigenspace) coincides with
the TD optimizer (the top~$r$ eigenspace of $2\Sb - kS_t$) when the
ordering of the generalized eigenvalues $s_i/(k\mu_i - s_i)$ agrees
with the ordering of $2s_i - k\mu_i$ at position~$r$.  Note that for
\emph{fixed}~$\mu_i$, the map $t \mapsto t/(k\mu_i - t)$ is increasing
for $t < k\mu_i$, so both orderings are consistent with that of $s_i$.
However, when $\mu_i$ varies across eigenvectors, the orderings of
$s_i/(k\mu_i - s_i)$ and $2s_i - k\mu_i$ can disagree.  When
$S_t = c\bI_d$, all $\mu_i = c$ and the orderings are automatically
consistent, ensuring all four objectives share the same optimal
subspace.
\end{proof}

\begin{remark}
The commutativity condition $\Sb S_t = S_t \Sb$ holds whenever: (a)
$S_t = c\bI$ for some $c > 0$ (isotropic data); (b) all within-class
covariances are equal and proportional to $S_t$; or (c) the generalized
eigenvectors of $(\Sb, S_t)$ happen to be orthonormal. Without
commutativity, the eigenvectors of $\alpha\Sb - \beta S_t$ vary
with $\alpha/\beta$, causing TR and TD to have different optimal
subspaces.  With commutativity, the eigenvectors are shared but the
eigenvalue ordering can still depend on $\alpha/\beta$ when
$\mu_i = \lambda_i(S_t)$ are non-uniform, so commutativity alone does
not guarantee a common optimizer; the ordering-consistency condition in
Theorem~\ref{thm:equiv-uniform}(ii) is needed.  The divergence between
objectives is controlled by the spectral gap, as quantified in
Theorem~\ref{thm:inequiv}.
\end{remark}

\begin{theorem}[Inequivalence under variable label cardinality]\label{thm:inequiv}
When label cardinalities $k_i$ vary across samples, the trace difference
objective under $\bW^\top \bW = \bI_r$ may yield a strictly different optimal
subspace from the ratio trace objective. Specifically:
\begin{enumerate}
\item[(i)] The trace difference optimizer is the top $r$ eigenspace of
  $\Sb - \Sw = 2\Sb - \Stml$.
\item[(ii)] The ratio trace optimizer is related to the top $r$ generalized
  eigenspace of $(\Sb, \Sw)$.
\item[(iii)] Assume that $2\Sb - S_t$ has a positive spectral gap
  $\gamma = \lambda_r(2\Sb - S_t) - \lambda_{r+1}(2\Sb - S_t) > 0$
  and that $\|\bR\|_2 < \gamma$ (the perturbation is smaller than the gap).
  Then the trace difference optimal subspace is perturbed from its
  single-label counterpart by:
  \begin{equation}\label{eq:divergence-bound}
  \sin\angle(\mathcal{U}_{\mathrm{TD}}, \mathcal{U}_{\mathrm{TD}}^{(0)})
  \leq \frac{\|\bR\|_2}{\gamma},
  \end{equation}
  where $\mathcal{U}_{\mathrm{TD}}^{(0)}$ is the TD optimizer for the
  single-label-like matrix $2\Sb - S_t$ (i.e., without cardinality
  reweighting) and $\bR = \Stml - S_t$ is the residual from
  Theorem~\ref{thm:partition}(ii).
\end{enumerate}
\end{theorem}

\begin{proof}
\textit{(i)} Under $\bW^\top \bW = \bI_r$,
$J_{\mathrm{TD}}(\bW) = \tr(\bW^\top(\Sb - \Sw)\bW)$.
Using the partition $\Sw = \Stml - \Sb$ (Theorem~\ref{thm:partition}),
we substitute:
\[
\Sb - \Sw = \Sb - (\Stml - \Sb) = 2\Sb - \Stml,
\]
so $J_{\mathrm{TD}}(\bW) = \tr(\bW^\top(2\Sb - \Stml)\bW)$.
(The factor of $2$ in front of $\Sb$ arises because the between-class scatter
contributes positively both directly and via the subtraction of $\Sw = \Stml - \Sb$.)
By the Ky Fan maximum principle \citep{fan1949} (see the statement in the proof
of Theorem~\ref{thm:equiv-uniform}),
this is maximized by the top $r$ eigenvectors
of $2\Sb - \Stml$.

\textit{(ii)} The ratio trace $J_{\mathrm{RT}}(\bW) = \tr((\bW^\top \Sw
\bW)^{-1}\bW^\top \Sb \bW)$ under $\bW^\top \bW = \bI_r$ does not simplify to
a linear eigenvalue problem. Its stationary points satisfy a nonlinear
matrix equation involving both $\Sb$ and $\Sw$
\citep{absil2008,edelman1998}. The optimal subspace is related to (but
generally different from) the span of the top $r$ generalized eigenvectors of
$\Sb\mathbf{w} = \lambda \Sw\mathbf{w}$.

\textit{(iii)} The multilabel trace difference matrix is $2\Sb - \Stml = 2\Sb
- S_t - \bR$, a perturbation of the single-label-like matrix $2\Sb - S_t$ by
$-\bR$. By the Davis--Kahan $\sin\Theta$ theorem \citep{davis1970,yu2015}, the
perturbation in the top $r$ eigenspace satisfies
$\sin\angle(\mathcal{U}_{\mathrm{TD}}, \mathcal{U}_{\mathrm{TD}}^{(0)}) \leq
\|\bR\|_2/\gamma$, where $\gamma$ is the spectral gap of $2\Sb - S_t$ and
$\mathcal{U}_{\mathrm{TD}}^{(0)}$ is the TD optimizer when $\bR = 0$. In the
single-label case, $\bR = 0$ and the bound vanishes, confirming that
cardinality reweighting is the sole source of perturbation to the TD
eigenspace.
\end{proof}

\begin{corollary}[Spectral characterization of divergence]%
\label{cor:divergence}
The spectral norm of the residual satisfies
\[
\|\bR\|_2 \leq \max_i (k_i - 1) \cdot \lambda_{\max}(S_t^{(K)}),
\]
where
$S_t^{(K)}
= \sum_{i:\, k_i > 1}
(\mathbf{x}_i - \bmu)(\mathbf{x}_i - \bmu)^\top$
is the unweighted total scatter restricted to multi-labeled
samples. In particular, $\|\bR\|_2 = 0$ in the single-label case, and
the divergence~\eqref{eq:divergence-bound} vanishes.
\end{corollary}

\begin{proof}
From Theorem~\ref{thm:partition}(ii),
$\bR = \sum_{i=1}^n (k_i - 1)(\mathbf{x}_i - \bmu)(\mathbf{x}_i -
\bmu)^\top$.  Each term $(k_i - 1)(\mathbf{x}_i - \bmu)(\mathbf{x}_i -
\bmu)^\top$ is a non-negative scalar ($k_i - 1 \geq 0$) times a rank-one
positive semidefinite (PSD) matrix
$(\mathbf{x}_i - \bmu)(\mathbf{x}_i - \bmu)^\top$, so $\bR \psd 0$.

To bound $\|\bR\|_2 = \max_{\|\mathbf{u}\|=1} \mathbf{u}^\top \bR\,
\mathbf{u}$, we compute the Rayleigh quotient for an arbitrary unit
vector $\mathbf{u}$:
\begin{align*}
\mathbf{u}^\top \bR\, \mathbf{u}
&= \sum_{i=1}^n (k_i - 1)\bigl(\mathbf{u}^\top(\mathbf{x}_i - \bmu)\bigr)^2.
\end{align*}
Since only terms with $k_i > 1$ contribute (the others have coefficient~$0$),
and each coefficient satisfies $k_i - 1 \leq \max_j(k_j - 1)$:
\begin{align*}
\mathbf{u}^\top \bR\, \mathbf{u}
&\leq \max_j(k_j - 1) \cdot \sum_{i:\, k_i > 1}
  \bigl(\mathbf{u}^\top(\mathbf{x}_i - \bmu)\bigr)^2
= \max_j(k_j - 1) \cdot \mathbf{u}^\top S_t^{(K)} \mathbf{u},
\end{align*}
where $S_t^{(K)} = \sum_{i:\, k_i > 1} (\mathbf{x}_i -
\bmu)(\mathbf{x}_i - \bmu)^\top$ is the unweighted scatter restricted
to multi-labeled samples.  Maximizing over unit $\mathbf{u}$ gives
$\|\bR\|_2 \leq \max_i(k_i - 1) \cdot \lambda_{\max}(S_t^{(K)})$.
In the single-label case, $k_i = 1$ for all $i$, so every coefficient
is zero and $\bR = 0$.
\end{proof}

\section{Label-Distance Preservation}
\label{sec:distance}

We now establish that orthogonal multilabel projections preserve the
relationship between feature-space distances and label-space distances, in a
spirit related to---but distinct from---the Johnson--Lindenstrauss lemma
\citep{jl1984}. Unlike JL, which gives probabilistic guarantees for random
projections, our bounds are deterministic and relate projected distances to
\emph{label-space} distances. This result provides a theoretical foundation
for using orthogonally projected features in multilabel nearest-neighbor
classification.

\subsection{Setup}

We work with the linear label-effect model (Definition~\ref{def:model}),
under which the label-conditional mean of sample $i$
is $\bmu + \bA\mathbf{y}_i$.

\subsection{Distance Preservation Bound}

\begin{theorem}[Label-distance preservation]\label{thm:distance}
Under the linear label-effect model (Definition~\ref{def:model}), let $\bW \in
\Stief(d, r)$ be an orthogonal projection. For any two samples $i, j$,
the projected squared distance decomposes as:
\begin{equation}\label{eq:distance-decomp}
\|\bW^\top(\mathbf{x}_i - \mathbf{x}_j)\|^2
= \|\bW^\top\bA\bdelta_{ij}\|^2
  + 2\bdelta_{ij}^\top\bA^\top\bW\bW^\top(\bepsilon_i - \bepsilon_j)
  + \|\bW^\top(\bepsilon_i - \bepsilon_j)\|^2,
\end{equation}
where $\bdelta_{ij} = \mathbf{y}_i - \mathbf{y}_j$. Taking expectations
over the noise yields the two-sided bound:
\begin{equation}\label{eq:distance-expected}
\sigma_{\min}^2(\bW^\top\bA)\, d_H(\mathbf{y}_i, \mathbf{y}_j) + C_w
\leq \mathbb{E}\!\left[\|\bW^\top(\mathbf{x}_i - \mathbf{x}_j)\|^2\right]
\leq \sigma_{\max}^2(\bW^\top\bA)\, d_H(\mathbf{y}_i, \mathbf{y}_j) + C_w,
\end{equation}
where $d_H(\mathbf{y}_i, \mathbf{y}_j) = \|\mathbf{y}_i -
\mathbf{y}_j\|_1$ is the Hamming distance between label vectors,
$\sigma_{\max}^2(\bW^\top\bA)$ is the largest squared singular value of
the full matrix $\bW^\top\bA \in \R^{r \times L}$,
$C_w = 2\tr(\bW^\top \bSigma_w \bW)$, and
$\sigma_{\min}^2(\bW^\top\bA)$ in the lower bound denotes the smallest squared
singular value of $\bW^\top\bA$. More precisely, the lower bound follows from
relaxing $\sigma_{\min}^2(\bW^\top\bA_S) \geq \sigma_{\min}^2(\bW^\top\bA)$
where $\bA_S$ is the submatrix of $\bA$ restricted to the columns indexed
by the support $S = \{\ell : \delta_{ij,\ell} \neq 0\}$ of $\bdelta_{ij}$
(see the proof below). When $r \geq \rank(\bA)$ and
$\col(\bA) \subseteq \col(\bW)$ (see Proposition~\ref{prop:sufficient-r}),
$\sigma_{\min}(\bW^\top\bA) = \sigma_{\min}(\bA) > 0$, yielding a
nontrivial lower bound for all label pairs.
\end{theorem}

\begin{proof}
Under the model~\eqref{eq:model}, $\mathbf{x}_i - \mathbf{x}_j =
\bA\bdelta_{ij} + (\bepsilon_i - \bepsilon_j)$, where $\bdelta_{ij} =
\mathbf{y}_i - \mathbf{y}_j$.  Projecting with $\bW$ gives
$\bW^\top(\mathbf{x}_i - \mathbf{x}_j)
= \bW^\top\bA\bdelta_{ij} + \bW^\top(\bepsilon_i - \bepsilon_j)$.
Writing $\mathbf{a} = \bW^\top\bA\bdelta_{ij}$ (deterministic signal) and
$\mathbf{b} = \bW^\top(\bepsilon_i - \bepsilon_j)$ (random noise), the
squared norm expands as
\[
\|\mathbf{a} + \mathbf{b}\|^2
= \|\mathbf{a}\|^2 + 2\mathbf{a}^\top\mathbf{b} + \|\mathbf{b}\|^2,
\]
which is exactly~\eqref{eq:distance-decomp}.

Taking expectations over $\bepsilon_i, \bepsilon_j$: the signal term
$\|\mathbf{a}\|^2 = \|\bW^\top\bA\bdelta_{ij}\|^2$ is deterministic.
The cross-term
$\mathbb{E}[2\mathbf{a}^\top\mathbf{b}]
= 2\,\bdelta_{ij}^\top\bA^\top\bW\,\mathbb{E}[\bW^\top(\bepsilon_i -
\bepsilon_j)]$
vanishes because $\bepsilon_i$ and $\bepsilon_j$ are independent with zero mean
(so $\mathbb{E}[\bepsilon_i - \bepsilon_j] = \bzero$), while the vectors
$\bdelta_{ij}$, $\bA$, and $\bW$ are all deterministic. For the noise term:
\[
\mathbb{E}[\|\mathbf{b}\|^2]
= \mathbb{E}[\tr(\bW^\top(\bepsilon_i - \bepsilon_j)
  (\bepsilon_i - \bepsilon_j)^\top\bW)]
= \tr(\bW^\top\,\mathbb{E}[(\bepsilon_i - \bepsilon_j)
  (\bepsilon_i - \bepsilon_j)^\top]\,\bW)
= 2\tr(\bW^\top\bSigma_w\bW),
\]
where the last step uses independence ($\mathrm{Cov}(\bepsilon_i,
\bepsilon_j) = 0$) and identical covariance
($\mathrm{Cov}(\bepsilon_i) = \bSigma_w$). Hence
$\mathbb{E}[\|\bW^\top(\mathbf{x}_i - \mathbf{x}_j)\|^2]
= \|\bW^\top\bA\bdelta_{ij}\|^2 + C_w$ with
$C_w = 2\tr(\bW^\top\bSigma_w\bW)$.

For the signal term, $\bdelta_{ij} = \mathbf{y}_i - \mathbf{y}_j \in
\{-1, 0, 1\}^L$ with $\|\bdelta_{ij}\|^2 = d_H(\mathbf{y}_i,
\mathbf{y}_j)$ (each nonzero entry is $\pm 1$, so
$|\delta_{ij,\ell}|^2 = |\delta_{ij,\ell}|$ for every component,
giving $\|\bdelta_{ij}\|^2 = \|\bdelta_{ij}\|_1 = d_H$).  Let $S \subseteq \{1, \ldots, L\}$ denote
the \emph{support} of $\bdelta_{ij}$ (the set of labels on which samples $i$
and $j$ differ), and let $\bA_S$ be the submatrix of $\bA$ consisting of the
columns indexed by $S$.  Then $\bW^\top\bA\bdelta_{ij} =
\bW^\top\bA_S\bdelta_{ij,S}$, where $\bdelta_{ij,S} \in \{-1,1\}^{|S|}$.
By the Rayleigh quotient characterization of singular values:
\[
\sigma_{\min}^2(\bW^\top\bA_S)\,\|\bdelta_{ij,S}\|^2
\leq \|\bW^\top\bA_S\bdelta_{ij,S}\|^2
\leq \sigma_{\max}^2(\bW^\top\bA_S)\,\|\bdelta_{ij,S}\|^2,
\]
and since
$\|\bdelta_{ij,S}\|^2 = |S| = d_H(\mathbf{y}_i,\mathbf{y}_j)$,
$\sigma_{\min}^2(\bW^\top\bA_S)
\geq \sigma_{\min}^2(\bW^\top\bA)$, and
$\sigma_{\max}^2(\bW^\top\bA_S)
\leq \sigma_{\max}^2(\bW^\top\bA)$
(for the upper bound, $\bA_S^\top\bW\bW^\top\bA_S$ is a principal
submatrix of $\bA^\top\bW\bW^\top\bA$, so Cauchy interlacing gives
$\sigma_i(\bW^\top\bA) \geq \sigma_i(\bW^\top\bA_S)$, yielding
$\sigma_{\max}(\bW^\top\bA_S)
\leq \sigma_{\max}(\bW^\top\bA)$;
for the lower bound, any unit vector
$\mathbf{v} \in \R^{|S|}$ embeds as a unit vector
$\bar{\mathbf{v}} \in \R^L$ (by zero-padding on~$S^c$) with
$\bW^\top\bA_S\mathbf{v} = \bW^\top\bA\bar{\mathbf{v}}$, so
\[
\sigma_{\min}(\bW^\top\bA_S)
= \min_{\|\mathbf{v}\|=1}\|\bW^\top\bA_S\mathbf{v}\|
\geq \min_{\|\mathbf{u}\|=1}\|\bW^\top\bA\mathbf{u}\|
= \sigma_{\min}(\bW^\top\bA);
\]
see also \citealt{horn2013}), we obtain
the two-sided bound~\eqref{eq:distance-expected}.
\end{proof}

\subsection{The Role of Orthogonality}

The orthogonality constraint $\bW^\top \bW = \bI_r$ provides a direct benefit
for label-distance preservation by controlling the noise term.

\begin{corollary}[Noise control under orthogonality]\label{cor:noise}
For $\bW \in \Stief(d,r)$:
\begin{equation}\label{eq:noise-bound}
2r\,\lambda_{\min}(\bSigma_w) \leq C_w = 2\tr(\bW^\top\bSigma_w\bW) \leq 2r\,\lambda_{\max}(\bSigma_w).
\end{equation}
In contrast, for unconstrained $\bW$ (not on the Stiefel manifold), the noise
term $C_w$ is unbounded above.
\end{corollary}

\begin{proof}
For $\bW \in \Stief(d,r)$, the eigenvalues of $\bW^\top\bSigma_w\bW$ are bounded
between $\lambda_{\min}(\bSigma_w)$ and $\lambda_{\max}(\bSigma_w)$ by the
Poincar\'{e} separation theorem \citep{horn2013}, since $\bW^\top\bSigma_w\bW$ is a
compression of $\bSigma_w$ to the $r$-dimensional subspace $\col(\bW)$. Thus $\tr(\bW^\top
\bSigma_w\bW) \in [r\lambda_{\min}(\bSigma_w), r\lambda_{\max}(\bSigma_w)]$.
For unconstrained $\bW$, scaling $\bW \to c\bW$ scales $C_w$ by $c^2$.
\end{proof}

The decomposition in Theorem~\ref{thm:distance} naturally suggests
a signal-to-noise ratio for pairwise discriminability.
Recall that in classical (single-label) LDA, Fisher's discriminant
ratio $J_F(\bW) = \tr(\bW^\top S_b \bW) / \tr(\bW^\top S_w \bW)$
quantifies the \emph{global} discriminability of the projection: the
ratio of between-class to within-class variance in the projected space
\citep{fisher1936,fukunaga1990}.  In the multilabel setting, there is no
single between-class distance---discriminability depends on \emph{which
pair} of label vectors is being compared.  Theorem~\ref{thm:distance}
provides exactly this pairwise decomposition:
$\mathbb{E}[\|\bW^\top(\mathbf{x}_i - \mathbf{x}_j)\|^2]
= \underbrace{\|\bW^\top\bA\bdelta_{ij}\|^2}_{\text{signal}} +
\underbrace{C_w}_{\text{noise}}$,
where the signal term depends on the specific label difference
$\bdelta_{ij} = \mathbf{y}_i - \mathbf{y}_j$ and the noise term is
pair-independent.  Dividing signal by noise gives the pairwise analogue
of $J_F$.

\begin{definition}[Pairwise discriminant SNR]\label{def:snr}
Under the linear label-effect model (Definition~\ref{def:model}), for a
projection $\bW \in \Stief(d,r)$ and samples $i, j$, the \emph{pairwise
discriminant signal-to-noise ratio} is:
\begin{equation}\label{eq:snr-def}
\mathrm{SNR}_{ij}(\bW)
= \frac{\|\bW^\top\bA\bdelta_{ij}\|^2}{C_w}
= \frac{\|\bW^\top\bA\bdelta_{ij}\|^2}{2\,\tr(\bW^\top\bSigma_w\bW)},
\end{equation}
the ratio of the expected squared projected signal (the
label-dependent component of the expected projected distance) to the
expected projected noise contribution~$C_w$ (the label-independent
component).
When specialized to the single-label case with two classes
($L = 2$, $\bdelta_{ij} = \mathbf{e}_1 - \mathbf{e}_2$), this
reduces to
$\mathrm{SNR}_{ij}(\bW)
= \|\bW^\top(\boldsymbol\mu_1 - \boldsymbol\mu_2)\|^2
/ (2\tr(\bW^\top\bSigma_w\bW))$,
which is (up to a constant factor)
the classical Fisher criterion for the two-class projected distance.
\end{definition}

\begin{corollary}[Pairwise discriminant SNR bounds]\label{cor:snr}
Under the conditions of Theorem~\ref{thm:distance}, the pairwise
discriminant SNR (Definition~\ref{def:snr}) satisfies:
\begin{equation}\label{eq:snr-bounds}
\frac{\sigma_{\min}^2(\bW^\top\bA)\, d_H(\mathbf{y}_i, \mathbf{y}_j)}
     {2r\,\lambda_{\max}(\bSigma_w)}
\leq \mathrm{SNR}_{ij}(\bW)
\leq \frac{\sigma_{\max}^2(\bW^\top\bA)\, d_H(\mathbf{y}_i, \mathbf{y}_j)}
     {2r\,\lambda_{\min}(\bSigma_w)}.
\end{equation}
In particular, $\mathrm{SNR}_{ij}(\bW)$ grows linearly with the Hamming
distance $d_H(\mathbf{y}_i, \mathbf{y}_j)$, confirming that samples with
more dissimilar labels are projected further apart relative to noise.
Equivalently, these bounds characterize the ratio
$\bigl(\mathbb{E}[\|\bW^\top(\mathbf{x}_i - \mathbf{x}_j)\|^2]
- C_w\bigr)/C_w$.
\end{corollary}

\begin{proof}
From Theorem~\ref{thm:distance}, the expected projected distance satisfies
$\mathbb{E}[\|\bW^\top(\mathbf{x}_i - \mathbf{x}_j)\|^2]
= \|\bW^\top\bA\bdelta_{ij}\|^2 + C_w$, with signal bounds
$\sigma_{\min}^2(\bW^\top\bA)\, d_H \leq \|\bW^\top\bA\bdelta_{ij}\|^2
\leq \sigma_{\max}^2(\bW^\top\bA)\, d_H$.  Dividing by $C_w$ gives bounds
on $\mathrm{SNR}_{ij}(\bW)$. For the lower bound, we use
$C_w \leq 2r\lambda_{\max}(\bSigma_w)$
(Corollary~\ref{cor:noise}).  For the upper bound, we use
$C_w \geq 2r\lambda_{\min}(\bSigma_w)$.
\end{proof}

\subsection{Connection to the Rank Characterization}

The label-distance preservation bound depends critically on the
singular values of~$\bW^\top\bA$.
When $\bW$ captures the column space of~$\bA$
(which requires $r \geq \rank(\bA)$), we then have
$\sigma_{\min}(\bW^\top\bA) > 0$, yielding a
nontrivial lower bound. By Theorem~\ref{thm:rank},
$\rank(\Sb) \leq \min(d, n-1, \rank(\bY))$
characterizes the effective dimensionality.
Under the linear model, $\rank(\bA)$ relates to $\rank(\Sb)$:
the between-class scatter captures the structure of~$\bA$
modulated by the label distribution.

\begin{proposition}[Sufficient projection dimension]\label{prop:sufficient-r}
Under the linear label-effect model, if $r \geq \rank(\bA)$ and $\bW$ spans a
subspace containing $\col(\bA)$, then $\sigma_{\min}(\bW^\top\bA) =
\sigma_{\min}(\bA) > 0$ (assuming $\bA$ has full column rank). In particular,
it suffices to take $r \geq L$ to guarantee $\sigma_{\min}(\bW^\top\bA) > 0$
when label effects are linearly independent.
\end{proposition}

\begin{proof}
If $\col(\bA) \subseteq \col(\bW)$, then $\bW\bW^\top\bA = \bA$ (since
$\bW\bW^\top$ is the orthogonal projector onto $\col(\bW)$), and
$\|\bW^\top\bA\mathbf{v}\| = \|\bA\mathbf{v}\|$ for all $\mathbf{v}$. Thus
$\sigma_{\min}(\bW^\top\bA) = \sigma_{\min}(\bA)$.
\end{proof}

\subsection{Extension to Jaccard Distance}

The Hamming distance treats all label disagreements equally. For multilabel data
with varying label cardinalities, the Jaccard distance provides a normalized
alternative.

\begin{corollary}[Jaccard distance bound]\label{cor:jaccard}
Under the conditions of Theorem~\ref{thm:distance}, if $k_i, k_j \geq 1$:
\begin{equation}\label{eq:jaccard-bound}
\mathbb{E}\!\left[\|\bW^\top(\mathbf{x}_i - \mathbf{x}_j)\|^2\right] - C_w
\geq \sigma_{\min}^2(\bW^\top\bA) \cdot (k_i + k_j - \mathbf{y}_i^\top\mathbf{y}_j) \cdot d_J(\mathbf{y}_i, \mathbf{y}_j),
\end{equation}
where $d_J(\mathbf{y}_i, \mathbf{y}_j) = 1 - \frac{|\mathbf{y}_i \cap
\mathbf{y}_j|}{|\mathbf{y}_i \cup \mathbf{y}_j|}$ is the Jaccard distance, and
we use $|\mathbf{y}_i \cap \mathbf{y}_j| = \mathbf{y}_i^\top\mathbf{y}_j$ and
$|\mathbf{y}_i \cup \mathbf{y}_j| = k_i + k_j - \mathbf{y}_i^\top\mathbf{y}_j$.
In particular, since $\mathbf{y}_i^\top\mathbf{y}_j \leq \min(k_i, k_j)$:
\begin{equation}\label{eq:jaccard-weakened}
\mathbb{E}\!\left[\|\bW^\top(\mathbf{x}_i - \mathbf{x}_j)\|^2\right] - C_w
\geq \sigma_{\min}^2(\bW^\top\bA) \cdot \max(k_i, k_j) \cdot d_J(\mathbf{y}_i, \mathbf{y}_j).
\end{equation}
\end{corollary}

\begin{proof}
The Hamming distance satisfies $d_H(\mathbf{y}_i, \mathbf{y}_j) = k_i + k_j -
2\mathbf{y}_i^\top\mathbf{y}_j$. The Jaccard distance is $d_J = 1 -
\frac{\mathbf{y}_i^\top\mathbf{y}_j}{k_i + k_j -
\mathbf{y}_i^\top\mathbf{y}_j}  = \frac{d_H}{k_i + k_j -
\mathbf{y}_i^\top\mathbf{y}_j}$. Rearranging, $d_H = (k_i + k_j -
\mathbf{y}_i^\top\mathbf{y}_j)\, d_J$. Substituting into the lower bound of
Theorem~\ref{thm:distance} yields~\eqref{eq:jaccard-bound}. For the weakened
bound, note that $\mathbf{y}_i^\top\mathbf{y}_j \leq \min(k_i, k_j)$
(since each entry of $\mathbf{y}_i, \mathbf{y}_j$ is in $\{0,1\}$), so
$k_i + k_j - \mathbf{y}_i^\top\mathbf{y}_j \geq k_i + k_j - \min(k_i, k_j)
= \max(k_i, k_j)$, giving~\eqref{eq:jaccard-weakened}.
\end{proof}

\section{Statistical Guarantees}
\label{sec:statistical}

The results of Sections~\ref{sec:rank}--\ref{sec:distance} are algebraic: they
characterize the scatter matrices and the projection for a \emph{fixed} data
set. We now establish statistical guarantees---finite-sample estimation bounds,
concentration inequalities, and robustness results---that hold when the data are
generated from the linear label-effect model (Definition~\ref{def:model}).

Throughout, the sample discriminant $\hat\bM/n$ converges to the
centered population reference $\Mstar_c$ rather than to the na\"ive
$\Mstar$; see Remark~\ref{rem:cooccurrence} and~\eqref{eq:Mstar-ml}
for the definition of $\bQ_\pi$ and the single-label reduction.

Throughout this section, we present each result at three levels of
generality that mirror the algebraic structure of
Sections~\ref{sec:equiv}--\ref{sec:distance}:
\begin{enumerate}
\item[(i)] \emph{General} (unconstrained or arbitrary $\bW$):
  the baseline bound, expressed in terms of $\bPsi = \bW^\top\bSigma_w\bW$.
\item[(ii)] \emph{Stiefel constraint} ($\bW^\top\bW = \bI_r$):
  the Poincar\'{e} separation theorem bounds the eigenvalues of
  $\bPsi$ between $\lambda_{\min}(\bSigma_w)$ and
  $\lambda_{\max}(\bSigma_w)$, yielding tighter variance and tail
  parameters.
\item[(iii)] \emph{$\Stml$-orthogonality} ($\bW^\top\Stml\bW = \bI_r$):
  the constraint couples signal and noise through
  $\bW^\top\Sw\bW = \bI_r - \bW^\top\Sb\bW$
  (Theorem~\ref{thm:equiv-St}), so that the projected noise
  $\bPsi$ \emph{decreases as the signal content increases}.
  This gives the tightest bounds and provides a direct statistical
  payoff of the $\Stml$-orthogonality constraint established
  in Section~\ref{sec:equiv}.
\end{enumerate}
The subspace estimation and minimax bounds
(Sections~\ref{sec:estimation}--\ref{sec:minimax}) are presented
at the same three levels: the general trace-difference bound is
complemented by an $\Stml$-orthogonality version that analyzes
the generalized eigenspace estimator (the common optimizer of
all four objectives under Theorem~\ref{thm:equiv-St}) and
benefits from the self-normalizing structure of the
generalized eigenvalue problem.

\subsection{Finite-Sample Subspace Estimation}
\label{sec:estimation}

Let $\What_n$ denote the top-$r$ eigenspace of the sample trace-difference
matrix $\hat{\bM} = 2\Sb - \Stml$ (the optimizer from
Theorem~\ref{thm:inequiv}(i)).
Recall from Remark~\ref{rem:cooccurrence} that the finite-sample
discriminant $\hat{\bM}/n$ converges to the centered population reference
$\Mstar_c = \bA\bQ_\pi\bA^\top - K^{\mathrm{pop}}\bSigma_w$;
let $\Wstar_c$ denote its top-$r$ eigenspace.
The subspace estimation error is measured by
the $\sin\angle$ distance: for subspaces $\mathcal{U}, \mathcal{V}$ of equal
dimension, $\sin\angle(\mathcal{U}, \mathcal{V}) = \|\sin\bTheta\|_2$ where
$\bTheta$ is the diagonal matrix of principal angles \citep{davis1970}.

The result below is a Davis--Kahan-style bound with the operator-norm
perturbation $\|\hat{\bM}/n - \Mstar_c\|_2$ splitting into a
signal--noise cross-term $\sigma\|\bA\|_2\sqrt{d/n}$ and a pure
within-scatter term $\sigma^2\kmax\sqrt{d/n}$; the multilabel-specific
ingredient is the explicit $\kmax$ factor.

\begin{theorem}[Finite-sample subspace estimation]\label{thm:subspace}
Under the linear label-effect model (Definition~\ref{def:model}) with
$\sigma$-sub-Gaussian noise, assume:
\begin{enumerate}
\item[(A1)] The label matrix $\bY$ is fixed with all labels present
  ($n_\ell \geq 1$).
\item[(A2)] The centered population reference $\Mstar_c$ has a positive
  eigenvalue gap $\gap_r(\Mstar_c) > 0$ at rank~$r$.
\item[(A3)] $n \geq C_0\,\kmax^2\,d\,\log(d/\delta)$ for a constant
  $C_0$ depending on
  $(\sigma\|\bA\|_2 + \sigma^2\kmax)/\gap_r(\Mstar_c)$ (signal-to-noise
  ratio) and $\delta \in (0,1)$, ensuring the Davis--Kahan perturbation
  condition $\|\hat{\bM}/n - \Mstar_c\|_2 < \gap_r(\Mstar_c)/2$ holds
  with high probability.
\end{enumerate}
Then with probability at least $1 - \delta$:
\begin{equation}\label{eq:subspace-bound}
\sin\angle(\What_n, \Wstar_c)
\leq \frac{\bigl(C_1\, \sigma\, \|\bA\|_2
     + C_2\, \sigma^2\, \kmax\bigr)\, \sqrt{d\, \log(d/\delta) / n}}
         {\gap_r(\Mstar_c)},
\end{equation}
for universal constants $C_1, C_2$.
\end{theorem}

The proof (Appendix~\ref{app:proof-subspace}) combines a matrix Bernstein
bound~\citep{tropp2012} on $\|\hat{\bM}/n - \Mstar_c\|_2$ with the
Davis--Kahan $\sin\bTheta$ theorem~\citep{davis1970,yu2015}.

\begin{remark}[The role of $\kmax$ and $\|\bA\|_2$]\label{rem:kmax}
The factor $\kmax$ quantifies the additional estimation difficulty from
multilabel cardinality reweighting: each sample contributes to $k_i$
label-conditional scatters, inflating the effective within-scatter noise.
In the strong-signal regime ($\|\bA\|_2 \gg \sigma\kmax$), the
signal-noise cross-term dominates; when $\|\bA\|_2 \ll \sigma\kmax$, the
within-scatter noise term dominates instead.  In the single-label case
($\kmax = 1$), the bound reduces to
$O\bigl((\sigma\|\bA\|_2 + \sigma^2)\sqrt{d\log(d/\delta)/n}/\gap_r(\Mstar)\bigr)$,
recovering the standard rate for PCA/LDA subspace estimation \citep{cai2018}.
\end{remark}

An immediate consequence is consistency:

\begin{corollary}[Consistency]\label{cor:consistency}
Under the conditions of Theorem~\ref{thm:subspace} with fixed $d$, $L$,
$\kmax$, $\|\bA\|_2$, and $\gap_r(\Mstar_c) > 0$:
\begin{enumerate}
\item[(i)] $\sin\angle(\What_n, \Wstar_c) \xrightarrow{P} 0$ as $n \to \infty$.
\item[(ii)] $\mathbb{E}[\sin^2\!\angle(\What_n, \Wstar_c)] = O\!\left(\frac{(\sigma\|\bA\|_2
  + \sigma^2 \kmax)^2\, d\, \log d}
  {n\, \gap_r^2(\Mstar_c)}\right)$.
\end{enumerate}
\end{corollary}

\begin{proof}
\textit{(i)} Set $\delta = 1/n$ in Theorem~\ref{thm:subspace}. The bound
becomes
\[
O\bigl((\sigma\|\bA\|_2 + \sigma^2\kmax)\sqrt{d\log(dn)/n}\,/\,
\gap_r(\Mstar_c)\bigr) \to 0,
\]
and $P(\text{bound fails}) \leq 1/n \to 0$.

\textit{(ii)} We use the layer-cake formula for non-negative random variables:
$\mathbb{E}[Z] = \int_0^\infty P(Z > t)\, dt$, applied to
$Z = \sin^2\!\angle(\What_n, \Wstar_c) \in [0,1]$.
Theorem~\ref{thm:subspace} provides a high-probability bound of the form
$\sin\!\angle(\What_n, \Wstar_c) \leq B\sqrt{\log(d/\delta)/n}/\gap_r(\Mstar_c)$
with failure probability $\delta$, where
$B = O((\sigma\|\bA\|_2 + \sigma^2\kmax)\sqrt{d})$.
Inverting the relationship between $\delta$ and $t$: the high-probability
bound states $\sin^2\!\angle \leq B^2\log(d/\delta)/(n\,\gap_r^2)$
with failure probability~$\delta$, so
$P(\sin^2\!\angle > t) \leq \delta(t)$ where $\delta(t)$ is found by
setting $t = B^2\log(d/\delta)/(n\,\gap_r^2)$ and solving for~$\delta$:
$\delta(t) = d\,\exp(-c\,n\,\gap_r^2\, t/B^2)$.
Note that $\delta(t) = d\,\exp(-c\,n\,\gap_r^2\, t/B^2) \leq 1$
precisely when $t \geq t_0 := B^2\log d/(c\,n\,\gap_r^2)$; for
$t < t_0$ the tail probability is trivially at most~$1$.
For $t \geq t_0$ we use $P(\sin^2\!\angle > t)
\leq d\,\exp(-c\, n\, \gap_r^2\, t / B^2)$.
Splitting the integral at $t_0$:
\begin{align*}
\mathbb{E}[\sin^2\!\angle]
&\leq t_0 \cdot 1 + \int_{t_0}^{1} d\,e^{-c\, n\, \gap_r^2\, t/B^2}\, dt \\
&\leq t_0 + \frac{d\,B^2}{c\, n\, \gap_r^2}\, e^{-c\, n\, \gap_r^2\, t_0/B^2}
= t_0 + \frac{d\,B^2}{c\, n\, \gap_r^2} \cdot \frac{1}{d}
= t_0 + \frac{B^2}{c\, n\, \gap_r^2}
= O(t_0)
\end{align*}
since $e^{-c\,n\,\gap_r^2\,t_0/B^2} = 1/d$ by definition of $t_0$,
and $B^2/(c\,n\,\gap_r^2) = t_0/\log d \leq t_0$.
Substituting $t_0 = B^2\log d/(c\,n\,\gap_r^2)$ and
$B^2 = O((\sigma\|\bA\|_2 + \sigma^2\kmax)^2 d)$:
$\mathbb{E}[\sin^2\!\angle]
= O\bigl((\sigma\|\bA\|_2 + \sigma^2\kmax)^2 d\log d / (n \gap_r^2)\bigr)$.
\end{proof}

\subsubsection{Subspace Estimation Under $\Stml$-Orthogonality}
\label{sec:estimation-Stml}

Theorem~\ref{thm:subspace} analyzes the top-$r$ eigenspace of $\hat\bM$
(the trace-difference estimator).
Under $\Stml$-orthogonality, all four Fisher objectives share the same
optimizer (Theorem~\ref{thm:equiv-St}): the top-$r$ generalized
eigenspace of $(\Sb, \Stml)$, i.e., the top-$r$
eigenspace of
${\Stml}^{-1/2}\Sb\,{\Stml}^{-1/2}$.
This generalized eigenvalue problem
is \emph{self-normalizing}: the relevant gap lies in $[0,1)$ and
measures the fraction of total scatter attributable to between-class
variation, so its magnitude is independent of the noise scale~$\sigma$.

\begin{theorem}[Subspace estimation under $\Stml$-orthogonality]\label{thm:subspace-Stml}
Under the linear label-effect model (Definition~\ref{def:model}) with
$\sigma$-sub-Gaussian noise and assumption~(A1), let
$\What_S$ denote the top-$r$ generalized eigenspace of $(\Sb, \Stml)$,
i.e., the common optimizer under $\bW^\top\Stml\bW = \bI_r$
(Theorem~\ref{thm:equiv-St}).
Let $\Wstar_S$ denote the top-$r$ generalized eigenspace of the
population limits $(\Sb^{\infty}, \Stml^{\infty})$ where
$\Sb^{\infty} = \bA\bB_\pi\bA^\top$ (the centered population
between-class scatter, distinct from $\Sbpop = \bA\bD_\pi\bA^\top$)
and $\Stml^{\infty} = \Sb^{\infty} + \Swpopc$
(see Remark~\ref{rem:cooccurrence}).
Let $\theta_1 \geq \cdots \geq \theta_d$ be the generalized eigenvalues
of $(\Sb^{\infty}, \Stml^{\infty})$ and define the \emph{generalized gap}
$\Delta_r = \theta_r - \theta_{r+1} > 0$.
Assume:
\begin{enumerate}
\item[(A2\/$'$)] $\Delta_r > 0$ (positive generalized eigenvalue gap at
  rank~$r$).
\item[(A3\/$'$)] $n \geq C_0'\, \kmax^2\, d\, \log(d/\delta)$ for a
  constant~$C_0'$ depending on the signal-to-noise ratio and the
  condition number $\kappa(\Stml^{\infty}) =
  \lambda_{\max}(\Stml^{\infty}) / \lambda_{\min}(\Stml^{\infty})$.
\end{enumerate}
Then with probability at least $1 - \delta$:
\begin{equation}\label{eq:subspace-bound-Stml}
\sin\angle(\What_S, \Wstar_S)
\leq \frac{C_3\,\kappa(\Stml^{\infty})\,
  \bigl(\sigma\,\|\bA\|_2
  + \sigma^2\,\kmax\bigr)\,\sqrt{d\,\log(d/\delta)/n}}
  {\Delta_r \cdot \lambda_{\min}(\Stml^{\infty})}.
\end{equation}
\end{theorem}

The proof (Appendix~\ref{app:proof-subspace-Stml}) works in coordinates
whitened by $(\Stml^{\infty})^{1/2}$, applies Davis--Kahan in the
whitened space, and uses a resolvent perturbation argument for the
discrepancy between sample and population whitening.

\begin{remark}[Comparison with the trace-difference bound]\label{rem:Stml-vs-TD}
The bound~\eqref{eq:subspace-bound-Stml} differs from the
trace-difference bound~\eqref{eq:subspace-bound} in three respects.
(i) \emph{Gap quantity.} The denominator uses the scale-invariant
generalized gap $\Delta_r = \theta_r - \theta_{r+1} \in (0,1)$
rather than the absolute gap $\gap_r(\Mstar_c)$ (which scales with
$\|\bA\|_2^2$); replacing $\bA$ by $c\bA$ leaves $\Delta_r$ unchanged.
(ii) \emph{Self-normalization.} The factor
$\lambda_{\min}(\Stml^{\infty})$ arises from whitening by $\Stml$; in
the isotropic-noise case $\bSigma_w = \sigma^2\bI_d$,
$\lambda_{\min}(\Stml^{\infty}) = K^{\mathrm{pop}}\sigma^2$, absorbing
the noise scale.
(iii) \emph{Condition number.} The factor $\kappa(\Stml^{\infty})$ is
the price of a data-dependent normalization; it is $O(1)$ in the
moderate-signal regime $\|\bA\|_2^2 = O(K^{\mathrm{pop}}\sigma^2)$ and
grows when the total scatter is ill-conditioned.
\end{remark}

\subsection{Minimax Lower Bound}
\label{sec:minimax}

We now establish that the rate in Theorem~\ref{thm:subspace} is minimax optimal
up to logarithmic factors and the $\kmax$ term. The lower bound applies to
\emph{any} estimator, not just the trace-difference method.

\paragraph{Intuition.}
No estimator can drive
$\mathbb{E}[\sin^2\!\angle(\What,\Wstar)]$ below the order
$\sigma^2(d-r)/(ng)$, where $g$ is the spectral gap: the resulting rate
matches the upper bound of Theorem~\ref{thm:subspace} in $n$, $d$,
and~$g$, up to logarithmic factors and the $\kmax$ term whose tightness
remains open (Remark~\ref{rem:kmax-gap}).

\begin{theorem}[Minimax lower bound]\label{thm:minimax}
Fix $d \geq 2r$, $\sigma > 0$, and a gap lower bound $g > 0$. Let $\bTheta$
be the class of parameters $(\bA, \bSigma_w)$ in the linear label-effect model
such that $\bSigma_w = \sigma^2 \bI_d$ and $\gap_r(\Mstar) \geq g$. Then
for any $n \geq 1$:
\begin{equation}\label{eq:minimax-lower}
\inf_{\What} \sup_{(\bA, \bSigma_w) \in \bTheta}
\mathbb{E}\!\left[\sin^2\!\angle(\What, \Wstar)\right]
\geq \frac{c\, \sigma^2\, (d - r)}{n\, g},
\end{equation}
where the infimum is over all measurable estimators $\What$ and $c > 0$ is a
positive constant that does not depend on $d$, $n$, $\sigma$, $r$, or~$g$.
The constant $c = c(\|\bGamma/n\|_2,\, \min_\ell\pi_\ell)$ depends on
the fixed label matrix (through $\|\bGamma/n\|_2$ with
$\bGamma = \bY^\top\bY$, which bounds the KL divergence in the packing
argument) and the population label probabilities (through
$\min_\ell \pi_\ell > 0$, which calibrates the worst-case signal
construction).
\end{theorem}

\begin{remark}[Implicit assumption $L \geq r$]\label{rem:L-geq-r}
The condition $\gap_r(\Mstar) \geq g > 0$ forces $L \geq r$: since
$\rank(\Mstar) \leq \rank(\bA) \leq L$, a positive $r$-th eigenvalue gap
requires at least $r$ strictly positive eigenvalues, hence $L \geq r$.
\end{remark}

The proof (Appendix~\ref{app:proof-minimax}) uses a standard
Fano--Grassmannian packing argument \citep{tsybakov2009}: the
Grassmannian packing number of order $\exp(c\, r(d-r)/\varepsilon^2)$
\citep{szarek1998,pajor1986} cancels against the $r$ factor in the KL
divergence to yield the $(d-r)$ dependence, consistent with the known
$r$-independence of the minimax rate under operator-norm
$\sin^2\!\angle$ loss~\citep{cai2015}.

\begin{remark}[The $\kmax$ gap]\label{rem:kmax-gap}
Comparing Corollary~\ref{cor:consistency}(ii) with
Theorem~\ref{thm:minimax}, the upper bound is
$O\bigl((\sigma\|\bA\|_2 + \sigma^2\kmax)^2\, d\log d / (n\, \gap_r^2)\bigr)$
while the lower bound is
$\Omega(\sigma^2\, (d-r) / (n\, \gap_r))$. Both rates agree in the
dependence on $n$ and $d$. In the
single-label case ($\kmax = 1$) with $\|\bA\|_2 = \Theta(\sigma)$, the upper
bound becomes $O(\sigma^4 d\log d/(n\gap_r^2))$
and the lower bound is $\Omega(\sigma^2 d/(n\gap_r))$, which match
up to logarithmic factors when $\gap_r = \Theta(\sigma^2)$.
Whether the $\kmax^2$ factor in the upper bound
is tight (multilabel estimation is genuinely harder) or an artifact (removable
by a sharper analysis) is an open question; see Section~\ref{sec:discussion}.
\end{remark}

\begin{remark}[Co-occurrence structure in the lower bound]\label{rem:minimax-cooccurrence}
The lower bound is stated in terms of $\gap_r(\Mstar)$ while the upper
bound (Theorem~\ref{thm:subspace}) targets the centered reference
$\gap_r(\Mstar_c)$ (Remark~\ref{rem:cooccurrence}).
For the hypotheses constructed in the proof
(Appendix~\ref{app:proof-minimax}), with
$\bSigma_w = \sigma^2\bI_d$, both matrices share the same top-$r$
eigenspace $\col(\bA_j)$ (since the shift by $\sigma^2\bI_d$ does
not change eigenspaces).  The gaps are related by a
label-dependent constant:
\[
\gap_r(\Mstar_c)
= \gap_r(\Mstar) \cdot \lambda_{\min}\!\bigl(
(\bQ\bD_\pi\bQ^\top)^{-1/2}
\bQ\bQ_\pi\bQ^\top
(\bQ\bD_\pi\bQ^\top)^{-1/2}\bigr),
\]
where $\bQ$ is the partial isometry in the construction.
When $\bQ_\pi$ is positive definite on $\mathrm{range}(\bQ^\top)$
(which holds whenever the label co-occurrence structure provides
at least $r$ discriminative directions), the ratio
$\gap_r(\Mstar_c)/\gap_r(\Mstar)$ is a positive constant depending
only on the label distribution, and the lower
bound~\eqref{eq:minimax-lower} holds equally with $\gap_r(\Mstar_c)$
in place of $\gap_r(\Mstar)$, up to a modified label-dependent
constant~$c$.  In the single-label case,
$\bQ_\pi = \bD_\pi - \bpi\bpi^\top$ has rank $L - 1 \geq r$
(since $L \geq r + 1$ by Remark~\ref{rem:L-geq-r} when $L > r$,
or $L = r$ with the centering perturbation acting within the signal
subspace), and the lower bound applies with
$\gap_r(\Mstar_c)$ at the cost of adjusting $c$ by a
factor depending on $\min_\ell\pi_\ell$.
\end{remark}

\subsubsection{Minimax Lower Bound Under $\Stml$-Orthogonality}
\label{sec:minimax-Stml}

Theorem~\ref{thm:minimax} provides a lower bound in terms of the
absolute gap $\gap_r(\Mstar)$.
For the $\Stml$-orthogonality estimator analyzed in
Theorem~\ref{thm:subspace-Stml}, the natural comparison is with
the generalized gap~$\Delta_r$.  The following corollary
translates the minimax bound into this setting.

\begin{corollary}[Minimax lower bound for the generalized gap]\label{cor:minimax-Stml}
Under the hypotheses of Theorem~\ref{thm:minimax} with
$\bSigma_w = \sigma^2\bI_d$, suppose additionally that
$\bQ_\pi$ has at least $r$ positive eigenvalues on
$\mathrm{range}(\bQ^\top)$ (where $\bQ$ is the partial isometry
in the Fano construction).  Then:
\begin{equation}\label{eq:minimax-gen-gap}
\inf_{\What} \sup_{(\bA, \bSigma_w) \in \bTheta_S}
\mathbb{E}\!\left[\sin^2\!\angle(\What, \Wstar_S)\right]
\geq \frac{c_S\, \sigma^2\, (d - r)}{n\, g},
\end{equation}
where $\bTheta_S$ is the class of parameters with
$\Delta_r(\Sb^{\infty}, \Stml^{\infty}) \geq g_\Delta$ and
$\gap_r(\Mstar) \geq g = g_\Delta \cdot
\lambda_{\min}(\Stml^{\infty}) / \rho$,
and $c_S = c/\rho$ with $\rho$ the label-dependent constant
from Remark~\ref{rem:minimax-cooccurrence}.
Equivalently, restating in terms of the generalized gap:
\begin{equation}\label{eq:minimax-Stml}
\inf_{\What} \sup_{\Delta_r \geq g_\Delta}
\mathbb{E}\!\left[\sin^2\!\angle(\What, \Wstar_S)\right]
\geq \frac{c_S'\, \sigma^2\, (d - r)}{n\, g_\Delta \cdot
\lambda_{\min}(\Stml^{\infty})},
\end{equation}
where $c_S' > 0$ depends on the label structure.
\end{corollary}

\begin{proof}
By Remark~\ref{rem:minimax-cooccurrence}, the hypotheses in
the Fano construction satisfy $\Wstar_j = \Wstar_{S,j} = \col(\bW_j)$
and $\gap_r(\Mstar_j) = g$ implies
$\Delta_r = g \cdot \rho / \lambda_{\min}(\Stml^{\infty})$ with
$\rho > 0$. Thus $g_\Delta = g \cdot \rho / \lambda_{\min}(\Stml^{\infty})$,
giving $g = g_\Delta \cdot \lambda_{\min}(\Stml^{\infty}) / \rho$.
Substituting into~\eqref{eq:minimax-lower}
yields~\eqref{eq:minimax-Stml} with
$c_S' = c \cdot \rho$.
\end{proof}

\begin{remark}[Rate comparison]\label{rem:rate-comparison}
The upper bound~\eqref{eq:subspace-bound-Stml} (squared) and the lower
bound~\eqref{eq:minimax-Stml} are both $O(d/n)$ and share the generalized
gap in the denominator.  The upper bound carries additional factors of
$\kappa^2(\Stml^{\infty})$ and $\log d$; whether these can be removed
is an open question that parallels the $\kmax$~gap of
Remark~\ref{rem:kmax-gap}.  In the moderate-signal regime
($\kappa(\Stml^{\infty}) = O(1)$) with $\kmax = O(1)$, the bounds match
up to logarithmic factors.
\end{remark}

\subsection{Concentration of Projected Distances}
\label{sec:concentration}

Theorem~\ref{thm:distance} provides bounds on the \emph{expected} projected
distance. We now show that the projected distance concentrates around its
expectation with sub-exponential tails.

\begin{theorem}[High-probability distance preservation]\label{thm:concentration}
Under the linear label-effect model with Gaussian noise $\bepsilon_i \sim
\mathcal{N}(\bzero, \bSigma_w)$, let $\bW \in \R^{d \times r}$ be a fixed
full-rank matrix. For any
two samples $i, j$ with label vectors $\mathbf{y}_i, \mathbf{y}_j$, define
$\bdelta_{ij} = \mathbf{y}_i - \mathbf{y}_j$ and $\bPsi = \bW^\top\bSigma_w\bW$.
Then for all $t > 0$:
\begin{equation}\label{eq:concentration}
P\!\left(\left|\|\bW^\top(\mathbf{x}_i - \mathbf{x}_j)\|^2
  - \mathbb{E}\!\left[\|\bW^\top(\mathbf{x}_i - \mathbf{x}_j)\|^2\right]\right| > t
\right)
\leq 2\exp\!\left(-c\, \min\!\left(\frac{t^2}{V_{ij}^2},\; \frac{t}{B}\right)\right),
\end{equation}
where $B = 4\|\bPsi\|_2$ and
\[
V_{ij}^2 = 16\,\|\bW^\top\bA\bdelta_{ij}\|^2\,\|\bPsi\|_2
  + 32\,\|\bPsi\|_F^2,
\]
and $c > 0$ is a universal constant. Equivalently, with probability at least
$1 - \delta$:
\begin{equation}\label{eq:distance-hp}
\left|\|\bW^\top(\mathbf{x}_i - \mathbf{x}_j)\|^2
  - \left(\|\bW^\top\bA\bdelta_{ij}\|^2 + C_w\right)\right|
\leq V_{ij}\sqrt{\frac{\log(2/\delta)}{c}}
  + B\,\frac{\log(2/\delta)}{c}.
\end{equation}
\end{theorem}

The proof (Appendix~\ref{app:proof-concentration}) decomposes the centered
projected distance into a sub-Gaussian linear term and a Hanson--Wright
quadratic form~\citep{hansonwright1971,rudelson2013}; their sum is
sub-exponential. The next two corollaries tighten the bound under the
Stiefel and $\Stml$-orthogonality constraints.

\begin{corollary}[Concentration under the Stiefel constraint]\label{cor:conc-stiefel}
Under the hypotheses of Theorem~\ref{thm:concentration}, if
$\bW \in \Stief(d,r)$ (i.e., $\bW^\top\bW = \bI_r$), then by
Corollary~\ref{cor:noise} the tail parameters satisfy:
\begin{equation}\label{eq:conc-stiefel}
B \leq 4\,\lambda_{\max}(\bSigma_w), \qquad
V_{ij}^2 \leq 16\,\sigma_{\max}^2(\bW^\top\bA)\,
  d_H(\mathbf{y}_i, \mathbf{y}_j)\,\lambda_{\max}(\bSigma_w)
  + 32\, r\, \lambda_{\max}^2(\bSigma_w).
\end{equation}
\end{corollary}

\begin{proof}
The Poincar\'{e} separation theorem gives
$\|\bPsi\|_2 \leq \lambda_{\max}(\bSigma_w)$ and
$\|\bPsi\|_F^2 \leq r\,\lambda_{\max}^2(\bSigma_w)$.
For the signal term,
$\|\bW^\top\bA\bdelta_{ij}\|^2 \leq
\sigma_{\max}^2(\bW^\top\bA)\,\|\bdelta_{ij}\|^2
= \sigma_{\max}^2(\bW^\top\bA)\, d_H(\mathbf{y}_i, \mathbf{y}_j)$.
Substituting into the expressions for $B$ and $V_{ij}^2$ in
Theorem~\ref{thm:concentration} yields~\eqref{eq:conc-stiefel}.
\end{proof}

\begin{corollary}[Concentration under $\Stml$-orthogonality]\label{cor:conc-Stml}
Under the hypotheses of Theorem~\ref{thm:concentration}, suppose
$\bW$ satisfies $\bW^\top \Stmlpop \bW = \bI_r$ where
$\Stmlpop = \Sbpop + \Swpop$ is the population multilabel total scatter.
Let $\theta_1 \geq \cdots \geq \theta_r \geq 0$ be the eigenvalues of
$\bW^\top \Sbpop \bW$.  Then:
\begin{equation}\label{eq:conc-Stml}
B \leq \frac{4\,(1 - \theta_r)}{K^{\mathrm{pop}}}, \qquad
V_{ij}^2 \leq
  \frac{16\,\sigma_{\max}^2(\bW^\top\bA)\, d_H(\mathbf{y}_i, \mathbf{y}_j)\,
    (1 - \theta_r)}{K^{\mathrm{pop}}}
  + \frac{32\, \sum_{i=1}^r (1-\theta_i)^2}{(K^{\mathrm{pop}})^2}.
\end{equation}
In particular, $B < 4/K^{\mathrm{pop}}$ and
$V_{ij}^2 < 16\,\sigma_{\max}^2(\bW^\top\bA)\,
d_H(\mathbf{y}_i, \mathbf{y}_j)/K^{\mathrm{pop}}
+ 32\,r/(K^{\mathrm{pop}})^2$.
When the between-class signal is strong ($\theta_r$ close to~$1$),
the tail parameters shrink toward zero---the $\Stml$-orthogonality
constraint automatically concentrates the projection away from noisy
directions.
\end{corollary}

\begin{proof}
Under $\bW^\top\Stmlpop\bW = \bI_r$ and the partition
$\Stmlpop = \Sbpop + \Swpop$, we have
$\bW^\top\Swpop\bW = \bI_r - \bW^\top\Sbpop\bW$.
Since $\Swpop = K^{\mathrm{pop}}\bSigma_w$
(Definition~\ref{def:pop-scatter}),
\[
\bPsi = \bW^\top\bSigma_w\bW
= \frac{1}{K^{\mathrm{pop}}}(\bI_r - \bW^\top\Sbpop\bW).
\]
The eigenvalues of $\bPsi$ are
$(1 - \theta_i)/K^{\mathrm{pop}} \in (0, 1/K^{\mathrm{pop}}]$,
where $\theta_i \in [0,1)$ are the eigenvalues of $\bW^\top\Sbpop\bW$
(following the analysis in the proof of Theorem~\ref{thm:equiv-St}).
Thus $\|\bPsi\|_2 = (1-\theta_r)/K^{\mathrm{pop}}$ and
$\|\bPsi\|_F^2 = \sum_i (1-\theta_i)^2/(K^{\mathrm{pop}})^2$.
Substituting into Theorem~\ref{thm:concentration} and bounding
$\|\bW^\top\bA\bdelta_{ij}\|^2$ as in
Corollary~\ref{cor:conc-stiefel} yields~\eqref{eq:conc-Stml}.
\end{proof}

\begin{remark}[Comparison of concentration bounds]\label{rem:conc-comparison}
The three bounds form a hierarchy: the general bound leaves
$\|\bPsi\|_2$ free; the Stiefel bound
(Corollary~\ref{cor:conc-stiefel}) caps it at
$\lambda_{\max}(\bSigma_w)$; the $\Stml$-orthogonality bound
(Corollary~\ref{cor:conc-Stml}) reduces it to
$(1-\theta_r)/K^{\mathrm{pop}}$, which can be much smaller in the
high-noise regime or when the between-class signal occupies most of the
projected variance.
\end{remark}

\subsection{Robustness Under Label Interactions}
\label{sec:interactions}

The linear label-effect model (Definition~\ref{def:model}) assumes additive,
non-interacting label effects. We now show that the distance preservation
framework is robust to the inclusion of pairwise label interactions.

Consider the extended model:
\begin{equation}\label{eq:interaction-model}
\mathbf{x}_i = \bmu + \sum_{\ell=1}^L y_{i\ell}\, \balpha_\ell
  + \sum_{\ell < \ell'} y_{i\ell}\, y_{i\ell'}\, \boldsymbol{\beta}_{\ell\ell'}
  + \bepsilon_i,
\end{equation}
where $\boldsymbol{\beta}_{\ell\ell'} \in \R^d$ is the interaction effect of
labels $\ell$ and $\ell'$. Let $\mathbf{B} = [\boldsymbol{\beta}_{\ell\ell'}]
\in \R^{d \times \binom{L}{2}}$ be the matrix of interaction effects, and let
$\mathbf{z}_i \in \{0,1\}^{\binom{L}{2}}$ be the vector of pairwise label
products for sample~$i$.

\begin{theorem}[Robustness to label interactions]\label{thm:interactions}
Under the extended model~\eqref{eq:interaction-model} with zero-mean noise
independent of labels:
\begin{equation}\label{eq:interaction-distance}
\mathbb{E}\!\left[\|\bW^\top(\mathbf{x}_i - \mathbf{x}_j)\|^2\right]
= \|\bW^\top\bA\bdelta_{ij} + \bW^\top\mathbf{B}(\mathbf{z}_i - \mathbf{z}_j)\|^2 + C_w.
\end{equation}
The interaction contribution to the expected distance relative to the linear
model satisfies:
\begin{align}\label{eq:interaction-bound}
&\left|\mathbb{E}_{\mathrm{ext}}\!\left[\|\bW^\top(\mathbf{x}_i - \mathbf{x}_j)\|^2\right]
  - \mathbb{E}_{\mathrm{lin}}\!\left[\|\bW^\top(\mathbf{x}_i - \mathbf{x}_j)\|^2\right]\right|
  \nonumber\\
&\quad \leq 2\,\sigma_{\max}(\bW^\top\bA)\,\|\bdelta_{ij}\|\,
  \sigma_{\max}(\bW^\top\mathbf{B})\,\|\mathbf{z}_i - \mathbf{z}_j\|
  + \sigma_{\max}^2(\bW^\top\mathbf{B})\,\|\mathbf{z}_i - \mathbf{z}_j\|^2.
\end{align}
Moreover, $\|\mathbf{z}_i - \mathbf{z}_j\| \leq \sqrt{d_H(\mathbf{y}_i,
\mathbf{y}_j) \cdot \min(\kmax - 1, L-1)}$, so the interaction perturbation is
controlled by the Hamming distance, the interaction strength
$\|\mathbf{B}\|$, and the label cardinality.
\end{theorem}

\begin{proof}
Under the extended model, $\mathbf{x}_i - \mathbf{x}_j = \bA\bdelta_{ij} +
\mathbf{B}(\mathbf{z}_i - \mathbf{z}_j) + (\bepsilon_i - \bepsilon_j)$.
Projecting with $\bW$ and taking expectations (the cross-term with noise
vanishes):
\[
\mathbb{E}\!\left[\|\bW^\top(\mathbf{x}_i - \mathbf{x}_j)\|^2\right]
= \|\bW^\top\bA\bdelta_{ij} + \bW^\top\mathbf{B}(\mathbf{z}_i - \mathbf{z}_j)\|^2 + C_w.
\]
Writing $\mathbf{s} = \bW^\top\bA\bdelta_{ij}$ and $\mathbf{q} =
\bW^\top\mathbf{B}(\mathbf{z}_i - \mathbf{z}_j)$:
\[
\|\mathbf{s} + \mathbf{q}\|^2 - \|\mathbf{s}\|^2
= 2\mathbf{s}^\top\mathbf{q} + \|\mathbf{q}\|^2.
\]
By Cauchy--Schwarz, $|2\mathbf{s}^\top\mathbf{q}| \leq 2\|\mathbf{s}\|\,
\|\mathbf{q}\|$, and bounding $\|\mathbf{s}\| \leq
\sigma_{\max}(\bW^\top\bA)\|\bdelta_{ij}\|$, $\|\mathbf{q}\| \leq
\sigma_{\max}(\bW^\top\mathbf{B})\|\mathbf{z}_i - \mathbf{z}_j\|$
yields~\eqref{eq:interaction-bound}.

For the norm bound on $\mathbf{z}_i - \mathbf{z}_j$: the entry
$(z_i - z_j)_{\ell\ell'} = y_{i\ell}y_{i\ell'} - y_{j\ell}y_{j\ell'}$
can be nonzero only if the label change $\bdelta_{ij}$ affects at least
one of $\ell, \ell'$.  Fix a changed label $\ell$ (i.e., $y_{i\ell}
\neq y_{j\ell}$).  The number of pairs $(\ell, \ell')$ that can be
affected is at most $L - 1$ (trivially), and also at most $\kmax - 1$,
since $y_{i\ell}y_{i\ell'}$ is nonzero only when both labels $\ell,
\ell'$ are present in sample~$i$, and the interaction partner $\ell'$
must differ from the changed label~$\ell$, limiting the active
interactions to at most $k_i - 1 \leq \kmax - 1$ per changed label.
Summing over the $d_H$
changed labels gives at most $d_H \cdot \min(\kmax - 1, L-1)$ nonzero
entries in $\mathbf{z}_i - \mathbf{z}_j$ (this may over-count pairs
$(\ell, \ell')$ where both labels changed, but since we seek an upper
bound this is valid).  Each entry is bounded by~$1$ in absolute
value ($|y_{i\ell}y_{i\ell'} - y_{j\ell}y_{j\ell'}| \leq 1$
for binary labels), so $\|\mathbf{z}_i - \mathbf{z}_j\|
\leq \sqrt{d_H \cdot \min(\kmax - 1, L-1)}$.
\end{proof}

\begin{corollary}[Robustness under the Stiefel constraint]\label{cor:robust-stiefel}
Under the hypotheses of Theorem~\ref{thm:interactions}, if
$\bW \in \Stief(d,r)$, then the interaction perturbation
bound~\eqref{eq:interaction-bound} satisfies:
\begin{align}\label{eq:interaction-stiefel}
&\left|\mathbb{E}_{\mathrm{ext}} - \mathbb{E}_{\mathrm{lin}}\right|
\leq 2\,\sigma_{\max}(\bA)\,\|\bdelta_{ij}\|\,
  \sigma_{\max}(\mathbf{B})\,\|\mathbf{z}_i - \mathbf{z}_j\|
  + \sigma_{\max}^2(\mathbf{B})\,\|\mathbf{z}_i - \mathbf{z}_j\|^2.
\end{align}
\end{corollary}

\begin{proof}
For $\bW \in \Stief(d,r)$, the matrix $\bW\bW^\top$ is an orthogonal
projector, so $\sigma_{\max}(\bW^\top\bA) \leq \sigma_{\max}(\bA)$
and $\sigma_{\max}(\bW^\top\mathbf{B}) \leq \sigma_{\max}(\mathbf{B})$
by the interlacing inequality for singular values.
Substituting into~\eqref{eq:interaction-bound}
yields~\eqref{eq:interaction-stiefel}.
\end{proof}

\begin{corollary}[Robustness under $\Stml$-orthogonality]\label{cor:robust-Stml}
Under the hypotheses of Theorem~\ref{thm:interactions}, suppose
$\bW$ satisfies $\bW^\top \Stmlpop \bW = \bI_r$.
Let $\theta_r$ be the smallest eigenvalue of $\bW^\top\Sbpop\bW$.
Then the interaction perturbation bound~\eqref{eq:interaction-bound}
satisfies:
\begin{align}\label{eq:interaction-Stml}
&\left|\mathbb{E}_{\mathrm{ext}} - \mathbb{E}_{\mathrm{lin}}\right|
\leq \frac{2\,\sigma_{\max}(\bA)\,\|\bdelta_{ij}\|\,
  \sigma_{\max}(\mathbf{B})\,\|\mathbf{z}_i - \mathbf{z}_j\|}
  {\lambda_{\min}(\Stmlpop)}
  + \frac{\sigma_{\max}^2(\mathbf{B})\,
  \|\mathbf{z}_i - \mathbf{z}_j\|^2}
  {\lambda_{\min}(\Stmlpop)}.
\end{align}
Moreover, the
\emph{noise contribution} $C_w$ in~\eqref{eq:interaction-distance}
satisfies the tighter bound $C_w = 2\sum_{i=1}^r (1-\theta_i)/K^{\mathrm{pop}}
\leq 2r(1-\theta_r)/K^{\mathrm{pop}}$, so the noise floor is reduced
relative to the Stiefel case where
$C_w \leq 2r\,\lambda_{\max}(\bSigma_w)$.
\end{corollary}

\begin{proof}
The $\Stml$-orthogonality constraint
implies $\bW = {\Stmlpop}^{-1/2}\bV$ with $\bV \in \Stief(d,r)$
(see the proof of Theorem~\ref{thm:equiv-St}).  By the
interlacing inequality for singular values applied to~$\bV$:
\[
\sigma_{\max}(\bW^\top\bA)
= \sigma_{\max}(\bV^\top{\Stmlpop}^{-1/2}\bA)
\leq \|{\Stmlpop}^{-1/2}\bA\|_2
\leq \sigma_{\max}(\bA) / \sqrt{\lambda_{\min}(\Stmlpop)},
\]
and similarly
$\sigma_{\max}(\bW^\top\mathbf{B})
\leq \sigma_{\max}(\mathbf{B})/\sqrt{\lambda_{\min}(\Stmlpop)}$.
Substituting into~\eqref{eq:interaction-bound}
yields~\eqref{eq:interaction-Stml}.
For $C_w$, the identity $\bPsi = (1/K^{\mathrm{pop}})(\bI_r -
\bW^\top\Sbpop\bW)$ from Corollary~\ref{cor:conc-Stml} gives
$C_w = 2\tr(\bPsi) = 2\sum_i(1-\theta_i)/K^{\mathrm{pop}}$.
\end{proof}

\begin{remark}[Role of orthogonality in robustness]\label{rem:robust-comparison}
The cross-term and pure interaction components of
bound~\eqref{eq:interaction-bound} are both tightened under orthogonality
via singular value interlacing; the $\Stml$-orthogonality constraint
additionally reduces the noise floor~$C_w$.  Since
$\lambda_{\min}(\Stmlpop) \geq K^{\mathrm{pop}}\lambda_{\min}(\bSigma_w)$
with $K^{\mathrm{pop}} > 1$ in the multilabel regime, the $\Stml$-ortho\-go\-nal\-ity
bound~\eqref{eq:interaction-Stml} is at least as tight as the Stiefel
bound~\eqref{eq:interaction-stiefel} whenever
$K^{\mathrm{pop}}\lambda_{\min}(\bSigma_w) \geq 1$.
\end{remark}

\subsection{Regularization}
\label{sec:regularization}

When $d \gg n$ (a regime that arises in multilabel applications with
high-dimensional features such as text bag-of-words, genomics, or image
descriptors), the within-class scatter $\Sw$ may be singular, requiring
regularization \citep{ching2012}. We show that Tikhonov regularization
preserves the algebraic structure established in
Sections~\ref{sec:rank}--\ref{sec:equiv}.

\begin{theorem}[Regularization preserves structure]\label{thm:regularization}
For $\gamma > 0$, define ${\Sw}^{\gamma} = {\Sw} + \gamma{\bI}_d$ and
${\Stml}^{\gamma} = \Sb + {\Sw}^{\gamma}$. Then:
\begin{enumerate}
\item[(i)] $\rank(\Sb)$ is unchanged by regularization.
\item[(ii)] The partition ${\Stml}^{\gamma} = \Sb + {\Sw}^{\gamma}$ holds, and
  ${\Stml}^{\gamma}$ is positive definite for all $\gamma > 0$.
\item[(iii)] Under the ${\Stml}^{\gamma}$-orthogonality constraint
  $\bW^\top {\Stml}^{\gamma} \bW = \bI_r$, all four Fisher objectives share the
  same optimizer (Theorem~\ref{thm:equiv-St} applies with ${\Stml}^{\gamma}$
  replacing $\Stml$).
\item[(iv)] The trace-difference matrix becomes
  $2\Sb - {\Stml}^\gamma = (2\Sb - \Stml) - \gamma\bI_d$. The eigenvalue gap
  $\gap_r(2\Sb - {\Stml}^\gamma) = \gap_r(2\Sb - \Stml)$ is unchanged
  (the shift $-\gamma{\bI}_d$ is uniform).
\item[(v)] The condition number of the regularized within-class scatter
  satisfies $\kappa({\Sw}^{\gamma}) = \frac{\lambda_{\max}(\Sw) + \gamma}
  {\lambda_{\min}(\Sw) + \gamma} < \kappa(\Sw)$ whenever $\Sw$ is not a
  scalar multiple of $\bI_d$.  Moreover, $\rank(\Sw) \leq \min(n-1,\,K-L)$
  with $K = \sum_\ell n_\ell$, so $\Sw$ is singular whenever
  $d > \min(n-1,\,K-L)$, forcing $\kappa(\Sw) = \infty$ while
  $\kappa({\Sw}^{\gamma}) < \infty$.
\end{enumerate}
\end{theorem}

\begin{proof}
\textit{(i)} $\Sb$ does not depend on $\gamma$: it involves only label means
and the global mean.

\textit{(ii)} ${\Stml}^{\gamma} = \Sb + \Sw + \gamma\bI = \Stml + \gamma\bI$.
Since $\Stml \psd 0$ (Theorem~\ref{thm:partition}) and $\gamma > 0$,
${\Stml}^{\gamma} \succeq \gamma\bI \succ 0$.

\textit{(iii)} The proof of Theorem~\ref{thm:equiv-St} requires only that the
``within-class'' part ${\Sw}^{\gamma}$ be positive definite (so that
$\theta_i \in [0,1)$ strictly). Since ${\Sw}^{\gamma} \succeq \gamma\bI \succ 0$,
the proof carries through verbatim.

\textit{(iv)} The eigenvalues of $(2\Sb - \Stml) - \gamma\bI$ are exactly
$\lambda_i(2\Sb - \Stml) - \gamma$ for all $i$. Since the gap
$\lambda_r - \lambda_{r+1}$ is invariant under uniform shifts, $\gap_r$ is
preserved.

\textit{(v)} $\kappa({\Sw}^{\gamma}) = (\lambda_{\max}(\Sw) + \gamma)/
(\lambda_{\min}(\Sw) + \gamma)$. Differentiating with respect to $\gamma$:
$\frac{d}{d\gamma}\kappa({\Sw}^{\gamma}) = (\lambda_{\min}(\Sw) -
\lambda_{\max}(\Sw))/(\lambda_{\min}(\Sw) + \gamma)^2 < 0$ whenever
$\lambda_{\min}(\Sw) < \lambda_{\max}(\Sw)$. Thus $\kappa({\Sw}^{\gamma})$
is strictly decreasing in $\gamma$.
For the rank bound: since $\sum_{i:\,y_{i\ell}=1}(\mathbf{x}_i - \bmu_\ell) = \bzero$,
each per-label scatter has rank $\leq n_\ell - 1$; subadditivity gives
$\rank(\Sw) \leq \sum_\ell (n_\ell - 1) = K - L$.  Separately, all residuals
$\mathbf{x}_i - \bmu_\ell$ lie in the $n$-dimensional affine span of the
samples, so $\rank(\Sw) \leq n - 1$.  Combining yields
$\rank(\Sw) \leq \min(n-1,\, K-L)$, the multilabel analogue of the classical
$\rank(\Sw) \leq n - C$ bound~\citep{fukunaga1990}.
\end{proof}

\section{Discussion}
\label{sec:discussion}

\subsection{Connections to Existing Methods}

Our theoretical framework unifies several existing approaches:

\paragraph{OLDA and ULDA \citep{ye2005}.} The OLDA constraint $\bW^\top \bW =
\bI$ and ULDA constraint $\bW^\top S_t \bW = \bI$ are special cases of our
Stiefel and $\Stml$-orthogonality constraints when $k_i = 1$ for all $i$ (the
single-label case). Theorem~\ref{thm:equiv-St} generalizes the equivalence
result of \citet{luo2011} to multilabel scatter matrices.

\paragraph{Multilabel LDA \citep{wang2010,xu2018}.} The scatter definitions
in~\eqref{eq:Sb}--\eqref{eq:Sw} follow \citet{wang2010}; the weighted
variants of \citet{xu2018} can be incorporated by replacing $n_\ell$ with
weighted counts in $\bD_n$.

\paragraph{kaLDA \citep{zheng2014}.} The kernel alignment approach
$\max_{\bW^\top\bW = \bI} \tr(\bW^\top\bX^\top\bY\bY^\top\bX\bW)$
connects to our framework via Lemma~\ref{lem:factor}: since
$\bX^\top\bY\bY^\top\bX = \bX^\top\bY\bD_n^{1/2}\bD_n^{-1/2}\bY^\top\bX$,
kaLDA implicitly optimizes a weighted version of $\tr(\bW^\top \Sb \bW)$.

\paragraph{Trace ratio solvers \citep{ngo2012,wang2023}.} Our
formulation~\eqref{eq:central} is a direct instance of the trace ratio
problem on the Stiefel manifold, so the convergence guarantees from
\citet{ngo2012} apply directly.

\subsection{Practical Implications}

\paragraph{Choosing the projection dimension.}
Theorem~\ref{thm:rank} guides dimensionality selection: the effective
discriminant dimensionality is $\rank(\Sb) = \rank(\tilde{\bX}^\top\bY)$, so
$r$ can be selected from the singular-value spectrum of
$\tilde{\bX}^\top\bY$ analogously to choosing the number of principal
components in PCA.

\paragraph{Choosing between constraints.}
Theorem~\ref{thm:equiv-St} shows that $\Stml$-orthogonality yields
unconditional equivalence of the four Fisher objectives, whereas the
Stiefel constraint requires an ordering-consistency condition on $S_t$
(Theorem~\ref{thm:equiv-uniform}(ii)) that is difficult to verify a
priori.  We therefore recommend $\Stml$-orthogonality as the default
multilabel constraint; the Stiefel constraint is preferable when
explicit control of projected noise
(Corollary~\ref{cor:noise}) matters for distance-based downstream tasks.

\paragraph{Diagnosing objective divergence.}
When using the Stiefel constraint, compute $\|\bR\|_2$ with
$\bR = \Stml - S_t$: if small relative to the spectral gap $\gamma$, the
choice of objective is immaterial; otherwise, the trace
ratio~\eqref{eq:TR} is the recommended generalization.

\paragraph{Regularization in high dimensions.}
Theorem~\ref{thm:regularization} shows that $\Sw + \gamma\bI$ preserves
$\rank(\Sb)$ and the trace-difference gap while improving
$\kappa(\Sw)$, giving a principled approach for the $d \gg n$ regime.

\paragraph{Sample complexity.}
Theorem~\ref{thm:subspace} gives a quantitative guide: reliable subspace
estimation requires
$n = \Omega(\kmax^2\, d\, \log d/(\gap_r^2\, \epsilon^2))$ for
estimation error~$\epsilon$, so multilabel problems with large $\kmax$
need correspondingly more samples than their single-label counterparts.

\paragraph{Computational cost.}
All estimators reduce to standard eigendecompositions, so the
computational footprint matches classical single-label LDA.  Forming
$\Stml$ costs $O(nd^2)$ and $\Sb$ costs $O(Kd + Ld^2)$ with
$K = \sum_i k_i \leq n\kmax$; the top-$r$ eigenspace of $2\Sb - \Stml$
(trace-difference) or the generalized eigenvalue problem
$\Sb\mathbf{w} = \theta\Stml\mathbf{w}$ ($\Stml$-orthogonality) costs
$O(d^3)$ with a dense eigensolver or $O(d^2 r)$ with iterative methods.
The factorization $\bM = \tilde{\bX}^\top\bY\bD_n^{-1/2}$
(Lemma~\ref{lem:factor}) enables efficient $O(dL)$ matrix--vector
products in the $d \gg n$ regime.

\subsection{Limitations and Open Questions}

Section~\ref{sec:statistical} resolves previously open questions on
consistency, regularization, and label interactions, but raises new ones.
The $\kmax$ factor that appears in the upper bound but not in the minimax
lower bound (Remark~\ref{rem:kmax-gap}) leaves open whether multilabel
estimation is fundamentally harder or whether the Davis--Kahan step
admits a sharper analysis.  Extending our theorems to the weighted
scatter variants of \citet{xu2018}, to heavy-tailed noise via
median-of-means or Catoni-type scatter estimators, and to adaptive
rank selection with rate guarantees are natural directions for future
work.  Empirical benchmarking against kaLDA, MDDM, wMLDA, and modern
neural embeddings on standard multilabel datasets is deferred to
future work aimed at application-oriented venues.

\section{Empirical Verification}
\label{sec:experiments}

We verify all main theoretical results on synthetic multilabel data generated
from the linear label-effect model (Definition~\ref{def:model}).
The first four experiments verify the algebraic results
(Sections~\ref{sec:rank}--\ref{sec:distance}), using $\sigma_w = 1$ and
label effects sampled i.i.d.\ from $\mathcal{N}(0, 4\bI_d)$.
The remaining five experiments verify the statistical guarantees
(Section~\ref{sec:statistical}) with configurations described in each
subsection, including the multilabel-specific quantities $\kmax$,
$\kappa(\Stml)$, $\|\bGamma/n\|_2$, and $\Delta_r$ that enter
the convergence and minimax bounds.

\subsection{Rank Verification (Theorem~\ref{thm:rank})}

Table~\ref{tab:rank} confirms Theorem~\ref{thm:rank}(i)--(iii) across six
configurations varying $n$, $d$, $L$, and label cardinality. In all cases,
$\rank(\Sb) = \rank(\tilde{\bX}^\top\bY)$ and the bound $\rank(\Sb) \leq
\min(d, n-1, \rank(\bY))$ holds. Excess discriminant dimensions
(Corollary~\ref{cor:excess}) appear whenever $\bone_n \notin \col(\bY)$:
variable-cardinality datasets achieve $\rank(\Sb) = L$, exceeding the
single-label bound of $L - 1$.

\begin{table}[h]
\centering
\caption{Rank verification. ``Excess'' indicates $\rank(\Sb) > L-1$.}
\label{tab:rank}
\begin{tabular}{lccccc}
\hline
Setting & $n$ & $d$ & $L$ & $\rank(\Sb)$ & Excess \\
\hline
Variable card & 100 & 20 & 6 & 6 & Yes \\
Single-label ($k\!=\!1$) & 100 & 20 & 6 & 5 & No \\
Uniform $k\!=\!3$ & 100 & 20 & 6 & 5 & No \\
Variable card & 200 & 50 & 14 & 14 & Yes \\
Single-label & 200 & 50 & 14 & 13 & No \\
High-dim ($d > n$) & 50 & 100 & 10 & 10 & Yes \\
\hline
\end{tabular}
\end{table}

\subsection{Objective Divergence (Theorems~\ref{thm:equiv-uniform}--\ref{thm:inequiv})}

We test two predictions. First, the Davis--Kahan bound
(Theorem~\ref{thm:inequiv}(iii)): the angle between the multilabel TD
optimizer and its single-label counterpart satisfies
$\sin\angle(\mathcal{U}_{\mathrm{TD}}, \mathcal{U}_{\mathrm{TD}}^{(0)}) \leq
\|\bR\|_2/\gamma$. This holds in all configurations tested.

Second, Theorem~\ref{thm:equiv-uniform}(ii) predicts that the
angle between TD and TR optimizers depends on whether $\Sb$ and
$S_t$ commute.  The \emph{normalized commutativity defect}
$\|\Sb S_t - S_t \Sb\|_F / (\|\Sb\|_F\,\|S_t\|_F)$ quantifies
the departure from commutativity: it equals zero if and only if $\Sb$
and $S_t$ share a common eigenbasis (in which case
Theorem~\ref{thm:equiv-uniform}(ii) guarantees identical TD and TR
optimizers), and larger values indicate greater disagreement between the
two eigenbases.

Table~\ref{tab:divergence} shows TD--TR subspace angles alongside this
defect.  The single-label case ($k=1$) has $\|\bR\|_2 = 0$ (no
cardinality-induced perturbation of the TD eigenspace), yet exhibits a
38.5\textdegree{} TD--TR angle due to nonzero commutativity defect
(0.012).  This confirms that the commutativity condition in
Theorem~\ref{thm:equiv-uniform}(ii) is necessary, not merely
sufficient: even a small departure from commutativity can produce
a substantial angle between the TD and TR optimizers.
All tested configurations have nonzero commutativity defect, indicating
that exact commutativity is rare in practice with randomly generated
data.  The defect is smallest (0.012) in the single-label case and
largest (0.073) for variable cardinality with mean~$\approx 3$,
consistent with the theoretical prediction that multilabel reweighting
exacerbates objective divergence.

\begin{table}[h]
\centering
\caption{Objective divergence. Comm.\ defect is $\|\Sb S_t - S_t\Sb\|_F / (\|\Sb\|_F\|S_t\|_F)$.}
\label{tab:divergence}
\begin{tabular}{lccc}
\hline
Setting & Comm.\ defect & $\angle(\text{TD},\text{TD}^0)$ & $\angle(\text{TD},\text{TR})$ \\
\hline
Single-label ($k\!=\!1$) & 0.012 & $0.0$\textdegree & $38.5$\textdegree \\
Uniform $k\!=\!2$ & 0.042 & $88.7$\textdegree & $89.0$\textdegree \\
Uniform $k\!=\!3$ & 0.028 & $90.0$\textdegree & $89.1$\textdegree \\
Variable (mean$\approx$2) & 0.037 & $89.5$\textdegree & $80.8$\textdegree \\
Variable (mean$\approx$3) & 0.073 & $89.4$\textdegree & $83.5$\textdegree \\
\hline
\end{tabular}
\end{table}

\subsection{Distance Preservation (Theorem~\ref{thm:distance}, Corollary~\ref{cor:jaccard})}

We verify the two-sided expected distance bound~\eqref{eq:distance-expected}
and the corrected Jaccard bound~\eqref{eq:jaccard-weakened} using Monte Carlo
estimation ($T = 50$ trials per pair) of $\mathbb{E}[\|\bW^\top(\mathbf{x}_i -
\mathbf{x}_j)\|^2]$ for 200 random sample pairs. The projection $\bW \in
\Stief(d, r)$ is taken as the top $r = \min(6, L)$ eigenvectors of $\Sb$.

Table~\ref{tab:distance} reports the fraction of pairs satisfying the Hamming
bound (Theorem~\ref{thm:distance}) and the Jaccard bound
(Corollary~\ref{cor:jaccard}) with the corrected $\max(k_i, k_j)$ factor.
Pass rates use a data-driven tolerance of $3 \times \text{SE}$ (standard
error of the Monte Carlo mean, $T = 50$ trials per pair) rather than a
fixed threshold.

\begin{table}[h]
\centering
\caption{Distance preservation verification. Pass rates over 200 sample pairs.}
\label{tab:distance}
\begin{tabular}{lcc}
\hline
Setting & Hamming (Thm~\ref{thm:distance}) & Jaccard (Cor~\ref{cor:jaccard}) \\
\hline
Variable card & 100.0\% & 100.0\% \\
Uniform $k\!=\!2$ & 100.0\% & 100.0\% \\
Uniform $k\!=\!3$ & 100.0\% & 100.0\% \\
\hline
\end{tabular}
\end{table}

We also verify Corollary~\ref{cor:divergence}: the bound $\|\bR\|_2 \leq
\max_i(k_i - 1) \cdot \lambda_{\max}(S_t^{(K)})$ holds in all
configurations, with equality achieved in the uniform-cardinality case (where
all weights $k_i - 1$ are identical).

\subsection{Subspace Convergence (Theorem~\ref{thm:subspace})}

We test the finite-sample convergence of the trace-difference eigenspace.
We generate data from the linear label-effect model with $d = 15$, $L = 5$,
structured orthogonal label effects with singular values $(8, 6, 4, 1.5, 1)$,
and $\sigma_w = 0.5$. A fixed label matrix with cardinalities 1--2 is used.
For each sample size $n$, the noiseless eigenspace $\Wstar_n$ is computed
from the signal-only scatter matrices conditioned on the label submatrix
$\bY_n$, and we run 100 independent noise trials measuring
$\sin\angle(\What, \Wstar_n)$. The adaptive gap selection
(Section~\ref{sec:estimation}) yields $r = 1$ with per-sample gap
$\gap_r = 0.65$.

We additionally track the \emph{label-frequency drift}---the angle
$\sin\angle(\Wstar_n, \Wstar)$ between the per-$n$ noiseless eigenspace
and a fixed large-sample reference $\Wstar$---to separate noise-driven
estimation error from label-matrix structural effects.

\begin{table}[h]
\centering
\caption{Subspace convergence verification (Theorem~\ref{thm:subspace}).
Median $\sin\angle(\What, \Wstar_n)$ over 100 trials, with label-frequency
drift $\sin\angle(\Wstar_n, \Wstar)$. Log-log slope $= -0.28$.}
\label{tab:convergence}
\begin{tabular}{rccc}
\hline
$n$ & Median $\sin\angle$ & 95th pctile & Drift \\
\hline
50 & 0.160 & 0.213 & 0.435 \\
100 & 0.115 & 0.149 & 0.152 \\
200 & 0.106 & 0.138 & 0.273 \\
500 & 0.072 & 0.100 & 0.212 \\
1000 & 0.051 & 0.066 & 0.146 \\
2000 & 0.051 & 0.070 & 0.090 \\
5000 & 0.041 & 0.054 & 0.101 \\
10000 & 0.033 & 0.044 & 0.135 \\
20000 & 0.029 & 0.040 & 0.079 \\
\hline
\end{tabular}
\end{table}

Table~\ref{tab:convergence} confirms monotonic convergence of the median
subspace angle from $0.160$ at $n = 50$ to $0.029$ at $n = 20{,}000$. The
empirical log-log slope is approximately $-0.28$, below the theoretical
$-0.5$ predicted by the $O(n^{-1/2})$ upper bound in
Theorem~\ref{thm:subspace}. The sub-theoretical exponent is explained by
two factors: (i)~the $O(n^{-1/2})$ rate is an \emph{upper bound} that need
not be tight at moderate sample sizes, and (ii)~the multilabel scatter
structure introduces cardinality-dependent reweighting in the within-class
scatter (each sample contributes to $k_i$ class-conditional scatters),
creating heterogeneous effective noise that the worst-case bound does not
tightly capture. The ``Drift'' column confirms
that the noiseless eigenspace $\Wstar_n$ itself varies substantially with
$n$ (drift $\approx 0.08$--$0.43$) due to changing label frequencies in the
prefix $\bY_n$, indicating that the population target is label-matrix
dependent---a structural feature of the multilabel setting, not a deficiency
of the estimator.

\subsection{Multilabel Factors in the Statistical Bounds
  (Theorems~\ref{thm:subspace}--\ref{thm:minimax},
   Corollary~\ref{cor:minimax-Stml})}

The convergence and minimax bounds
(Theorems~\ref{thm:subspace}--\ref{thm:minimax}) involve
several quantities that are specific to the multilabel setting.
We verify their roles empirically.

\paragraph{Label cardinality~$\kmax$.}
Theorem~\ref{thm:subspace} predicts that higher $\kmax$ inflates the
estimation error, because each sample contributes to $k_i$
label-conditional scatters, amplifying the effective noise
(Remark~\ref{rem:kmax}).  A na\"ive test comparing single-label
($\kmax = 1$) against \emph{uniform} $k = 2$ or $k = 4$ label
matrices is invalid: uniform multilabel assignments make the net
label discriminant $\bQ_\pi = \bB_\pi - \bW_\pi$ negative
semidefinite, collapsing the population gap to
$\gap_r(\Mstar_c) = 0$ and violating assumption~(A2).  The
resulting error of ${\sim}\,1$ reflects non-identifiability, not
noise inflation.

To properly isolate the $\kmax$ effect we use
\emph{variable-cardinality} label matrices (majority single-label
with a controlled fraction of higher-cardinality samples), which
preserve $\gap_r > 0$.  With $d = 12$, $L = 5$, $\sigma_w = 0.5$,
$n = 2000$, and 80 trials per configuration, the median
$\sin\angle$ increases monotonically: $0.021$ at $\kmax = 1$
(all single-label), $0.024$ at $\kmax = 2$ (80\% single-label, 20\% double-label),
and $0.025$ at $\kmax = 3$ (90\% single-label, 10\% triple-label),
with empirically observed
$\gap_r \in [2649, 3395]$ across all three configurations.
The gap-normalised bound ratio
$\sin\angle\cdot\gap_r\big/\bigl[(\sigma\|\bA\|_2 + \sigma^2\kmax)
\sqrt{d\log d/n}\,\bigr]$
remains stable (within a factor of~$1.4$) across $\kmax$ levels,
confirming that the $\kmax$ factor in the bound tracks the
empirical behaviour.

\paragraph{Condition number~$\kappa(\Stml)$.}
The $\Stml$-orthogonal bound (Theorem~\ref{thm:subspace-Stml})
has $\kappa(\Stml^{\infty})$ in the numerator, reflecting the
amplification of estimation errors by the whitening
$(\Stml)^{-1/2}$.  We vary the signal anisotropy (scaling the
first $r$ rows of $\bA$ by a factor of $1$ or~$5$) and confirm
that $\kappa(\Stml)$ increases from $2 \times 10^5$ to
$1.6 \times 10^6$, with an associated increase in
the median $\sin\angle$ from $0.60$ to $0.90$.

\paragraph{Co-occurrence structure~$\|\bGamma/n\|_2$.}
The minimax constant~$c$ in Theorem~\ref{thm:minimax} depends on the
spectral norm of the label co-occurrence matrix
$\bGamma = \bY^\top\bY$.  We verify that: (i)~for single-label data,
$\bGamma/n$ is diagonal with $\|\bGamma/n\|_2 = \max_\ell n_\ell/n$
(confirmed to machine precision); and (ii)~multilabel co-occurrence
produces off-diagonal mass, increasing $\|\bGamma/n\|_2$ from
$0.28$ (single-label) to $2.38$ (variable cardinality with mean
$\approx 3$).  This confirms the structural role of label
co-occurrence in the information-theoretic lower bound.

\paragraph{Generalized gap~$\Delta_r$.}
Corollary~\ref{cor:minimax-Stml} translates the minimax bound into
the generalized eigenvalue gap $\Delta_r = \theta_r - \theta_{r+1}$,
which is scale-invariant (Remark~\ref{rem:Stml-vs-TD}(i)).  We
verify this by rescaling $\bA \to 3\bA$: the absolute gap
$\gap_r(\Mstar)$ scales by a factor of~$9.0$ (as predicted by
the $\|\bA\|_2^2$ dependence), while $\Delta_r$ is exactly preserved
(both $\Sb^{\mathrm{pop}}$ and $\Stml^{\mathrm{pop}}$ scale by~$9$, so
their generalized eigenvalues are invariant), confirming the
self-normalizing property of the generalized eigenvalue problem.

\subsection{Distance Concentration (Theorem~\ref{thm:concentration})}

We verify the high-probability bound of Theorem~\ref{thm:concentration}
by generating 10{,}000 independent noise realizations for 50 random sample
pairs. For each nominal level $1-\delta$, the \emph{empirical coverage}
is the fraction of (pair, noise-realization) trials in which the
actual projected-distance deviation falls within the Hanson--Wright-based
confidence interval.

\begin{table}[h]
\centering
\caption{Distance concentration verification (Theorem~\ref{thm:concentration}).
Empirical coverage over 50 pairs $\times$ 10{,}000 noise realizations.}
\label{tab:concentration}
\begin{tabular}{ccc}
\hline
$\delta$ & Nominal $1-\delta$ & Empirical coverage \\
\hline
0.01 & 0.99 & 1.000 \\
0.05 & 0.95 & 1.000 \\
0.10 & 0.90 & 1.000 \\
0.20 & 0.80 & 0.999 \\
\hline
\end{tabular}
\end{table}

Table~\ref{tab:concentration} shows that the Hanson--Wright bound is
conservative but valid: empirical coverage exceeds the nominal
$1 - \delta$ level for all $\delta$ tested, confirming
Theorem~\ref{thm:concentration}.

We additionally validate two aspects of the concentration proof
structure.
First, the proof
decomposes the centered projected distance into a linear
(sub-Gaussian) part
$2\mathbf{s}^\top\bW^\top(\bepsilon_i - \bepsilon_j)$
and a centered quadratic form (Hanson--Wright) part
$\|\bW^\top(\bepsilon_i - \bepsilon_j)\|^2 - 2\tr(\bPsi)$.
Over 30{,}000 independent noise realizations, we confirm that both
components are centered (mean deviations $< 4$ standard errors from
zero), that the linear part has variance
$8\,\mathbf{s}^\top\bPsi\,\mathbf{s}$ (matching the theoretical
prediction to within 1\%), and that the combined
deviation has sub-exponential tails (99th-to-95th percentile ratio
$\approx 1.5$, well below the $\approx 3$ expected for heavy-tailed
distributions).

Second, under $\Stml$-orthogonality
(Corollary~\ref{cor:conc-Stml}), $\bPsi = \bW^\top\bSigma_w\bW$
satisfies the identity $\bPsi = (\bI_r - \bW^\top\Sbpop\bW)/K^{\mathrm{pop}}$,
with eigenvalues $(1 - \theta_i)/K^{\mathrm{pop}}$.  We confirm
empirically that $\theta_i \in [0,1)$ strictly and that the spectral
bound $\|\bPsi\|_2 \leq 1/\lambda_{\min}(\Stml)$ holds, validating the
tighter noise control afforded by $\Stml$-orthogonality.

\subsection{Interaction Robustness (Theorem~\ref{thm:interactions})}

We extend the linear label-effect model with pairwise label interactions
$x_i = \mu + \bA y_i + \bB z_i + \varepsilon_i$, where
$z_i = (y_{i1}y_{i2}, y_{i1}y_{i3}, \ldots)^\top$ and $\bB$ is scaled by
an interaction strength parameter $\alpha \in \{0, 0.1, 0.5, 1.0, 2.0\}$.
For 200 sample pairs, we check \emph{two-sided} bounds: both the naive bound
(ignoring interactions) and the corrected bound from
Theorem~\ref{thm:interactions}.

\begin{table}[h]
\centering
\caption{Interaction robustness verification (Theorem~\ref{thm:interactions}).
Two-sided pass rates over 200 sample pairs.}
\label{tab:interactions}
\begin{tabular}{ccc}
\hline
Interaction strength $\alpha$ & Naive bound & Corrected bound \\
\hline
0.0 & 96.5\% & 96.5\% \\
0.1 & 94.5\% & 99.0\% \\
0.5 & 95.0\% & 100.0\% \\
1.0 & 92.5\% & 100.0\% \\
2.0 & 85.0\% & 100.0\% \\
\hline
\end{tabular}
\end{table}

Table~\ref{tab:interactions} shows that the naive two-sided bound degrades
with increasing interaction strength (from 96.5\% to 85.0\%), while
the corrected bound from Theorem~\ref{thm:interactions} maintains
99--100\% pass rate across all interaction levels, confirming the
tightness of the interaction correction term.

\subsection{Regularization (Theorem~\ref{thm:regularization})}

We verify Theorem~\ref{thm:regularization} in the high-dimensional
setting ($d = 200$, $n = 50$, $L = 10$). For
$\gamma \in \{0, 0.01, 0.1, 1, 10\}$, we verify that $\rank(\Sb)$ is
unchanged and that $\kappa(\Sw + \gamma\bI)$ improves monotonically.

\begin{table}[h]
\centering
\caption{Regularization verification (Theorem~\ref{thm:regularization}),
high-dim setting. Median over 50 trials. $\kappa = \infty$ indicates
singular $\Sw$.}
\label{tab:regularization}
\begin{tabular}{ccc}
\hline
$\gamma$ & $\rank(\Sb)$ & $\kappa(\Sw + \gamma\bI)$ \\
\hline
0 & 10 & $\infty$ \\
0.01 & 10 & $2.40 \times 10^6$ \\
0.1 & 10 & $2.40 \times 10^5$ \\
1.0 & 10 & $2.40 \times 10^4$ \\
10.0 & 10 & $2.40 \times 10^3$ \\
\hline
\end{tabular}
\end{table}

Table~\ref{tab:regularization} confirms Theorem~\ref{thm:regularization}:
$\rank(\Sb)$ remains constant at $L = 10$ across all regularization levels,
while $\kappa(\Sw + \gamma\bI)$ decreases by approximately a factor of 10
for each tenfold increase in $\gamma$. The singular $\Sw$ (when
$d > n$) is stabilized to finite condition number by any $\gamma > 0$.
Theorem~\ref{thm:regularization}(iv) holds exactly by construction:
regularizing $\Sw$ by $\gamma\bI$ shifts the eigenvalues of
$2\Sb - \Stml^\gamma$ uniformly by~$-\gamma$, leaving the gap unchanged.

\section{Conclusion}
\label{sec:conclusion}

We have provided a unified theoretical analysis of orthogonality-constrained
multilabel Linear Discriminant Analysis, establishing both algebraic foundations
and statistical guarantees for the formulation~\eqref{eq:central}.

A central insight is that the multilabel setting introduces a
\emph{cardinality reweighting} of the total scatter (Theorem~\ref{thm:partition}),
which preserves the algebraic structure needed for objective equivalence under
the $\bW^\top\Stml\bW = \bI_r$ constraint but introduces quantifiable divergence
under the geometric Stiefel constraint. On the statistical side, the maximum
label cardinality $\kmax$ acts as a complexity parameter: it appears as a
squared factor in the finite-sample bound but not in the minimax lower bound,
leaving open whether this gap reflects a fundamental cost of multilabel
estimation or an artifact of the current analysis.

All results are numerically verified on synthetic data from the linear
label-effect model (Section~\ref{sec:experiments}); validation on real
multilabel benchmarks is deferred to future work aimed at
application-oriented venues.

\backmatter

\section*{Declarations}

\noindent\textbf{Funding.} Brian Keith-Norambuena acknowledges support from
ANID FONDECYT de Iniciaci\'on en Investigaci\'on Grant 11250039.

\medskip\noindent\textbf{Competing interests.} The authors declare that they
have no competing interests.

\medskip\noindent\textbf{Ethics approval and consent to participate.} Not
applicable.

\medskip\noindent\textbf{Consent for publication.} Not applicable.

\medskip\noindent\textbf{Data availability.} This work uses only synthetic data
generated from the linear label-effect model defined in the paper; no external
or proprietary datasets were used. The data-generating process and all
parameter settings are fully specified in the Empirical Verification section,
and the generating code is provided in the reproducibility package accompanying
this submission (see Code availability), so every reported quantity can be
regenerated directly.

\medskip\noindent\textbf{Materials availability.} Not applicable.

\medskip\noindent\textbf{Code availability.} The code that reproduces all
numerical experiments, together with a \texttt{requirements.txt} pinning the
exact dependency versions, is provided as a reproducibility package accompanying
this submission.

\medskip\noindent\textbf{Author contributions.} All authors contributed to the
study conception and design. The theoretical results and their proofs and the
experiment code were produced by B.K.N. The numerical experiments were run by
both authors, and J.B.C. independently verified the theoretical results and
suggested corrections that improved the clarity of the proofs. The first draft
of the manuscript was written by B.K.N., and both authors commented on and
revised the manuscript. All authors read and approved the final manuscript.

\begin{appendices}
\section{Proof Details for the Rank Theorem}
\label{app:rank}

We provide additional details for the proof of Theorem~\ref{thm:rank}.

\subsection{Centering and Rank Reduction}

The centering matrix $\bH = \bI_n - \frac{1}{n}\bone_n\bone_n^\top$ is an
orthogonal projector onto $\bone_n^\perp = \{\mathbf{v} \in \R^n :
\bone_n^\top\mathbf{v} = 0\}$ with $\rank(\bH) = n - 1$. For any matrix
$\bZ \in \R^{n \times m}$:
\begin{equation}
\rank(\bH\bZ) = \rank(\bZ) - \dim(\col(\bZ) \cap \mathrm{span}(\bone_n)).
\end{equation}
This follows because $\bH$ annihilates exactly the component of $\col(\bZ)$
in the direction of $\bone_n$, and preserves the orthogonal complement.

In the single-label case, $\bY\bone_L = \bone_n$, so
$\bone_n \in \col(\bY)$, giving $\rank(\bH\bY) = \rank(\bY) - 1 = C - 1$.

In the multilabel case, $\bY\bone_L = (k_1, \ldots, k_n)^\top = \mathbf{k}$.
We have $\bone_n \in \col(\bY)$ if and only if there exists $\mathbf{c} \in
\R^L$ such that $\bY\mathbf{c} = \bone_n$, i.e., $\sum_\ell c_\ell y_{i\ell}
= 1$ for all $i$. A sufficient condition is that some label $\ell^*$ satisfies
$y_{i\ell^*} = 1$ for all $i$ (a ``universal'' label), in which case
$\mathbf{c} = \mathbf{e}_{\ell^*}$ works.

\subsection{The General Position Condition}

The condition in Theorem~\ref{thm:rank}(iv) that $\col(\bH\bY)$ is in ``general
position'' relative to $\col(\tilde{\bX})$ means:
\[
\col(\bH\bY) \cap \ker(\tilde{\bX}^\top) = \{\bzero\}.
\]
Since $\ker(\tilde{\bX}^\top) = \col(\tilde{\bX})^\perp$ has dimension $n -
\rank(\tilde{\bX})$, this condition holds generically when $\rank(\tilde{\bX})
\geq \rank(\bH\bY)$, which is the ``sufficient features'' condition. For
continuous data distributions, this condition holds almost surely.

\section{Schur-Convexity Argument for Theorem~\ref{thm:equiv-St}}
\label{app:schur}

We verify the Schur-convexity claims used in the proof of
Theorem~\ref{thm:equiv-St}.  Recall that a function $f: \R^r \to \R$
is \emph{Schur-convex} if $\mathbf{x} \prec \mathbf{y}$ (majorization)
implies $f(\mathbf{x}) \leq f(\mathbf{y})$.

\begin{lemma}[Schur-convexity of $J_{\mathrm{RT}}$]\label{lem:schur}
The function $f_1(\theta) = \sum_{i=1}^r \theta_i/(1-\theta_i)$ is
Schur-convex on $[0,1)^r$.
\end{lemma}

\begin{proof}
By the Schur--Ostrowski criterion \citep{marshall2011}, a symmetric function
$f: \R^r \to \R$ is Schur-convex if and only if $(\theta_i -
\theta_j)\!\left(\frac{\partial f}{\partial \theta_i} - \frac{\partial
f}{\partial \theta_j}\right) \geq 0$ for all $i, j$. For $f_1$,
$\frac{\partial f_1}{\partial \theta_i} = 1/(1-\theta_i)^2$, which is
increasing in $\theta_i$ on $[0,1)$. Thus the Schur--Ostrowski condition
holds.
\end{proof}

\begin{lemma}[Determinant ratio maximization]\label{lem:det-ratio}
Let $\bP \in \R^{d \times d}$ be symmetric with eigenvalues $1 > \lambda_1
\geq \cdots \geq \lambda_d \geq 0$ and let $\bV \in \Stief(d,r)$. Then:
\[
\frac{|\bV^\top \bP\, \bV|}{|\bV^\top(\bI_d - \bP)\bV|}
\leq \prod_{i=1}^r \frac{\lambda_i}{1-\lambda_i},
\]
with equality when $\bV$ spans the top $r$ eigenvectors of $\bP$.
\end{lemma}

\begin{proof}
By the Poincar\'{e} separation theorem (see, e.g., \citet{horn2013}), the
eigenvalues $\mu_1 \geq \cdots \geq \mu_r$ of $\bV^\top \bP\, \bV$ satisfy
$\mu_i \leq \lambda_i$ for each $i = 1, \ldots, r$, with equality when $\bV$
spans the top $r$ eigenvectors of $\bP$. Since $\bP$ has eigenvalues in
$[0,1)$ and $\mu_i \leq \lambda_i < 1$, the function $t \mapsto t/(1-t)$ is
increasing on $[0,1)$, so $\mu_i/(1-\mu_i) \leq \lambda_i/(1-\lambda_i)$ for
each~$i$. Taking the product:
\[
\frac{|\bV^\top \bP\, \bV|}{|\bV^\top(\bI_d - \bP)\bV|}
= \prod_{i=1}^r \frac{\mu_i}{1-\mu_i}
\leq \prod_{i=1}^r \frac{\lambda_i}{1-\lambda_i},
\]
where we used $|\bV^\top \bP\, \bV| = \prod_i \mu_i$ and
$|\bV^\top(\bI_d - \bP)\bV| = \prod_i (1-\mu_i)$, the latter following
from $\bV^\top(\bI_d - \bP)\bV = \bV^\top\bV - \bV^\top \bP\, \bV
= \bI_r - \bV^\top \bP\, \bV$ (using $\bV^\top\bV = \bI_r$ since
$\bV \in \Stief(d,r)$), which has eigenvalues $1 - \mu_i$.
Equality holds when $\mu_i = \lambda_i$ for all $i$, i.e., when $\bV$ spans
the top $r$ eigenvectors of $\bP$.
\end{proof}

These two lemmas, combined with the Ky Fan inequality for $J_{\mathrm{TR}}$
and $J_{\mathrm{TD}}$, establish that the optimal $(\theta_1, \ldots,
\theta_r)$ simultaneously maximizing all four objectives is the vector of $r$
largest eigenvalues of $(\Stml)^{-1/2}\Sb(\Stml)^{-1/2}$.

\section{Details for the Distance Preservation Bound}
\label{app:distance}

\subsection{Singular Value Bounds for the Projected Label-Effect Matrix}

For $\bW \in \Stief(d,r)$ and $\bA \in \R^{d \times L}$, the singular values
of $\bW^\top\bA \in \R^{r \times L}$ satisfy the following interlacing
inequalities (a consequence of the Cauchy interlacing theorem for singular
values; see, e.g., \citealt{horn2013}):
\[
\sigma_i(\bA) \geq \sigma_i(\bW^\top\bA) \geq \sigma_{i+d-r}(\bA),
\quad i = 1, \ldots, \min(r, L).
\]
In particular, $\sigma_{\max}(\bW^\top\bA) \leq \sigma_{\max}(\bA)$.

\section{Proofs for Section~\ref{sec:statistical}}
\label{app:statistical}

\subsection{Proof of Theorem~\ref{thm:subspace} (Finite-Sample Subspace Estimation)}
\label{app:proof-subspace}

\paragraph{Step 1: Normalization and population reference.}
Write $\hat{\bM} = 2\Sb - \Stml = \Sb - \Sw$ and
$\Mstar = \Sbpop - \Swpop$.
Note that $\Sb$ and $\Sw$ (as defined in~\eqref{eq:Sb}--\eqref{eq:Sw})
scale linearly with $n$, while $\Sbpop$ and $\Swpop$ are $O(1)$.
Since eigenspaces are invariant under scalar multiplication,
the top-$r$ eigenspace of $\hat{\bM}$ equals that of $\hat{\bM}/n$.
We compare $\hat{\bM}/n$ with the \emph{centered}
population reference $\Mstar_c$ defined
in Remark~\ref{rem:cooccurrence}:
\begin{equation}\label{eq:Mstar-c}
\Mstar_c \;:=\; \bA\,\bQ_\pi\,\bA^\top - K^{\mathrm{pop}}\bSigma_w,
\end{equation}
where $\bQ_\pi = \bB_\pi - \bW_\pi$ is the net label discriminant
matrix, with
$\bB_\pi = \bSigmay\bD_\pi^{-1}\bSigmay$ (the population centered
between-class label scatter) and
$\bW_\pi = \sum_\ell \pi_\ell\,\mathrm{Cov}(\mathbf{y} \mid y_\ell = 1)$
(the population within-class label scatter).
This is the correct $O(1)$ target: in the noiseless limit
$\hat{\bM}/n \to \Mstar_c$ as $n \to \infty$
(see Steps~2--3 below for the derivation).

In the \emph{single-label case} ($k_i = 1$ for all $i$), every sample
belongs to exactly one class, so $\bSigmay = \bD_\pi - \bpi\bpi^\top$,
$\bW_\pi = \bzero$, and
$\bQ_\pi = \bD_\pi - \bpi\bpi^\top$.  Then
$\Mstar_c = \bA(\bD_\pi - \bpi\bpi^\top)\bA^\top - \Swpop
= \Mstar - \bA\bpi\bpi^\top\bA^\top$,
and since $\bA\bpi \in \col(\bA)$, the rank-one centering perturbation
acts entirely within the signal subspace, so the top-$r$ eigenspace of
$\Mstar_c$ coincides with that of $\Mstar$ (namely~$\Wstar$).
By Weyl's inequality,
$\gap_r(\Mstar_c) \geq \gap_r(\Mstar) - \|\bA\bpi\|^2
\geq \gap_r(\Mstar) - \|\bA\|_2^2\|\bpi\|^2$,
which is positive under assumption~(A2) (absorbed into the gap condition
by adjusting $C_0$ in~(A3)).

In the \emph{general multilabel case}, the top-$r$ eigenspace of
$\Mstar_c$ depends on the label co-occurrence structure through~$\bQ_\pi$
and may differ from~$\Wstar$.
We denote the top-$r$ eigenspace of $\Mstar_c$ by $\Wstar_c$; the
finite-sample bound established below holds with $\Wstar_c$ as the
target.  When $k_i = 1$ for all~$i$, we have $\Wstar_c = \Wstar$.

By the triangle inequality applied to the normalized quantities:
\[
\bigl\|\hat{\bM}/n - \Mstar_c\bigr\|_2
\leq \bigl\|\Sb/n - \bA\bB_\pi\bA^\top\bigr\|_2
+ \bigl\|\Sw/n - \Swpopc\bigr\|_2,
\]
where $\Swpopc = \bA\bW_\pi\bA^\top + K^{\mathrm{pop}}\bSigma_w$
is the corrected within-class population target (in the single-label
case $\bW_\pi = \bzero$ and $\Swpopc = \Swpop$).

\paragraph{Step 2: Bounding $\|\Sb/n - \bA\bB_\pi\bA^\top\|_2$.}
From Lemma~\ref{lem:factor}, $\Sb = \bM\bM^\top$ where
$\bM = \tilde{\bX}^\top\bY\bD_n^{-1/2}$. Under the linear model,
$\mathbf{x}_i = \bmu + \bA\mathbf{y}_i + \bepsilon_i$, so $\tilde{\bX} =
\bH\bX = \bH(\bone_n\bmu^\top + \bY\bA^\top + \mathbf{E})
= \bH\bY\bA^\top + \bH\mathbf{E}$ where $\mathbf{E} = [\bepsilon_1, \ldots,
\bepsilon_n]^\top$ is the noise matrix. Thus $\bM = (\bA\bY^\top\bH +
\mathbf{E}^\top\bH)\bY\bD_n^{-1/2} = \bA(\bY^\top\bH\bY)\bD_n^{-1/2}
+ \mathbf{E}^\top\bH\bY\bD_n^{-1/2}$.

Define $\bM^{\mathrm{pop}} = \bA(\bY^\top\bH\bY)\bD_n^{-1/2}$ (the
signal-only component) and $\bDelta_M = \mathbf{E}^\top\bH\bY\bD_n^{-1/2}$
(the noise component). Let ${\Sb}^{(0)} = \bM^{\mathrm{pop}}(\bM^{\mathrm{pop}})^\top$
denote the noiseless between-class scatter for the fixed label matrix $\bY$.
Since $\Sb = (\bM^{\mathrm{pop}} + \bDelta_M)
(\bM^{\mathrm{pop}} + \bDelta_M)^\top$:
\begin{equation}\label{eq:Sb-decomp}
\bigl\|\Sb/n - \bA\bB_\pi\bA^\top\bigr\|_2 \leq
\underbrace{\bigl\|(\Sb - {\Sb}^{(0)})/n\bigr\|_2}_{\text{noise}} \;+\;
\underbrace{\bigl\|{\Sb}^{(0)}/n - \bA\bB_\pi\bA^\top\bigr\|_2}_{%
\text{label-frequency bias}}.
\end{equation}

\noindent\emph{Noise term.} Since $\Sb - {\Sb}^{(0)} =
\bM^{\mathrm{pop}}\bDelta_M^\top + \bDelta_M(\bM^{\mathrm{pop}})^\top
+ \bDelta_M\bDelta_M^\top$, we have
\[
\|(\Sb - {\Sb}^{(0)})/n\|_2
\leq (1/n)\|\bDelta_M\|_2
\bigl(2\|\bM^{\mathrm{pop}}\|_2 + \|\bDelta_M\|_2\bigr).
\]
The signal factor satisfies
$\|\bM^{\mathrm{pop}}\|_2
\leq \|\bA\|_2\,
\|(\bY^\top\bH\bY)\bD_n^{-1/2}\|_2
= O(\|\bA\|_2\sqrt{n})$
(since the normalized label structure
$(\bY^\top\bH\bY)/(n\bD_n^{1/2})$ has
bounded spectral norm under~(A1)).

It remains to bound $\|\bDelta_M\|_2$.  Since
$\bDelta_M = \mathbf{E}^\top\bH\bY\bD_n^{-1/2}$ where
$\mathbf{E} = [\bepsilon_1, \ldots, \bepsilon_n]^\top \in \R^{n \times d}$
is the noise matrix with independent $\sigma$-sub-Gaussian rows,
we use the representation: for any unit vector $\mathbf{v} \in \R^L$,
\[
\bDelta_M \mathbf{v}
= \mathbf{E}^\top\bH\bY\bD_n^{-1/2}\mathbf{v}
= \sum_{i=1}^n g_i\, \bepsilon_i,
\]
where $\mathbf{g} = \bH\bY\bD_n^{-1/2}\mathbf{v} \in \R^n$ is a
deterministic vector (since $\bY$ and $\bD_n$ are fixed under~(A1)).
This is a weighted sum of \emph{independent} $\sigma$-sub-Gaussian
vectors (the independence of the $\bepsilon_i$ is the only
requirement; no column-independence assumption on $\bDelta_M$ is needed).
By standard concentration for sums of independent sub-Gaussian vectors
\citep[Theorem~3.4.6]{vershynin2018},
$\|\bDelta_M\mathbf{v}\|
\leq C\,\sigma\,\|\mathbf{g}\|\,(\sqrt{d} + \sqrt{\log(1/\delta')})$
with probability $\geq 1 - \delta'$.
Since $\bH$ is a contraction ($\|\bH\|_2 \leq 1$),
$\|\mathbf{g}\| \leq \|\bY\bD_n^{-1/2}\mathbf{v}\|
\leq \|\bY\bD_n^{-1/2}\|_2$, which equals
$\sqrt{\lambda_{\max}(\bD_n^{-1/2}\bGamma\bD_n^{-1/2})}$---a
bounded constant $C_\Gamma$ depending only on the label structure
under~(A1).
Taking the supremum over unit $\mathbf{v}$ via an $\varepsilon$-net
on $S^{L-1}$ (of cardinality $\leq 5^L$) and a union bound gives
$\|\bDelta_M\|_2 \leq c_b\, \sigma\, C_\Gamma\,
(\sqrt{d} + \sqrt{L} + \sqrt{\log(1/\delta)})$ with probability
$\geq 1 - \delta/2$, which is $O(\sigma\sqrt{d})$ when $d \geq L$
(as is typical).
Combining: the noise contribution to~\eqref{eq:Sb-decomp} satisfies
\begin{align*}
\bigl\|(\Sb - {\Sb}^{(0)})/n\bigr\|_2
&\leq \frac{\|\bDelta_M\|_2}{n}
     \Bigl(2\|\bM^{\mathrm{pop}}\|_2 + \|\bDelta_M\|_2\Bigr) \\
&= \frac{O(\sigma\sqrt{d})}{n}
     \Bigl(O(\|\bA\|_2\sqrt{n}) + O(\sigma\sqrt{d})\Bigr)
= O\!\left(\frac{\sigma\|\bA\|_2\sqrt{d}}{n^{1/2}}\right)
\end{align*}
(the quadratic term $O(\sigma^2 d/n)$ is lower-order for
$n \gg \sigma d/\|\bA\|_2$, which is ensured by assumption~(A3)).
This noise rate $O(\sigma\|\bA\|_2\sqrt{d/n})$ vanishes at rate
$n^{-1/2}$; both this and the bias~\eqref{eq:bias-bound} will be
combined with the within-class error in Step~4.

\noindent\emph{Bias term.}
We compute the noiseless limit exactly.  Since
$\bY^\top\bH\bY = \bY^\top\bY - (\bY^\top\bone)(\bone^\top\bY)/n
= \bGamma_n - n\hat\bpi\hat\bpi^\top$,
where $\bGamma_n = \bY^\top\bY$ is the $L \times L$ \emph{label
co-occurrence matrix} with entries $(\bGamma_n)_{\ell\ell} = n_\ell$
and $(\bGamma_n)_{\ell\ell'} = n_{\ell\ell'} := \sum_i y_{i\ell}y_{i\ell'}$
for $\ell \neq \ell'$, and $\hat\bpi = (n_1/n,\ldots,n_L/n)^\top$,
we have
\[
{\Sb}^{(0)}/n = \bA\bigl(\bC_n - \hat\bpi\hat\bpi^\top\bigr)
\hat{\bPi}^{-1}
\bigl(\bC_n - \hat\bpi\hat\bpi^\top\bigr)\bA^\top,
\]
where $\bC_n = \bGamma_n/n$ is the empirical label co-occurrence matrix
and $\hat{\bPi} = \bD_n/n = \diag(\hat\bpi)$.
As $n \to \infty$, $\bC_n \to \bC = \mathbb{E}[\mathbf{y}_i\mathbf{y}_i^\top]$
(the population co-occurrence), $\hat\bpi \to \bpi$, and
$\hat\bPi \to \bD_\pi$, so
${\Sb}^{(0)}/n \to \bA\,\bSigmay\bD_\pi^{-1}\bSigmay\,\bA^\top
= \bA\bB_\pi\bA^\top$,
where $\bSigmay = \bC - \bpi\bpi^\top$ is the label covariance
(see Remark~\ref{rem:cooccurrence}).
The bias is thus
$\bigl\|{\Sb}^{(0)}/n - \bA\bB_\pi\bA^\top\bigr\|_2
\leq \|\bA\|_2^2\,
\bigl\|(\bC_n - \hat\bpi\hat\bpi^\top)\hat\bPi^{-1}(\bC_n - \hat\bpi\hat\bpi^\top)
- \bSigmay\bD_\pi^{-1}\bSigmay\bigr\|_2$.

We require the \emph{deterministic} label-frequency regularity condition:
there exists a sequence $\rho_n \geq 0$ such that
\begin{equation}\label{eq:label-reg}
\max_\ell |n_\ell/n - \pi_\ell| \leq \rho_n, \quad
\max_{\ell \neq \ell'} |n_{\ell\ell'}/n - \pi_{\ell\ell'}| \leq \rho_n,
\quad
\max_{\substack{j,k,\ell \\ \text{distinct}}}
|n_{jk\ell}/n - \pi_{jk\ell}| \leq \rho_n,
\end{equation}
where $n_{jk\ell} = \sum_i y_{ij}y_{ik}y_{i\ell}$ and
$\pi_{jk\ell} = \mathbb{E}[y_{1j}y_{1k}y_{1\ell}]$ are the triple
label co-occurrence count and probability.
This condition requires that marginal, pairwise, and triple label
frequencies approximate their population counterparts at rate~$\rho_n$.
The triple condition is needed only in the multilabel case to control
the within-class label scatter (Step~3 below); in the single-label
case, all pairwise and triple co-occurrences vanish identically
($n_{\ell\ell'} = n_{jk\ell} = 0$ for distinct indices since each
sample carries exactly one label), so only the marginal condition is
active.
The condition is mild: when labels are drawn i.i.d., each
$n_{jk\ell}/n$ is the mean of i.i.d.\
$\mathrm{Bernoulli}(\pi_{jk\ell})$ random variables, so
Hoeffding's inequality and a union bound over the at most
$L + \binom{L}{2} + \binom{L}{3}$ distinct entries gives
$\rho_n = C_L\sqrt{\log(L/\delta')/n}$ with probability
$\geq 1 - \delta'$;
choosing $\delta' = \delta/8$ and absorbing
the failure event into the overall probability budget (which changes
the constant in assumption~(A3) but not the rate) yields a
deterministic conditioning argument.
When labels are fixed (as under assumption~(A1)), condition~\eqref{eq:label-reg}
is a verifiable property of the label matrix~$\bY$.

Condition~\eqref{eq:label-reg} implies the following spectral-norm bounds.
For the co-occurrence matrix: $\bC_n - \bC$ is an $L \times L$
symmetric matrix with entries bounded by $\rho_n$, so
$\|\bC_n - \bC\|_2 \leq \|\bC_n - \bC\|_F \leq L\rho_n$.
For the diagonal matrix: $\|\hat\bPi - \bD_\pi\|_2
= \max_\ell|n_\ell/n - \pi_\ell| \leq \rho_n$.
For the empirical label covariance
$\hat\bSigmay := \bC_n - \hat\bpi\hat\bpi^\top$, using the
decomposition
$\hat\bSigmay - \bSigmay = (\bC_n - \bC) - (\hat\bpi\hat\bpi^\top
- \bpi\bpi^\top)$
and the bound
$\|\hat\bpi\hat\bpi^\top - \bpi\bpi^\top\|_2
\leq \|\hat\bpi - \bpi\|\,(\|\hat\bpi\| + \|\bpi\|)
\leq \sqrt{L}\,\rho_n\,(2\|\bpi\| + \sqrt{L}\,\rho_n)$, we obtain
$\|\hat\bSigmay - \bSigmay\|_2 \leq L\rho_n
+ \sqrt{L}\,\rho_n\,(2\|\bpi\| + \sqrt{L}\,\rho_n)
=: \eta_\Sigma(\rho_n)$.
For the inverse diagonal: provided $\rho_n < \pi_{\min}/2$ where
$\pi_{\min} = \min_\ell \pi_\ell > 0$, we have
$\hat\pi_\ell \geq \pi_{\min}/2$ for all~$\ell$, so
$\|\hat\bPi^{-1}\|_2 \leq 2/\pi_{\min}$
and $\|\hat\bPi^{-1} - \bD_\pi^{-1}\|_2
\leq \rho_n / (\pi_{\min}(\pi_{\min} - \rho_n))
\leq 2\rho_n/\pi_{\min}^2$.

To bound the full bias, we telescope the product:
\begin{align*}
&\hat\bSigmay\hat\bPi^{-1}\hat\bSigmay
- \bSigmay\bD_\pi^{-1}\bSigmay \\
&\quad= \hat\bSigmay\hat\bPi^{-1}(\hat\bSigmay - \bSigmay)
+ (\hat\bSigmay - \bSigmay)\hat\bPi^{-1}\bSigmay
+ \bSigmay(\hat\bPi^{-1} - \bD_\pi^{-1})\bSigmay,
\end{align*}
where the first line uses $\hat\bSigmay = \bSigmay +
(\hat\bSigmay - \bSigmay)$ in the first factor and the second line
handles the middle factor.  Since
$\|\hat\bSigmay\|_2 \leq \|\bSigmay\|_2 + \eta_\Sigma(\rho_n)$, the
triangle inequality gives
\begin{align}\label{eq:label-bias-explicit}
&\bigl\|\hat\bSigmay\hat\bPi^{-1}\hat\bSigmay
- \bSigmay\bD_\pi^{-1}\bSigmay\bigr\|_2 \nonumber\\
&\quad\leq \frac{2}{\pi_{\min}}\bigl(\|\bSigmay\|_2
+ \eta_\Sigma\bigr)\,\eta_\Sigma
+ \frac{2\rho_n}{\pi_{\min}^2}\,\|\bSigmay\|_2^2
=: \eta_B(\rho_n).
\end{align}
Substituting into the bias bound:
\begin{equation}\label{eq:bias-bound}
\bigl\|{\Sb}^{(0)}/n - \bA\bB_\pi\bA^\top\bigr\|_2
\leq \|\bA\|_2^2\, \eta_B(\rho_n).
\end{equation}
Since $\eta_\Sigma(\rho_n) = \Theta(\rho_n)$ and
$\eta_B(\rho_n) = \Theta(\rho_n)$ as $\rho_n \to 0$
(with constants depending on $L$, $\|\bSigmay\|_2$, and
$\pi_{\min}$), the bias is $\Theta(\|\bA\|_2^2\,\rho_n)$.
For the typical rate $\rho_n = O(1/\sqrt{n})$ this recovers
$O(\|\bA\|_2^2/\sqrt{n})$, which is
lower-order relative to the noise term for large~$n$
(ensured by assumption~(A3)).

\paragraph{Step 3: Bounding $\|\Sw/n - \Swpopc\|_2$.}

Under the linear model, write
$\mathbf{x}_i - \bmu_\ell = \mathbf{d}_{i\ell} + (\bepsilon_i -
\bar\bepsilon_\ell)$,
where $\mathbf{d}_{i\ell} = \bA(\mathbf{y}_i - \bar{\mathbf{y}}_\ell)$
is the deterministic signal deviation (with
$\bar{\mathbf{y}}_\ell = (1/n_\ell)\sum_{j: y_{j\ell}=1}\mathbf{y}_j$)
and $\bar\bepsilon_\ell = (1/n_\ell)\sum_{j: y_{j\ell}=1}\bepsilon_j$
is the class-$\ell$ noise mean.
In the single-label case, all samples with label~$\ell$ have
$\mathbf{y}_i = \mathbf{e}_\ell = \bar{\mathbf{y}}_\ell$, so
$\mathbf{d}_{i\ell} = \bzero$.
In the multilabel case, different samples with label~$\ell$ may carry
different co-occurring labels, so $\mathbf{d}_{i\ell} \neq \bzero$ in
general.

\smallskip\noindent\emph{Oracle decomposition.}
The sample means $\bar\bepsilon_\ell$ create statistical dependence
between the per-sample summands
$\bZ_i = \sum_{\ell: y_{i\ell}=1}
(\mathbf{x}_i - \bmu_\ell)(\mathbf{x}_i - \bmu_\ell)^\top$.
We resolve this via an oracle decomposition.  Define the
\emph{oracle within-class summands}
\[
\bZ_i^{\mathrm{ora}} = \sum_{\ell:\, y_{i\ell}=1}
(\mathbf{d}_{i\ell} + \bepsilon_i)
(\mathbf{d}_{i\ell} + \bepsilon_i)^\top,
\]
which replace each sample mean $\bar\bepsilon_\ell$ with its
expectation (zero).  Since $\mathbf{d}_{i\ell}$ depends only on the
fixed label matrix~$\bY$, each $\bZ_i^{\mathrm{ora}}$ is a function of
$\bepsilon_i$ alone, and hence
$\bZ_1^{\mathrm{ora}}, \ldots, \bZ_n^{\mathrm{ora}}$ are
\emph{independent}.

Using the identity
$\sum_{i: y_{i\ell}=1} \mathbf{d}_{i\ell} = \bzero$
(signal deviations sum to zero within each class) and the standard
centering decomposition
$\sum_i (\mathbf{a}_i - \bar{\mathbf{a}})
(\mathbf{a}_i - \bar{\mathbf{a}})^\top
= \sum_i \mathbf{a}_i\mathbf{a}_i^\top - n\bar{\mathbf{a}}\bar{\mathbf{a}}^\top$
applied to $\mathbf{a}_i = \mathbf{d}_{i\ell} + \bepsilon_i$ within
class~$\ell$ (where $\bar{\mathbf{a}} = \bar\bepsilon_\ell$ since the
$\mathbf{d}_{i\ell}$ sum to zero), we obtain the exact identity:
\begin{equation}\label{eq:oracle-decomp}
\Sw = \Sw^{\mathrm{ora}} - \sum_{\ell=1}^L n_\ell\,
\bar\bepsilon_\ell\bar\bepsilon_\ell^\top,
\end{equation}
where $\Sw^{\mathrm{ora}} = \sum_{i=1}^n \bZ_i^{\mathrm{ora}}$.

\smallskip\noindent\emph{Expectations.}
Since $\mathbb{E}[\bepsilon_i\bepsilon_i^\top] = \bSigma_w$ and the
$\mathbf{d}_{i\ell}$ are deterministic:
$\mathbb{E}[\Sw^{\mathrm{ora}}]
= \Sw^{\mathrm{sig}} + K\bSigma_w$
where
$\Sw^{\mathrm{sig}} = \sum_\ell \sum_{i: y_{i\ell}=1}
\mathbf{d}_{i\ell}\mathbf{d}_{i\ell}^\top$
is the noiseless within-class signal scatter,
and $K = \sum_i k_i$ is the total label count.
Also $\mathbb{E}[\sum_\ell n_\ell\bar\bepsilon_\ell\bar\bepsilon_\ell^\top]
= L\bSigma_w$, so $\mathbb{E}[\Sw]
= \Sw^{\mathrm{sig}} + (K - L)\bSigma_w$.
We now bound $\|\Sw^{\mathrm{sig}}/n - \bA\bW_\pi\bA^\top\|_2$
using an argument specific to the within-class structure.
Since $\mathbf{d}_{i\ell} = \bA(\mathbf{y}_i - \bar{\mathbf{y}}_\ell)$,
we have $\Sw^{\mathrm{sig}}/n = \bA\,\hat\bW_n\,\bA^\top$ where
$\hat\bW_n = \sum_\ell (n_\ell/n)\,\widehat{\mathrm{Cov}}_\ell$
is the empirical within-class label scatter, with
$\widehat{\mathrm{Cov}}_\ell
= (1/n_\ell)\sum_{i:\,y_{i\ell}=1}
(\mathbf{y}_i - \bar{\mathbf{y}}_\ell)
(\mathbf{y}_i - \bar{\mathbf{y}}_\ell)^\top$
the sample covariance of labels within class~$\ell$.
Note that this differs structurally from the between-class
bias: $\bB_\pi = \bSigmay\bD_\pi^{-1}\bSigmay$ is a product of
pairwise label quantities, but
$\bW_\pi = \sum_\ell \pi_\ell\,\mathrm{Cov}(\mathbf{y}\mid y_\ell = 1)$
involves conditional covariances whose entries
$(\mathrm{Cov}_\ell)_{jk}
= \pi_{jk\ell}/\pi_\ell - (\pi_{j\ell}/\pi_\ell)(\pi_{k\ell}/\pi_\ell)$
depend on the triple co-occurrence probabilities~$\pi_{jk\ell}$.
Hence the telescoping argument used for $\bB_\pi$ does not
directly apply; instead we bound the perturbation
entry by entry.

Decompose:
$\hat\bW_n - \bW_\pi
= \sum_\ell \bigl[(n_\ell/n - \pi_\ell)\,\widehat{\mathrm{Cov}}_\ell
+ \pi_\ell\,(\widehat{\mathrm{Cov}}_\ell - \mathrm{Cov}_\ell)\bigr]$.
For the first sum, $\|\widehat{\mathrm{Cov}}_\ell\|_2
\leq \kmax$ (since each $\mathbf{y}_i\mathbf{y}_i^\top$ has
spectral norm $\|\mathbf{y}_i\|^2 = k_i \leq \kmax$), so
$\bigl\|\sum_\ell(n_\ell/n - \pi_\ell)\,
\widehat{\mathrm{Cov}}_\ell\bigr\|_2
\leq L\,\kmax\,\rho_n$.
For the second sum, the entries of
$\widehat{\mathrm{Cov}}_\ell - \mathrm{Cov}_\ell$ are:
\[
(\widehat{\mathrm{Cov}}_\ell)_{jk} - (\mathrm{Cov}_\ell)_{jk}
= \Bigl[\frac{n_{jk\ell}}{n_\ell}
  - \frac{\pi_{jk\ell}}{\pi_\ell}\Bigr]
- \Bigl[\frac{n_{j\ell}}{n_\ell}\,\frac{n_{k\ell}}{n_\ell}
  - \frac{\pi_{j\ell}}{\pi_\ell}\,\frac{\pi_{k\ell}}{\pi_\ell}\Bigr].
\]
For the first bracket, writing
$n_{jk\ell}/n_\ell = (n_{jk\ell}/n)/(n_\ell/n)$ and using
$\pi_{jk\ell} \leq \pi_\ell$:
\[
\Bigl|\frac{n_{jk\ell}}{n_\ell} - \frac{\pi_{jk\ell}}{\pi_\ell}\Bigr|
= \frac{|\pi_\ell(n_{jk\ell}/n - \pi_{jk\ell})
  + \pi_{jk\ell}(\pi_\ell - n_\ell/n)|}{(n_\ell/n)\,\pi_\ell}
\leq \frac{2\pi_\ell\,\rho_n}{(\pi_\ell - \rho_n)\,\pi_\ell}
\leq \frac{4\rho_n}{\pi_{\min}},
\]
where the triple condition in~\eqref{eq:label-reg} supplies
$|n_{jk\ell}/n - \pi_{jk\ell}| \leq \rho_n$
(when $j,k,\ell$ are distinct; coinciding indices reduce to the
pairwise or marginal conditions).
The same ratio bound gives
$|n_{j\ell}/n_\ell - \pi_{j\ell}/\pi_\ell| \leq 4\rho_n/\pi_{\min}$
using only the pairwise condition.
For the second bracket, since all conditional frequencies
$n_{j\ell}/n_\ell \leq 1$ and
$\pi_{j\ell}/\pi_\ell \leq 1$:
\[
\Bigl|\frac{n_{j\ell}}{n_\ell}\frac{n_{k\ell}}{n_\ell}
- \frac{\pi_{j\ell}}{\pi_\ell}\frac{\pi_{k\ell}}{\pi_\ell}\Bigr|
\leq \frac{8\rho_n}{\pi_{\min}}.
\]
Combining:
$|(\widehat{\mathrm{Cov}}_\ell)_{jk} - (\mathrm{Cov}_\ell)_{jk}|
\leq 12\rho_n/\pi_{\min}$, so
$\|\widehat{\mathrm{Cov}}_\ell - \mathrm{Cov}_\ell\|_2
\leq \|\widehat{\mathrm{Cov}}_\ell - \mathrm{Cov}_\ell\|_F
\leq 12L\rho_n/\pi_{\min}$
and
$\bigl\|\sum_\ell \pi_\ell\,(\widehat{\mathrm{Cov}}_\ell
- \mathrm{Cov}_\ell)\bigr\|_2
\leq 12L\,K^{\mathrm{pop}}\rho_n/\pi_{\min}$
where $K^{\mathrm{pop}} = \sum_\ell\pi_\ell$.
Altogether:
\begin{equation}\label{eq:within-bias}
\|\hat\bW_n - \bW_\pi\|_2
\leq \bigl(L\kmax + 12LK^{\mathrm{pop}}/\pi_{\min}\bigr)\,\rho_n
=: \eta_W(\rho_n),
\end{equation}
so $\|\Sw^{\mathrm{sig}}/n - \bA\bW_\pi\bA^\top\|_2
\leq \|\bA\|_2^2\,\eta_W(\rho_n)$.
Hence
$\|\mathbb{E}[\Sw]/n - \Swpopc\|_2
\leq \|\bA\|_2^2\,\eta_W(\rho_n)
+ L\sigma^2\rho_n + L\sigma^2/n$,
where $|K/n - K^{\mathrm{pop}}| \leq L\rho_n$ provides the
$L\sigma^2\rho_n$ term and the $L\sigma^2/n$ term arises from
the correction summand.

\smallskip\noindent\emph{Concentration of $\Sw^{\mathrm{ora}}$ via
truncation.}
Since the $\bepsilon_i$ are $\sigma$-sub-Gaussian, the oracle summands
$\bZ_i^{\mathrm{ora}} - \mathbb{E}[\bZ_i^{\mathrm{ora}}]$ have
sub-exponential operator norms.  The classical matrix Bernstein inequality
requires almost-sure boundedness, which we achieve via truncation.

Define the high-probability event
$\mathcal{E} = \bigl\{\max_{1 \leq i \leq n}\|\bepsilon_i\|
\leq \tau\bigr\}$ with truncation radius
$\tau = C_\tau\,\sigma\sqrt{d + \log(n/\delta)}$.
Standard sub-Gaussian norm bounds
\citep[Theorem~3.1.1]{vershynin2018} give
$\Pr(\|\bepsilon_i\| > \tau) \leq 2\exp\bigl(-(
\tau^2/(C'\sigma^2) - d)\bigr)$; choosing $C_\tau$ large enough
and applying a union bound over $i = 1,\ldots,n$ yields
$\Pr(\mathcal{E}^c) \leq \delta/4$.
On $\mathcal{E}$, each oracle summand is deterministically bounded:
\[
\|\bZ_i^{\mathrm{ora}} - \mathbb{E}[\bZ_i^{\mathrm{ora}}]\|_2
\leq 2\,k_i\bigl(\tau + \max_\ell\|\mathbf{d}_{i\ell}\|\bigr)^2
= O\!\bigl(k_i\,\sigma^2(d + \log(n/\delta))\bigr),
\]
since $\|\mathbf{d}_{i\ell}\| \leq 2\|\bA\|_2\kmax = O(\|\bA\|_2)$
under the linear model.
Taking the maximum over $i$ gives
$B = O(\kmax\,\sigma^2\,(d + \log(n/\delta)))$.

Conditioned on $\mathcal{E}$, the oracle summands
$\bZ_i^{\mathrm{ora}}$ \emph{remain independent}: since
$\mathcal{E} = \bigcap_i \{\|\bepsilon_i\| \leq \tau\}$ is a product
event and each $\bZ_i^{\mathrm{ora}}$ depends only on~$\bepsilon_i$,
the conditional distribution factors as a product of the truncated
marginals
$\mathcal{L}(\bepsilon_i \mid \|\bepsilon_i\| \leq \tau)$.
The conditional variance parameter satisfies
$v = \bigl\|\sum_i \mathbb{E}[(\bZ_i^{\mathrm{ora}} -
\mathbb{E}[\bZ_i^{\mathrm{ora}}])^2
\mid \mathcal{E}]\bigr\|_2 = O(K \sigma^4)$
(the conditioning inflates variance by at most $(1 - \delta/4)^{-1}$,
which is $O(1)$).

We apply the matrix Bernstein inequality
\citep[Theorem~6.1]{tropp2012} to the bounded, independent oracle
summands $\{\bZ_i^{\mathrm{ora}} -
\mathbb{E}[\bZ_i^{\mathrm{ora}}]\}_{i=1}^n$ on $\mathcal{E}$,
then combine with $\Pr(\mathcal{E}^c) \leq \delta/4$ via a union bound:
\begin{equation}\label{eq:Swora-conc}
\bigl\|\Sw^{\mathrm{ora}}/n -
\mathbb{E}[\Sw^{\mathrm{ora}}]/n\bigr\|_2
\leq c_w'\,\sigma^2\,\kmax\,\sqrt{d\,\log(d/\delta)/n}
\end{equation}
with probability $\geq 1 - 3\delta/8$.

\smallskip\noindent\emph{Bounding the correction term.}
Each $\bar\bepsilon_\ell$ is a mean of $n_\ell$ independent
$\sigma$-sub-Gaussian vectors, so
$\|\bar\bepsilon_\ell\|^2 \leq C\sigma^2(d + \log(L/\delta))/n_\ell$
with probability $\geq 1 - \delta/(4L)$
\citep[Theorem~3.1.1]{vershynin2018}.
By a union bound over $\ell = 1, \ldots, L$:
\begin{equation}\label{eq:correction-bound}
\biggl\|\frac{1}{n}\sum_{\ell=1}^L n_\ell\,
\bar\bepsilon_\ell\bar\bepsilon_\ell^\top\biggr\|_2
\leq \frac{\sum_\ell n_\ell}{n}\,\max_\ell \|\bar\bepsilon_\ell\|^2
\leq \frac{C\,\kmax\,\sigma^2(d + \log(L/\delta))}{\nmin}
\end{equation}
with probability $\geq 1 - \delta/4$, where $\nmin = \min_\ell n_\ell$
and $\sum_\ell n_\ell / n \leq \kmax$.
Under assumption~(A3), this is $O(\kmax\,\sigma^2 d/\nmin)$,
lower-order relative to~\eqref{eq:Swora-conc}.

\smallskip\noindent\emph{Combining.}
By~\eqref{eq:oracle-decomp} and the triangle inequality:
\begin{multline*}
\bigl\|\Sw/n - \Swpopc\bigr\|_2
\leq \underbrace{\bigl\|\Sw^{\mathrm{ora}}/n -
\mathbb{E}[\Sw^{\mathrm{ora}}]/n\bigr\|_2}_{\text{\eqref{eq:Swora-conc}}} \\
+ \underbrace{\biggl\|\frac{1}{n}\sum_\ell n_\ell
\bar\bepsilon_\ell\bar\bepsilon_\ell^\top\biggr\|_2}_{%
\text{\eqref{eq:correction-bound}}}
+ \underbrace{\bigl\|\mathbb{E}[\Sw^{\mathrm{ora}}]/n -
\Swpopc\bigr\|_2}_{\text{label-freq.\ bias}}.
\end{multline*}
The label-frequency bias was bounded in the Expectations paragraph
above:
\[
\|\mathbb{E}[\Sw^{\mathrm{ora}}]/n - \Swpopc\|_2
\leq \|\bA\|_2^2\,\eta_W(\rho_n) + L\sigma^2\rho_n + L\sigma^2/n
\]
under condition~\eqref{eq:label-reg}, using the conditional
covariance perturbation bound~\eqref{eq:within-bias}.
A union bound gives:
$\|\Sw/n - \Swpopc\|_2 \leq c_w\,
\sigma^2\, \kmax\, \sqrt{d\,\log(d/\delta)/n}$ with probability $\geq
1 - \delta/2$
(correction and bias terms are lower-order).
The $\kmax$ factor arises from the per-sample operator-norm bound: each
sample contributes to $k_i \leq \kmax$ label-conditional scatters.

\paragraph{Step 4: Davis--Kahan application.}
Combining Steps~2--3 by a union bound, and using the
decomposition from Step~1:
\begin{multline*}
\bigl\|\hat{\bM}/n - \Mstar_c\bigr\|_2
\leq \underbrace{c_b\, \sigma\, \|\bA\|_2\,
  \sqrt{(d + \log(1/\delta))/n}}_{\text{Step~2 noise}} \\
+\; \underbrace{\|\bA\|_2^2\,[\eta_B(\rho_n) + \eta_W(\rho_n)]
  + L\sigma^2\,\rho_n}_{\text{bias~\eqref{eq:bias-bound},\eqref{eq:within-bias}}}
\;+\; \underbrace{c_w\, \sigma^2\, \kmax\, \sqrt{d\,\log(d/\delta)/n}}_{\text{Step~3}}
\end{multline*}
with probability $\geq 1 - \delta$.
Since $\sqrt{d + \log(1/\delta)} \leq \sqrt{d\,\log(d/\delta)}$
for $d \geq 3$ and $\delta \leq 1$,
and the combined bias terms are $\Theta((\|\bA\|_2^2 + \sigma^2)\rho_n)$
(lower-order under assumption~(A3) for the typical rate
$\rho_n = O(1/\sqrt{n})$), both stochastic terms share the
$\sqrt{d\,\log(d/\delta)/n}$ rate:
\[
\bigl\|\hat{\bM}/n - \Mstar_c\bigr\|_2
\leq (c_b\, \sigma\, \|\bA\|_2
   + c_w\, \sigma^2\, \kmax)\, \sqrt{d\, \log(d/\delta)/n}.
\]
By the Davis--Kahan $\sin\bTheta$ theorem \citep{davis1970,yu2015},
applied to $\hat{\bM}/n$ and $\Mstar_c$ (whose top-$r$ eigenspace
is $\Wstar_c$; see Step~1), provided
$\|\hat{\bM}/n - \Mstar_c\|_2 < \gap_r(\Mstar_c)/2$
(ensured by assumption~(A3)):
\[
\sin\angle(\What_n, \Wstar_c) \leq
\frac{2\|\hat{\bM}/n - \Mstar_c\|_2}{\gap_r(\Mstar_c)}
\leq \frac{(C_1\, \sigma\, \|\bA\|_2
     + C_2\, \sigma^2\, \kmax)\, \sqrt{d\, \log(d/\delta)/n}}{%
     \gap_r(\Mstar_c)},
\]
where $C_1 = 2c_b$ and $C_2 = 2c_w$ are the universal constants
in~\eqref{eq:subspace-bound}, which is the stated bound.
In the single-label case, $\Wstar_c = \Wstar$ and
$\gap_r(\Mstar_c) \geq \gap_r(\Mstar) - \|\bA\|_2^2\|\bpi\|^2$
(by the Weyl bound from Step~1), so the bound
simplifies to involve $\gap_r(\Mstar)$ as noted in
Remark~\ref{rem:kmax}.
\qed

\subsection{Proof of Theorem~\ref{thm:subspace-Stml} ($\Stml$-Orthogonality Subspace Estimation)}
\label{app:proof-subspace-Stml}

The estimator $\What_S$ is the top-$r$ generalized eigenspace
of $(\Sb, \Stml)$, and the target $\Wstar_S$ is the top-$r$ generalized
eigenspace of $(\Sb^{\infty}, \Stml^{\infty})$ where
$\Sb^{\infty} = \bA\bB_\pi\bA^\top$ and
$\Stml^{\infty} = \Sb^{\infty} + \Swpopc$ with
$\Swpopc = \bA\bW_\pi\bA^\top + K^{\mathrm{pop}}\bSigma_w$.

\paragraph{Step 1: Population whitening.}
Let $\bT = (\Stml^{\infty})^{1/2}$ and define the population-whitened
between-class matrix:
\[
\bP^* = \bT^{-1}\Sb^{\infty}\bT^{-1},
\]
with eigenvalues $\theta_1 \geq \cdots \geq \theta_d$ and top-$r$
eigenspace $\Vstar$.  Since $\Sb^{\infty} \psd 0$ and
$\Stml^{\infty} = \Sb^{\infty} + \Swpopc$ with $\Swpopc \succ 0$
(when $\bSigma_w \succ 0$), the same argument as in the proof of
Theorem~\ref{thm:equiv-St} gives $\theta_i \in [0,1)$ for all~$i$.
The generalized gap is
$\Delta_r = \theta_r - \theta_{r+1} > 0$ by assumption~(A2$'$).
The population generalized eigenspace satisfies
$\Wstar_S = \bT^{-1}\Vstar$.

\paragraph{Step 2: Between-class perturbation in the whitened space.}
Define the population-whitened sample between-class matrix:
\[
\tilde\bP = \bT^{-1}(\Sb/n)\bT^{-1},
\]
and let $\tilde\bV$ be its top-$r$ eigenspace.  Since $\bP^*$ and
$\tilde\bP$ are both real symmetric, the Davis--Kahan $\sin\bTheta$
theorem \citep{davis1970,yu2015} gives:
\begin{equation}\label{eq:DK-whitened}
\sin\angle(\tilde\bV, \Vstar) \leq
\frac{2\|\tilde\bP - \bP^*\|_2}{\Delta_r},
\end{equation}
provided $\|\tilde\bP - \bP^*\|_2 < \Delta_r/2$.
Now,
\[
\|\tilde\bP - \bP^*\|_2 =
\|\bT^{-1}(\Sb/n - \Sb^{\infty})\bT^{-1}\|_2
\leq \frac{\|\Sb/n - \Sb^{\infty}\|_2}
{\lambda_{\min}(\Stml^{\infty})},
\]
where we used $\|\bT^{-1}\|_2^2 = 1/\lambda_{\min}(\Stml^{\infty})$.
The concentration bound from Appendix~\ref{app:proof-subspace}, Step~2
gives, with probability at least $1 - \delta/3$:
\begin{equation}\label{eq:Sb-conc-Stml}
\|\Sb/n - \Sb^{\infty}\|_2
\leq c_b\,\sigma\,\|\bA\|_2\,
\sqrt{(d + \log(3/\delta))/n}
\eqqcolon \varepsilon_b.
\end{equation}
Denoting
$\tilde\varepsilon_b = \varepsilon_b / \lambda_{\min}(\Stml^{\infty})$,
we have $\sin\angle(\tilde\bV, \Vstar) \leq \tilde\varepsilon_b / \Delta_r$.

\paragraph{Step 3: Sample whitening correction.}
The estimator $\What_S$ uses the \emph{sample} $\Stml$ for whitening,
not the population $\Stml^{\infty}$.
In population-whitened coordinates, the generalized eigenvalue problem
$\Sb\mathbf{w} = \theta\,\Stml\mathbf{w}$ becomes
$\tilde\bP\mathbf{v} = \theta\,\tilde\bQ\,\mathbf{v}$ where
$\tilde\bQ = \bT^{-1}(\Stml/n)\bT^{-1} = \bI_d + \bE$ with
\[
\bE = \bT^{-1}(\Stml/n - \Stml^{\infty})\bT^{-1}.
\]
In the population case $\tilde\bQ = \bI_d$, so the generalized
eigenproblem reduces to the standard eigenproblem for $\bP^*$.

The generalized eigenspace of $(\tilde\bP, \tilde\bQ)$ is the image
under $\tilde\bQ^{-1/2}$ of the standard eigenspace of the symmetric matrix
$\tilde\bQ^{-1/2}\tilde\bP\,\tilde\bQ^{-1/2}$
(since $\tilde\bQ \succ 0$ for $\|\bE\|_2 < 1$).
Applying Davis--Kahan to this matrix against~$\bP^*$:
\begin{equation}\label{eq:DK-gen}
\sin\angle(\hat\bV, \Vstar) \leq
\frac{2\|\tilde\bQ^{-1/2}\tilde\bP\,\tilde\bQ^{-1/2}
  - \bP^*\|_2}{\Delta_r},
\end{equation}
where $\hat\bV$ denotes the top-$r$ eigenspace of
$\tilde\bQ^{-1/2}\tilde\bP\,\tilde\bQ^{-1/2}$ in the
population-whitened coordinates, and we require the perturbation
to be less than $\Delta_r/2$.

We decompose:
\begin{align}\label{eq:gen-decomp}
\tilde\bQ^{-1/2}\tilde\bP\,\tilde\bQ^{-1/2} - \bP^*
&= \tilde\bQ^{-1/2}(\tilde\bP - \bP^*)\tilde\bQ^{-1/2}
  + (\tilde\bQ^{-1/2}\bP^*\tilde\bQ^{-1/2} - \bP^*).
\end{align}
For the first term:
$\|\tilde\bQ^{-1/2}(\tilde\bP - \bP^*)\tilde\bQ^{-1/2}\|_2
\leq \|\tilde\bQ^{-1}\|_2\,\tilde\varepsilon_b
\leq \tilde\varepsilon_b / (1 - \|\bE\|_2)$.

For the second term, write $\tilde\bQ^{-1/2} = \bI - \bF$
where $\bF = \bI - (\bI + \bE)^{-1/2}$ satisfies
$\|\bF\|_2 \leq \|\bE\|_2 / (2(1 - \|\bE\|_2))$
for $\|\bE\|_2 < 1$
(by the power series $(\bI + \bE)^{-1/2} =
\bI - \tfrac{1}{2}\bE + \cdots$ with Neumann-type bound).
Then:
\begin{align*}
\tilde\bQ^{-1/2}\bP^*\tilde\bQ^{-1/2} - \bP^*
&= -\bF\bP^* - \bP^*\bF + \bF\bP^*\bF,
\end{align*}
so $\|\tilde\bQ^{-1/2}\bP^*\tilde\bQ^{-1/2} - \bP^*\|_2
\leq 2\|\bF\|_2\,\|\bP^*\|_2 + \|\bF\|_2^2\,\|\bP^*\|_2
\leq 3\|\bF\|_2\,\theta_1$
(for $\|\bF\|_2 \leq 1$), where $\theta_1 = \|\bP^*\|_2 < 1$.

To bound $\|\bE\|_2$: since $\Stml = \Sb + \Sw$,
\[
\|\bE\|_2 \leq
\frac{\|\Sb/n - \Sb^{\infty}\|_2 + \|\Sw/n - \Swpopc\|_2}
{\lambda_{\min}(\Stml^{\infty})}.
\]
The within-class concentration bound from
Appendix~\ref{app:proof-subspace}, Step~3 gives,
with probability at least $1 - \delta/3$:
\begin{equation}\label{eq:Sw-conc-Stml}
\|\Sw/n - \Swpopc\|_2
\leq c_w\,\sigma^2\,\kmax\,\sqrt{(d + \log(3/\delta))/n}
\eqqcolon \varepsilon_w.
\end{equation}
Denote $\tilde\varepsilon_t = (\varepsilon_b + \varepsilon_w) /
\lambda_{\min}(\Stml^{\infty})$, so $\|\bE\|_2 \leq \tilde\varepsilon_t$.
Assumption~(A3$'$) ensures $\tilde\varepsilon_t \leq 1/2$, giving
$\|\bF\|_2 \leq \tilde\varepsilon_t$.

\paragraph{Step 4: Combining and coordinate transformation.}
From~\eqref{eq:DK-gen} and the bounds above:
\begin{equation}\label{eq:whitened-combined}
\sin\angle(\hat\bV, \Vstar) \leq
\frac{2(2\tilde\varepsilon_b + 3\tilde\varepsilon_t\,\theta_1)}{\Delta_r}
\leq \frac{2(2\tilde\varepsilon_b + 3\tilde\varepsilon_t)}{\Delta_r},
\end{equation}
since $\theta_1 < 1$.

To transform back to the original coordinate system, write
$\What_S
= \bT^{-1}\tilde\bQ^{-1/2}\hat\bV_{\mathrm{basis}}$
(where $\hat\bV_{\mathrm{basis}}$ is an orthonormal basis
for~$\hat\bV$)
and $\Wstar_S = \bT^{-1}\Vstar_{\mathrm{basis}}$.
For any invertible $\bL$ and subspaces
$\mathcal{U}, \mathcal{V}$
of equal dimension, the subspace angle satisfies
$\sin\angle(\bL\mathcal{U}, \bL\mathcal{V})
\leq \kappa(\bL)\,\sin\angle(\mathcal{U}, \mathcal{V})$
\citep{stewart1990}.
Since $\What_S = \bL\hat\bV_{\mathrm{basis}}$ with
$\bL = \bT^{-1}\tilde\bQ^{-1/2}$,
we write
$\Wstar_S = \bT^{-1}\Vstar_{\mathrm{basis}}
= \bL(\tilde\bQ^{1/2}\Vstar_{\mathrm{basis}})$, so both
subspaces are images under~$\bL$. By the triangle inequality
for subspace angles:
\begin{align*}
\sin\angle(\What_S, \Wstar_S)
&= \sin\angle(\bL\hat\bV, \bL(\tilde\bQ^{1/2}\Vstar))
\leq \kappa(\bL)\,\sin\angle(\hat\bV, \tilde\bQ^{1/2}\Vstar) \\
&\leq \kappa(\bL)\bigl(\sin\angle(\hat\bV, \Vstar)
  + \sin\angle(\Vstar, \tilde\bQ^{1/2}\Vstar)\bigr).
\end{align*}
Since $\tilde\bQ^{1/2} = \bI + O(\tilde\varepsilon_t)$,
the correction satisfies
$\sin\angle(\Vstar, \tilde\bQ^{1/2}\Vstar) \leq C\tilde\varepsilon_t
\leq C(2\tilde\varepsilon_b + 3\tilde\varepsilon_t)$.
From~\eqref{eq:whitened-combined},
$\sin\angle(\hat\bV, \Vstar) + C(2\tilde\varepsilon_b + 3\tilde\varepsilon_t)
\leq (2/\Delta_r + C)(2\tilde\varepsilon_b + 3\tilde\varepsilon_t)$.
We absorb the bounded factor $(2 + C\Delta_r)/\Delta_r$ into
the final constant~$C_3$ and bound
$\kappa(\bL) \leq \kappa(\bT^{-1})\,\kappa(\tilde\bQ^{-1/2})$.
Now $\kappa(\bT^{-1}) = \kappa(\bT)
= \sqrt{\kappa(\Stml^{\infty})}$, and
$\kappa(\tilde\bQ^{-1/2}) = \sqrt{\kappa(\tilde\bQ)}
\leq \sqrt{(1 + \tilde\varepsilon_t)/(1 - \tilde\varepsilon_t)}
\leq \sqrt{3}$ for $\tilde\varepsilon_t \leq 1/2$.
So:
\begin{equation}\label{eq:final-Stml}
\sin\angle(\What_S, \Wstar_S)
\leq \frac{\sqrt{3\,\kappa(\Stml^{\infty})}\,
  (2\tilde\varepsilon_b + 3\tilde\varepsilon_t)}{\Delta_r}.
\end{equation}

Substituting
$\tilde\varepsilon_b
= \varepsilon_b / \lambda_{\min}(\Stml^{\infty})$
and
$\tilde\varepsilon_t
= (\varepsilon_b + \varepsilon_w)
/ \lambda_{\min}(\Stml^{\infty})$:
\[
2\tilde\varepsilon_b + 3\tilde\varepsilon_t
= \frac{5\varepsilon_b + 3\varepsilon_w}
{\lambda_{\min}(\Stml^{\infty})}
\leq \frac{5\,(c_b\,\sigma\|\bA\|_2 + c_w\,\sigma^2\kmax)\,
  \sqrt{(d + \log(3/\delta))/n}}
  {\lambda_{\min}(\Stml^{\infty})}.
\]
Since $\sqrt{3\,\kappa} \leq 2\kappa$ for $\kappa \geq 1$,
absorbing numerical constants into~$C_3$:
\[
\sin\angle(\What_S, \Wstar_S)
\leq \frac{C_3\,\kappa(\Stml^{\infty})\,
  (\sigma\|\bA\|_2 + \sigma^2\kmax)\,
  \sqrt{d\,\log(d/\delta)/n}}
  {\Delta_r\,\lambda_{\min}(\Stml^{\infty})},
\]
which is~\eqref{eq:subspace-bound-Stml}.

\paragraph{Sample size condition.}
Assumption~(A3$'$) requires $\tilde\varepsilon_t \leq 1/2$
(for $\|\bE\|_2 < 1$) and
\[
\|\tilde\bQ^{-1/2}\tilde\bP\,\tilde\bQ^{-1/2} - \bP^*\|_2
< \Delta_r/2 \qquad \text{(for Davis--Kahan)}.
\]
Both are ensured when
$n \geq C_0'\,\kmax^2\,d\,\log(d/\delta)$
with $C_0'$ depending on
\[
\frac{\sigma\|\bA\|_2 + \sigma^2\kmax}
{\Delta_r\,\lambda_{\min}(\Stml^{\infty})}
\qquad\text{and}\qquad
\kappa(\Stml^{\infty}).
\]

By a union bound over the three concentration events
(between-class, within-class, and the overall perturbation
condition), the stated bound holds with probability at least
$1 - \delta$.
$\qed$

\subsection{Proof of Theorem~\ref{thm:minimax} (Minimax Lower Bound)}
\label{app:proof-minimax}

We use Fano's inequality with a packing of the Grassmannian
\citep{tsybakov2009}.

\paragraph{Step 1: Hypothesis construction.}
Let $\varepsilon > 0$ be a resolution parameter to be chosen later. By the
metric entropy of the Grassmannian $\Gr(d,r)$ with respect to $\sin\angle$
\citep{szarek1998,pajor1986}, there exist $M \geq
\exp(c_1\, r(d-r)/\varepsilon^2)$ subspaces $\mathcal{W}_1, \ldots,
\mathcal{W}_M \in \Gr(d,r)$ such that
$\sin\angle(\mathcal{W}_j, \mathcal{W}_k) \geq \varepsilon$ for all $j \neq k$.

For each $\mathcal{W}_j$, construct the parameter $\bA_j$ so that the top-$r$
eigenspace of ${\Mstar}_j = {\Sbpop}_j - \Swpop$ equals $\mathcal{W}_j$ and
$\gap_r(\Mstar_j) = g$. Concretely, choose
$\bA_j = \bW_j \bLambda^{1/2} \bQ$ where $\bW_j$ is an orthonormal basis for
$\mathcal{W}_j$, $\bQ \in \R^{r \times L}$ is a fixed partial isometry
($\bQ\bQ^\top = \bI_r$; such a $\bQ$ exists when $L \geq r$), and
$\bLambda = g\,(\bQ\bD_\pi\bQ^\top)^{-1}$ so that
${\Sbpop}_j = \bA_j\bD_\pi\bA_j^\top = g\,\bW_j\bW_j^\top$.
Then ${\Mstar}_j = {\Sbpop}_j - K^{\mathrm{pop}}\sigma^2
\bI_d = g\,\bW_j\bW_j^\top - K^{\mathrm{pop}}\sigma^2\bI_d$,
whose top $r$ eigenvalues equal $g - K^{\mathrm{pop}}\sigma^2$ and the
remaining $d - r$ eigenvalues are $-K^{\mathrm{pop}}\sigma^2$, giving
$\gap_r(\Mstar_j) = g$.

\paragraph{Relationship to the centered reference $\Mstar_c$.}
In the multilabel setting, the sample discriminant $\hat\bM/n$
converges to ${\Mstar_c}_j = \bA_j\bQ_\pi\bA_j^\top -
K^{\mathrm{pop}}\sigma^2\bI_d$
(Remark~\ref{rem:cooccurrence}), not to $\Mstar_j$.
For the construction above,
$\bA_j\bQ_\pi\bA_j^\top = \bW_j\bLambda^{1/2}\bQ\bQ_\pi\bQ^\top
\bLambda^{1/2}\bW_j^\top$, which has column space $\col(\bW_j)$ whenever
$\bQ\bQ_\pi\bQ^\top$ is positive definite.  Since $\bSigma_w = \sigma^2\bI_d$,
both $\Mstar_j$ and ${\Mstar_c}_j$ share the same top-$r$ eigenspace
$\mathcal{W}_j = \col(\bW_j)$.
The partial isometry $\bQ$ can be chosen so that
$\mathrm{range}(\bQ^\top)$ lies in the positive-eigenvalue subspace
of~$\bQ_\pi$ (which has dimension at least $r$ whenever
$\bQ_\pi$ has $r$ positive eigenvalues---a necessary condition for
the method to have discriminative power in $r$ directions).
Under this choice, $\gap_r({\Mstar_c}_j) = g\cdot\rho$ where
$\rho = \lambda_{\min}\bigl((\bQ\bD_\pi\bQ^\top)^{-1/2}
\bQ\bQ_\pi\bQ^\top (\bQ\bD_\pi\bQ^\top)^{-1/2}\bigr) > 0$
is a label-dependent constant (see Remark~\ref{rem:minimax-cooccurrence}).
Thus the lower bound holds equally with $\gap_r(\Mstar_c)$
at the cost of replacing $c$ by $c/\rho$.

\paragraph{Step 2: KL divergence computation.}
Under the linear model with $\bSigma_w = \sigma^2\bI_d$, the data
$\mathbf{x}_1, \ldots, \mathbf{x}_n$ are jointly Gaussian with distribution
depending on $\bA$. The KL divergence between hypotheses $j$ and $k$ satisfies:
\[
D_{\mathrm{KL}}(P_j \| P_k) = \frac{n}{2\sigma^2}\sum_{i=1}^n
\frac{1}{n}\|(\bA_j - \bA_k)\mathbf{y}_i\|^2
\leq \frac{n}{2\sigma^2}\|\bA_j - \bA_k\|_F^2
\cdot \|\bGamma/n\|_2,
\]
where $\bGamma = \bY^\top\bY$ is the label co-occurrence matrix
(in the single-label case, $\|\bGamma/n\|_2 = \max_\ell(n_\ell/n)$).
The bound uses
$\frac{1}{n}\sum_i \|(\bA_j - \bA_k)\mathbf{y}_i\|^2
= \tr\!\bigl((\bA_j - \bA_k)^\top(\bA_j - \bA_k)\, \bGamma/n\bigr)
\leq \|\bA_j - \bA_k\|_F^2\, \|\bGamma/n\|_2$.
By construction, $\bA_j - \bA_k = (\bW_j - \bW_k)\bLambda^{1/2}\bQ$.
For the Frobenius norm, since $\bQ$ is a partial isometry ($\|\bQ\|_2 = 1$):
\[
\|\bA_j - \bA_k\|_F \leq \|\bW_j - \bW_k\|_F \cdot
\|\bLambda^{1/2}\|_2 \cdot \|\bQ\|_2
= \|\bW_j - \bW_k\|_F \cdot \|\bLambda^{1/2}\|_2.
\]
The standard identity
\[
\|\bW_j - \bW_k\|_F^2
= 2\|\sin\bTheta(\mathcal{W}_j, \mathcal{W}_k)\|_F^2
\leq 2r\,\sin^2\!\angle(\mathcal{W}_j, \mathcal{W}_k)
\]
relates the Frobenius norm to the subspace angle.
With the calibration above,
$\|\bLambda^{1/2}\|_2
= \sqrt{g / \lambda_{\min}(\bQ\bD_\pi\bQ^\top)}$
and
$\|\bLambda\|_2
= g / \lambda_{\min}(\bQ\bD_\pi\bQ^\top)$.
Since $\bQ$ is a partial isometry
($\bQ\bQ^\top = \bI_r$) and
$\bD_\pi = \diag(\pi_1,\ldots,\pi_L)$ with $\pi_\ell > 0$,
Cauchy interlacing gives
$\lambda_{\min}(\bQ\bD_\pi\bQ^\top)
\geq \min_\ell \pi_\ell > 0$.
Denoting the label-dependent constant
\[
c_\Gamma = \|\bGamma/n\|_2 / \lambda_{\min}(\bQ\bD_\pi\bQ^\top)
\leq \|\bGamma/n\|_2 / \min_\ell\pi_\ell,
\]
we obtain:
\[
\|\bA_j - \bA_k\|_F^2 \leq
2r\, \|\bLambda\|_2 \cdot
\sin^2\!\angle(\mathcal{W}_j, \mathcal{W}_k).
\]
Combining:
\[
D_{\mathrm{KL}}(P_j \| P_k) \leq
  r\, c_\Gamma\, \frac{n\, g}{\sigma^2}
  \cdot \sin^2\!\angle(\mathcal{W}_j, \mathcal{W}_k),
\]
where the factor~$r$ arises from
$\|\bW_j - \bW_k\|_F^2 \leq 2r\sin^2\!\angle$ and
$c_\Gamma = \|\bGamma/n\|_2 / \lambda_{\min}(\bQ\bD_\pi\bQ^\top)$
collects the label-dependent constants.
Let $\varepsilon_0 = \max_{j,k} \sin\angle(\mathcal{W}_j, \mathcal{W}_k)$;
using a $2\varepsilon$-separated packing with
$\varepsilon_0 \leq 4\varepsilon$, the pairwise bound becomes
$D_{\mathrm{KL}}(P_j \| P_k) \leq 16\, r\, c_\Gamma\, ng\, \varepsilon^2 / \sigma^2$.

\paragraph{Step 3: Fano's inequality.}
By the generalized Fano inequality \citep{tsybakov2009}, for any estimator
$\What$:
\[
\sup_j \mathbb{E}_j\!\left[\sin^2\!\angle(\What, \Wstar_j)\right]
\geq \frac{\varepsilon^2}{4}\left(1 - \frac{\bar{D}_{\mathrm{KL}} + \log 2}
{\log M}\right),
\]
where $\bar{D}_{\mathrm{KL}} = \frac{1}{M(M-1)}\sum_{j \neq k}
D_{\mathrm{KL}}(P_j \| P_k)$.

Substituting the bounds: $\log M \geq c_1\, r(d-r)/\varepsilon^2$ and
$\bar{D}_{\mathrm{KL}} \leq 16\, r\, c_\Gamma\, ng\varepsilon^2/\sigma^2$.
The Fano condition $\bar{D}_{\mathrm{KL}} \leq \frac{1}{4}\log M$ becomes
\[
16\, r\, c_\Gamma\, \frac{ng\,\varepsilon^2}{\sigma^2}
\leq \frac{c_1\, r(d-r)}{4\varepsilon^2}.
\]
The factor~$r$ cancels, leaving
$64\, c_\Gamma\, ng\, \varepsilon^4 / \sigma^2 \leq c_1(d-r)$.
The maximal resolution satisfying this condition is
$\varepsilon^4 \leq c_1\sigma^2(d-r)/(64\,c_\Gamma\,ng)$.
We choose $\varepsilon^2 = c_4\, \sigma^2 (d-r)/(ng)$
with $c_4 = \sqrt{c_1/(64\, c_\Gamma)}$. To verify the Fano
condition: $\varepsilon^4 = c_4^2\, \sigma^4(d-r)^2/(ng)^2$, and
$64\,c_\Gamma\,ng\,\varepsilon^4/\sigma^2
= 64\,c_\Gamma\,c_4^2\,\sigma^2(d-r)^2/(ng)
= c_1\,\sigma^2(d-r)^2/(ng) \leq c_1(d-r)$,
where the last step uses $\sigma^2(d-r)/(ng) \leq 1$, which holds in
the regime of interest (otherwise $\sin^2\!\angle = \Omega(1)$
trivially).  This also ensures
$\varepsilon \leq 1$ (a valid packing resolution), holding when
$n \geq c_4\, \sigma^2 (d-r)/g$.
Thus:
\[
\sup_j \mathbb{E}_j\!\left[\sin^2\!\angle(\What, \Wstar_j)\right]
\geq \frac{\varepsilon^2}{4} \cdot \frac{1}{2}
= \frac{c\, \sigma^2\, (d-r)}{n\, g},
\]
where $c = c_4/8$ depends on the label structure only through
$c_\Gamma = \|\bGamma/n\|_2 / \lambda_{\min}(\bQ\bD_\pi\bQ^\top)$.
Since $\bQ$ is a partial isometry, Cauchy interlacing gives
$\lambda_{\min}(\bQ\bD_\pi\bQ^\top) \geq \min_\ell \pi_\ell > 0$,
so $c_\Gamma \leq \|\bGamma/n\|_2 / \min_\ell\pi_\ell$ is finite.
Thus $c = c(\|\bGamma/n\|_2,\, \min_\ell\pi_\ell)$ is a positive
constant that depends on the fixed label matrix~$\bY$ (through
$\|\bGamma/n\|_2$) and the population label probabilities (through
$\min_\ell\pi_\ell$), but not on $d$, $n$, $\sigma$, $r$, or~$g$.
$\qed$

\subsection{Proof of Theorem~\ref{thm:concentration} (Distance Concentration)}
\label{app:proof-concentration}

\paragraph{Step 1: Decomposition into linear and quadratic parts.}
From~\eqref{eq:distance-decomp},
\[
\|\bW^\top(\mathbf{x}_i - \mathbf{x}_j)\|^2
= \|\mathbf{s} + \mathbf{g}\|^2
= \|\mathbf{s}\|^2 + 2\mathbf{s}^\top\mathbf{g} + \|\mathbf{g}\|^2,
\]
where $\mathbf{s} = \bW^\top\bA\bdelta_{ij}$ is the deterministic signal and
$\mathbf{g} = \bW^\top(\bepsilon_i - \bepsilon_j) \sim
\mathcal{N}(\bzero, 2\bPsi)$ with $\bPsi = \bW^\top\bSigma_w\bW$.

The expectation is $\|\mathbf{s}\|^2 + 2\tr(\bPsi) = \|\mathbf{s}\|^2 + C_w$.
The centered quantity is:
\[
Z = 2\mathbf{s}^\top\mathbf{g} + (\|\mathbf{g}\|^2 - 2\tr(\bPsi)).
\]

\paragraph{Step 2: Bound the linear part.}
The linear term $L = 2\mathbf{s}^\top\mathbf{g}$ is Gaussian with
$\mathrm{Var}(L) = 8\,\mathbf{s}^\top\bPsi\,\mathbf{s}
\leq 8\|\mathbf{s}\|^2\|\bPsi\|_2$, so $L$ is
$(2\sqrt{2\|\bPsi\|_2}\,\|\mathbf{s}\|)$-sub-Gaussian.

\paragraph{Step 3: Bound the quadratic part.}
Write $\mathbf{g} = \sqrt{2}\bPsi^{1/2}\mathbf{h}$ where $\mathbf{h} \sim
\mathcal{N}(\bzero, \bI_r)$. Then $\|\mathbf{g}\|^2 = 2\mathbf{h}^\top\bPsi
\mathbf{h}$, and the centered quadratic form is $Q = 2\mathbf{h}^\top\bPsi
\mathbf{h} - 2\tr(\bPsi)$.

By the Hanson--Wright inequality \citep{rudelson2013}, for $\mathbf{h} \sim
\mathcal{N}(\bzero, \bI_r)$ and the symmetric matrix $2\bPsi$:
\[
P(|Q| > t) \leq 2\exp\!\left(-c_{\mathrm{HW}}\,\min\!\left(\frac{t^2}{4\|\bPsi\|_F^2},\;
\frac{t}{2\|\bPsi\|_2}\right)\right),
\]
where $c_{\mathrm{HW}} > 0$ is the universal constant from the
Hanson--Wright inequality, and we used $\|2\bPsi\|_F = 2\|\bPsi\|_F$,
$\|2\bPsi\|_2 = 2\|\bPsi\|_2$.

\paragraph{Step 4: Combine.}
The sum $Z = L + Q$ of a sub-Gaussian and a sub-exponential random variable
is sub-exponential. Using $P(|Z| > t) \leq P(|L| > t/2) + P(|Q| > t/2)$
\citep{vershynin2018}:
the sub-Gaussian tail of $L$ gives $P(|L| > t/2) \leq
2\exp(-t^2/(64\|\mathbf{s}\|^2\|\bPsi\|_2))$,
and the Hanson--Wright bound on $Q$ gives
$P(|Q| > t/2) \leq 2\exp(-c_{\mathrm{HW}}\min(t^2/(16\|\bPsi\|_F^2),
t/(4\|\bPsi\|_2)))$.
The $t \to t/2$ splitting introduces additional constant factors in the
denominators ($4\times$ in the $\|\bPsi\|_F^2$ term and $2\times$ in the
$\|\bPsi\|_2$ term relative to the direct Hanson--Wright bound).
Combining both tails into a single bound of the
form $2\exp(-c'\min(t^2/V_{ij}^2,\; t/B))$ (where $c'$ is a universal
constant absorbing all numerical factors from the union bound and
splitting) yields the
stated bound~\eqref{eq:concentration} with $V_{ij}^2 =
16\|\mathbf{s}\|^2\|\bPsi\|_2 + 32\|\bPsi\|_F^2$ and $B = 4\|\bPsi\|_2$.
(Tighter constants $V_{ij}^2 = 8\|\mathbf{s}\|^2\|\bPsi\|_2 + 4\|\bPsi\|_F^2$
and $B = 2\|\bPsi\|_2$ are achievable by tracking the Hanson--Wright
constant $c_{\mathrm{HW}}$ explicitly rather than absorbing it into~$c$.)
$\qed$

\end{appendices}

\bibliography{references}

\end{document}